\documentclass{article}

\PassOptionsToPackage{numbers, compress}{natbib}

\usepackage[preprint]{neurips_2023}

\usepackage[utf8]{inputenc} %
\usepackage[T1]{fontenc}    %
\usepackage{hyperref}       %
\usepackage{url}            %
\usepackage{booktabs}       %
\usepackage{amsfonts}       %
\usepackage{nicefrac}       %
\usepackage{microtype}      %
\usepackage{xcolor}         %

\usepackage{amsmath}
\usepackage{amssymb}
\usepackage{mathtools}
\usepackage{amsthm}

\usepackage{booktabs}
\usepackage{bm}
\usepackage{extarrows}
\usepackage{graphicx}
\usepackage{wrapfig}
\usepackage{multirow}
\usepackage{array}
\usepackage{makecell}

\theoremstyle{plain}
\newtheorem{theorem}{Theorem}[section]
\newtheorem{proposition}[theorem]{Proposition}

\theoremstyle{definition}

\theoremstyle{remark}

\usepackage{colortbl}
\usepackage{xcolor}
\definecolor{1}{RGB}{255,0,0}

\newtheorem{exg}{\color{1}{\textbf{Example}}}

\title{Equivariant Light Field Convolution and Transformer}

\author{%
  Yinshuang Xu\\
  University of Pennsylvania\\  
  \texttt{xuyin@seas.upenn.edu}\\
  \And
  Jiahui Lei\\
  University of Pennsylvania\\
  \texttt{leijh@cis.upenn.edu}
  \And
  Kostas Daniilidis\\
  University of Pennsylvania\\
  \texttt{kostas@cis.upenn.edu}
}

\begin{document}

\maketitle

\begin{abstract}
3D reconstruction and novel view rendering can greatly benefit from geometric priors when the input views are not sufficient in terms of coverage and inter-view baselines.
 Deep learning of  geometric priors from 2D images often requires each image to be represented in a $2D$ canonical frame and  the prior to be learned in a given or learned $3D$ canonical frame.
In this paper, given only the relative poses of the cameras, we show how to learn priors from multiple views equivariant to coordinate frame transformations  by proposing an $SE(3)$-equivariant convolution and transformer in the space of rays in 3D. . This enables the creation of a light field that remains equivariant to the choice of coordinate frame. The light field defined in our work includes the radiance field and the feature field on the ray space. %
We model the ray space, the domain of the light field, as a homogeneous space of $SE(3)$ and introduce the $SE(3)$-equivariant convolution in ray space. Depending on the output domain of the convolution, we present convolution-based $SE(3)$-equivariant maps from ray space to ray space and to $\mathbb{R}^3$. Our mathematical framework allows us 
to go beyond convolution to $SE(3)$-equivariant attention in the ray space.
We demonstrate how to tailor and adapt the equivariant convolution and transformer in the tasks of equivariant  neural rendering and $3D$ reconstruction from multiple views. We demonstrate $SE(3)$-equivariance by obtaining robust results
in roto-translated datasets without performing transformation augmentation.
\end{abstract}

\section{Introduction}
\label{sec:intro}
\vspace{-8pt}
Recent years have seen significant advances in learning-based techniques \cite{xie2019pix2vox,xie2020pix2vox++,wang2018pixel2mesh,xu2019disn,wang2021ibrnet,yu2021pixelnerf,chen2021mvsnerf, suhail2022generalizable}
harnessing the power of deep learning for 
extraction of geometric priors from multiple images and associated ground-truth shapes. Such approaches extract features from each view and aggregate these features into a geometric prior.
However, these approaches are not $SE(3)$-equivariant to transformations of the frame where the priors and images are defined.
While view pooling or calculating variance \cite{xu2019disn, yang2022fvor, saito2019pifu, yu2021pixelnerf, chen2021mvsnerf} can be used to aggregate features and tackle  equivariance, view pooling discards the rich geometric information contained in a multiple view setup. 

In this paper, we address the problem of learning geometric priors that are $SE(3)$-equivariant with respect to transformations of the reference coordinate frame.
We argue that all information needed for tasks like novel view rendering or 3D reconstruction is contained in the light field~\cite{bergen1991plenoptic,levoy1996light}. 
Our input is a light field, a function defined on oriented rays in 3D whose values can be the radiance or features extracted from pixel values. We will use the term light field, and we will be specific when it is a radiance field or a feature field. Images are discrete samples of this field: the camera position determines which rays are sampled, while the camera orientation leaves the sample of the light field unchanged up to pixel discretization. We model the light field as a field over a homogeneous space of $SE(3)$, the ray space $\mathcal{R}$ parameterized by the Pl{\"u}cker coordinates. We define a convolution in the continuous ray space as an equivariant convolution on a homogeneous space \cite{cohen2019general}. In Sec. \ref{two_conv}, by varying the output domain of the convolution, we introduce equivariant convolutions from the ray space to the ray space and from the ray space to the $3D$ Euclidean space. 
Since our features are not limited to scalar values, we will draw upon the tools of tensor field networks and representation theory, discussed in detail in the Appendix. We study how the group action of $SE(3)$ on $\mathcal{R}$, the stabilizer group for $\mathcal{R}$, and how $SE(3)$ transforms the feature field over $\mathcal{R}$.
We then focus on developing the equivariant convolution in $\mathcal{R}$, providing analytical solutions for the kernels with the derived constraints in convolution from $\mathcal{R}$ to $\mathcal{R}$ and from $\mathcal{R}$ to $\mathbb{R}^3$, respectively. Meanwhile, we make the kernel locally supported without breaking the equivariance.

The constraint of the kernel limits the expressiveness of equivariant convolution when used without a deep structure. In Sec \ref{equi_tr_over_rays}, we introduce an  equivariant transformer in $\mathcal{R}$. The equivariant transformer generates the equivariant key, query, and value by leveraging the kernel derived in the convolution, resulting, thus, in invariant attention weights and, hence, equivariant outputs. We provide a detailed derivation of two cases of cross-attention: the equivariant transformer from $\mathcal{R}$ to $\mathcal{R}$ and the equivariant transformer from $\mathcal{R}$ to $\mathbb{R}^3$. In the first case, the features that generate the key and value are attached to source rays, while the feature generating the query is attached to the target ray. In the second case, the feature generating the query is attached to the target point.

We demonstrate the composition of equivariant convolution and transformer modules in the tasks of
$3D$ reconstruction from multi-views and novel view synthesis given the multi-view features. The inputs consist of finite sampled radiance fields or finite feature fields, while our proposed equivariant convolution and transformer are designed for continuous light fields.
If an object or a scene undergoes a rigid transformation and is resampled by the same multiple cameras, the $SE(3)$ group action is not transitive in the light field sample. This lack of transitivity can significantly impact the computation of equivariant features, mainly because the views are sparse, unlike densely sampled point clouds. Object motion introduces new content, resulting in previously non-existing rays in the light field sampling. Hence, our equivariance is an exact equivariance with respect to the choice of coordinate frame. 
In the 3D reconstruction task, we experimentally show that equivariance is effective for small camera motions or arbitrary object rotations and generally provides more expressive representations. 
In the $3D$ object reconstruction application, we first apply an equivariant convolutional network in ray space to obtain the equivariant features attached to rays. We then apply equivariant convolution and equivariant transformer from $\mathcal{R}$ to $\mathbb{R}^3$ to obtain equivariant features attached to the query point, which are used to calculate the signed distance function (SDF) values and ultimately reconstruct the object. %
In the generalized rendering task, our model queries a target ray and obtains neighboring rays from source views. We then apply an equivariant convolution and transformer over the rays to get the features and colors of the points along the ray (a special light field type, see Sec. \ref{neural_rendering}) and then apply an equivariant transformer over these points to get the density required for volumetric rendering. 

We summarize here our main contributions:

(1)  We model the ray space as a homogeneous space with $SE(3)$ as the acting group, and we propose the $SE(3)$-equivariant generalized convolution as the fundamental operation on a light field whose values may be radiance or features. We derive two $SE(3)$-equivariant convolutions, both taking input ray features and producing output ray features and point features, respectively.

(2) To enhance the feature expressiveness, we extend the equivariant convolution to an equivariant transformer in $\mathcal{R}$, in particular, a transformer from $\mathcal{R}$ to $\mathcal{R}$ and a transformer from $\mathcal{R}$ to $\mathbb{R}^3$.

(3) We adapt and compose the equivariant convolution and transformer module for $3D$ reconstruction from multiple views and generalized rendering from multi-view features. The experiments demonstrate the equivariance of our models.

\vspace{-6pt}
\section{Related Work}
\vspace{-4pt}
\paragraph{Equivariant Networks}
Group equivariant networks \cite{cohen2016group, worrall2017harmonic,weiler2019general,thomas2018tensor,weiler20183d,chen2021equivariant, deng2021vector,cohen2018spherical,esteves2018learning,esteves2020spin,esteves2019equivariant} provide deep learning pipelines that are 
equivariant by design with respect to group transformations of the input. While inputs like point clouds, 2D and 3D images, and spherical images have been studied extensively, our work is the first, as far as we know, to study equivariant convolution and cross-attention on light fields.
 The convolutional structure on homogeneous spaces or groups is sufficient and necessary for equivariance with respect to compact group  actions as proved in \cite{cohen2019general, aronsson2022homogeneous,kondor2018generalization}. 
 Recently, \citet{cesa2021program,xu2022unified} provided a uniform way to design the steerable kernel in an equivariant convolutional neural network on a homogeneous space using Fourier analysis of the stabilizer group and the acting group, respectively, while \citet{finzi2021practical} proposed a numerical algorithm to compute a kernel by solving the linear  equivariant map constraint. 
For arbitrary Lie groups, \citet{finzi2020generalizing,macdonald2022enabling,bekkers2019b} designed the uniform group convolutional neural network. The fact that any $O(n)$ equivariant function can be expressed in terms of a collection of scalars is shown in \cite{villar2021scalars}. For general manifolds, \citet{cohen2019gauge, weiler2021coordinate} derived the general steerable kernel from a differential geometry perspective, where the group convolution on homogeneous space is a special case. The equivalent derivation for the light field is in the Appendix. %
Recently, equivariant transformers  drew increasing attention, in particular for 3D point cloud analysis and reconstruction \cite{fuchs2020se,satorras2021n,chatzipantazis2022se,brandstetter2021geometric}.  A general equivariant self-attention mechanism for arbitrary groups was proposed in \cite{romero2020group,romero2020attentive}, while  an equivariant transformer model for Lie groups was introduced in \citet{hutchinson2021lietransformer}. We are the first to propose an equivariant  attention model in the space of rays in 3D.
\vspace{-8pt}
\paragraph{Light Field and Neural Rendering from Multiple Views}
The plenoptic function introduced in perception \cite{bergen1991plenoptic} and later in graphics \cite{levoy1996light} brought a new light into the scene representation problem and was directly applicable to the rendering problem. %
 Recently, learning-based light field reconstruction \cite{mildenhall2019local, kalantari2016learning, bemana2020x,wu2021revisiting, srinivasan2017learning, attal2022learning} became increasingly popular for novel view synthesis, while \cite{sitzmann2021light, suhail2022light,suhail2022generalizable} proposed non-equivariant networks in the ray space. Due to the smaller dimension of the ray space, the networks in the ray space are more efficient compared to neural radiance fields \cite{mildenhall2021nerf}, which leverages volumetric rendering. Several studies  \cite{yu2021pixelnerf,wang2021ibrnet,suhail2022generalizable, sitzmann2021light, chen2021mvsnerf,long2022sparseneus,chen2023explicit,huang2023local,varma2022attention}  concentrate on generalizable rendering. These works are similar to ours in that they obtain the $3D$ prior from the $2D$ images,  but they are not equivariant since they explicitly use the coordinates of the points or the rays in the network.
\vspace{-8pt}
\paragraph{Reconstruction from Multiple Views}
Dense reconstruction from multiple views is a well-established 
field of computer vision with advanced results even before the introduction of deep learning \cite{furukawa2015multi}. Such approaches cannot take advantage of shape priors and need a lot of views to provide a dense reconstruction. 
Deep learning enabled semantic reconstruction, i.e., the reconstruction from single or multiple views by providing the ground-truth 3D shape during training
 \cite{choy20163d, xie2019pix2vox,xie2020pix2vox++,mescheder2019occupancy}. These approaches decode the object from a global code without using absolute or relative camera poses. Regression of absolute or relative poses applied in 
  \cite{kar2017learning, yang2021deep,tulsiani2018multi,xu2019disn, yang2022fvor,tyszkiewicz2022raytran,bautista2021generalization,saito2019pifu,jiang2022few,du2023learning} is non-equivariant.

\vspace{-8pt}
\section{Method}
\vspace{-4pt}
\subsection{Equivariant Convolution in Ray Space}
\label{two_conv}
The ray space is the space of oriented light rays. As introduced in App. Ex.~\ref{SE3_ex_homo_ray_space}, we use Pl{\"u}cker coordinates to parameterize the ray space $\mathcal{R}$: for any ray $x \in \mathcal{R}$, $x$ can be denoted as $(\bm{d},\bm{m})$,  where $\bm{d} \in \mathbb{S}^2$ is the direction of the ray, and $\bm{m} = \bm{x} \times \bm{d}$ is the moment of the ray with $\bm{x}$ being a point on the ray. Then  any $g=(R,\bm{t}) \in SE(3)$  acts on the the ray space as:
\begin{align}
     \label{group action ray space}
    gx=g(\bm{d},\bm{m})=(R\bm{d},R\bm{m}+\bm{t}\times (R\bm{d})).
\end{align}
The ray space $\mathcal{R}$ is a homogeneous space with a transitive group action by $SE(3)$.
Given the origin in the homogeneous space  as $\eta =  ([0,0,1]^T, [0,0,0]^T)$ (the line representing $z$-axis), the stabilizer group $H$ that leaves $\eta$ unchanged is $SO(2)\times \mathbb{R}$ (the rotation around and translation along the ray). The ray space is, thus, isomorphic to the quotient space $\mathcal{R} \cong SE(3)/(SO(2)\times \mathbb{R})$.
We parameterize the stabilizer group $H$ as $H=\left\{(\gamma,t) |\gamma \in[0, 2\pi), t \in \mathbb{R} \right\}$.

We follow the generalized convolution derivation for other homogeneous spaces in \cite{cohen2019general}, which requires the use of principal bundles, section maps, and twists \cite{gallier2020differential} explained in the appendix section \ref{principal_bundle} and onwards. 
$SE(3)$ can be viewed as the principal $SO(2)\times \mathbb{R}$-bundle, where we have the projection $p: SE(3) \rightarrow \mathcal{R}$, for any $g \in SE(3)$, $p(g)= g\eta$;  a section map $s: \mathcal{R} \rightarrow SE(3)$ can be defined such that $p\circ s = id_{\mathcal{R}}$.  In App.~\ref{SE3_ex_sec_ray_space}, we further elaborate on how we define the section map from the ray space to $SE(3)$ in our model. Generally, the action of $SE(3)$ induces a twist as $gs(x) \neq s(gx)$. The twist can be characterized by the twist function $\text{h}: SE(3) \times \mathcal{R} \rightarrow  SO(2)\times \mathbb{R}$, $gs(x) = s(gx)\text{h}(g,x)$,  we provide the twist function in our model and its visualization in App.~\ref{SE3_ex_sec_ray_space}. 

\begin{figure}[t]
\centering
\begin{minipage}[t]{0.45\textwidth}
\centering
\includegraphics[width=7cm]{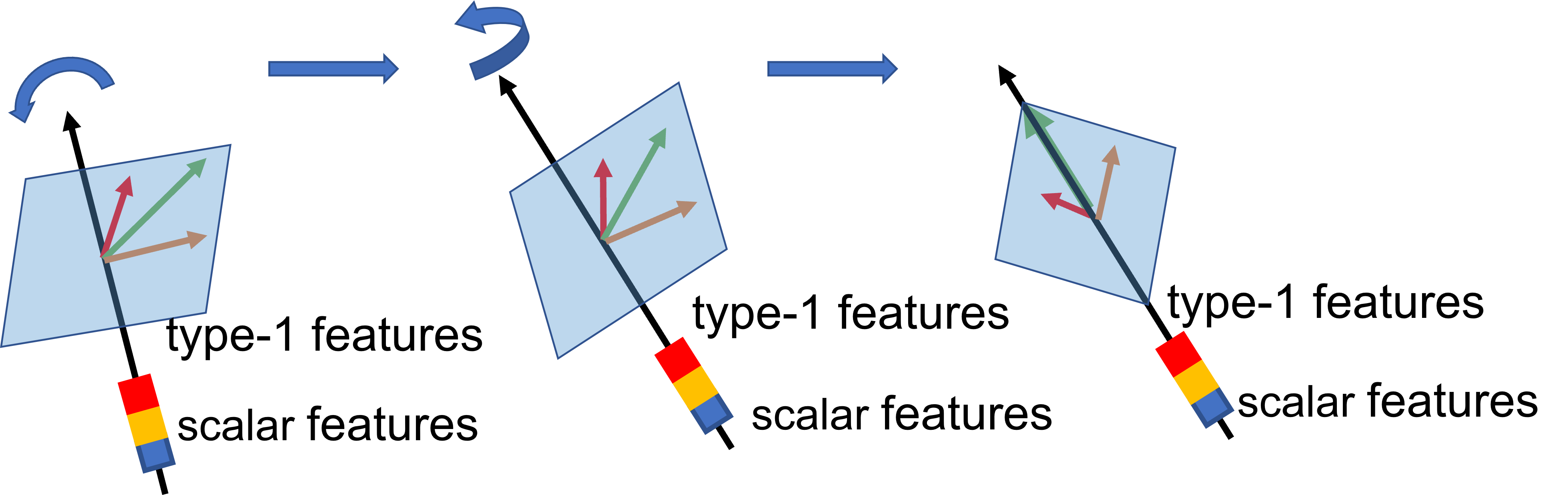}
\caption{Feature attached to rays: we show the scalar feature and type-1  feature. When $\rho_2$ is the trivial representation, tensor features can be viewed in the plane orthogonal to the ray (the blue plane). When rotations act on the feature field, the scalar feature only changes position as attached to the rays: $(\mathcal{L}_gf)(x) = f(g^{-1}x)$; while the type-1 feature changes position and itself is rotated: $(\mathcal{L}_gf)(x) = \rho(\text{h}(g^{-1},x)^{-1})f(g^{-1}x)$, where $\rho (\gamma,t) =e^{i\gamma}$.}
\label{fig:ray_features_paper}
\end{minipage}
\hspace{0.3cm}
\begin{minipage}[t]{0.42\textwidth}
\centering
\includegraphics[width=5cm]{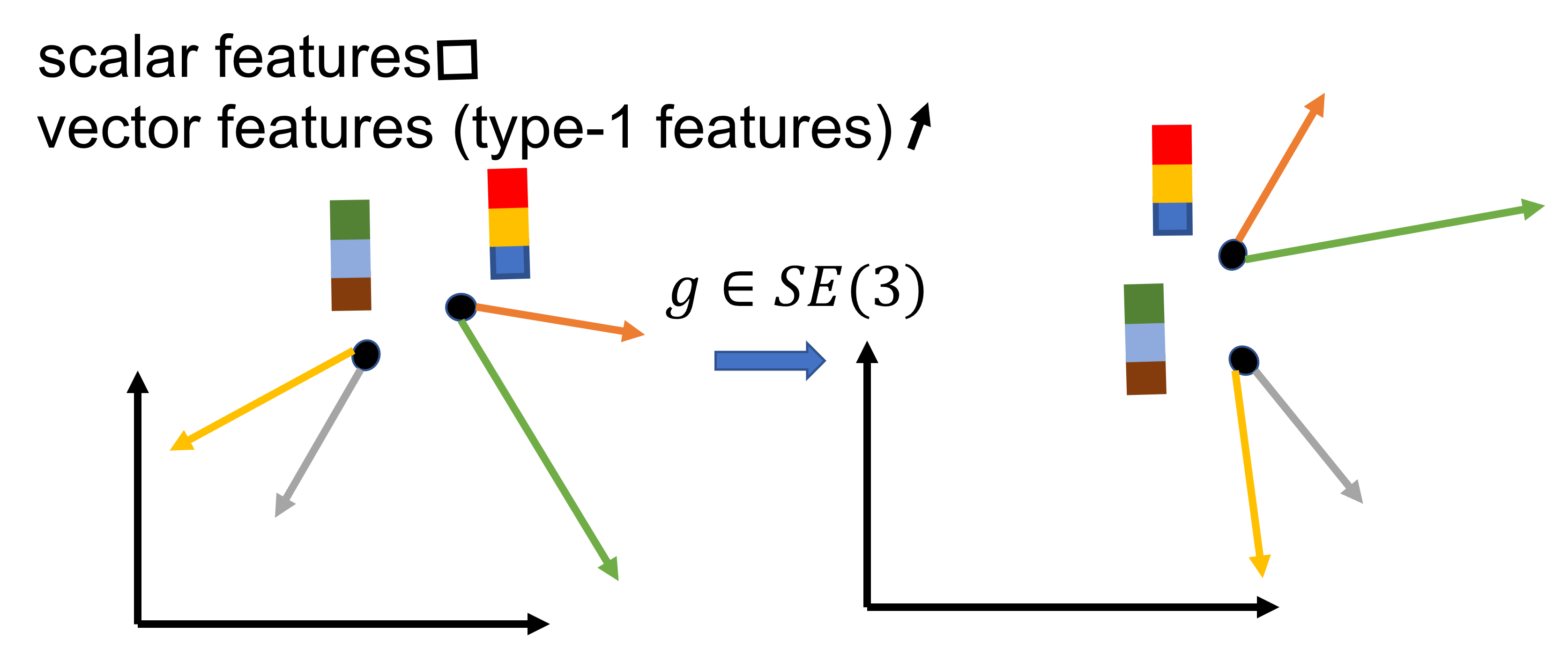}
\caption{Features attached to points: we show scalars and vectors (type-1 features). The black dot in the figure is the point, and the square and the vectors are the scalar features and type-1  features attached to the point. When $g \in SE(3)$ acts on the feature field, we will see that the scalars are kept the same while the attached position is rotated, and the vector features change their position and alter their direction.}
\label{figure:point_features}
\end{minipage}
\end{figure}
\vspace{-4pt}
\subsubsection{Convolution from Rays to Rays}
\label{equiconv}
\vspace{-2pt}
To define convolution on a light field $f: \mathcal{R} \rightarrow V$, we first need to define the $SE(3)$ group action on the values of that field. Since the group action will be on a vector space $V$, we will use the corresponding group representation of the stabilizer group $\rho: SO(2) \times \mathbb{R} \rightarrow GL(V)$, see details in App. Sec. \ref{associated vector}.
For example, a light field can be a radiance field $f$ that maps the ray space of oriented rays to their observed radiance (RGB) $f: \mathcal{R} \rightarrow \mathbb{R}^3$ which is a concatenation of three scalar fields over $\mathcal{R}$.  The group representation $\rho$ in this case is the identity 
and $g \in SE(3)$ acts on the radiance field $f$ as $(\mathcal{L}_gf)(x) = f(g^{-1}x)$, shown as the scalar features in Fig. \ref{fig:ray_features_paper}.
Given that the stabilizer $H = SO(2) \times \mathbb{R}$ is a product group, 
the stabilizer representation can be written as the product $\rho(\gamma, t)= \rho_1(\gamma)\otimes\rho_2(t)$, where $\rho_1$ is the group representation of $SO(2)$ and $\rho_2$ is the group representation of $\mathbb{R}$. 
If the light field is a feature field (Fig. \ref{fig:ray_features_paper}) with 
 $\rho_2$ being the identity representation and $\rho_1$ corresponding to a type-1 field, $\rho_1(\gamma) = e^{i\gamma}$, then type-1 features change position and orientation when $g \in SE(3)$ acts on it.
Having explained the examples of scalar (type-0) and type-1 fields, we introduce the action on any feature field $f$ as 
\cite{cohen2019general}:
  \begin{align}
  (\mathcal{L}_gf)(x) = \rho(\text{h}(g^{-1},x)^{-1})f(g^{-1}x),\label{action on filed over ray}
  \end{align}
where $\rho$ is the group representation of $SO(2) \times \mathbb{R}$ corresponding to the space $V$, determined by the field type of $f$, and $\text{h}$ is the twist function introduced by $SE(3)$ as shown in App. Ex. \ref{SE3_ex_sec_ray_space}. 
The convolution as stated in App. Sec. \ref{convolution sec} and \cite{cohen2019general} is then defined as 
\begin{align}
f^{l_{out}}(x)=\int_{\mathcal{R}}\kappa(s(x)^{-1}y)\rho_{in}(\text{h}(s(x)^{-1}s(y)))f^{l_{in}}(y)dy,
    \label{paper:convolution}
\end{align}
\begin{wrapfigure}{l}{0.4\linewidth}
        \includegraphics[width=0.25\textheight]{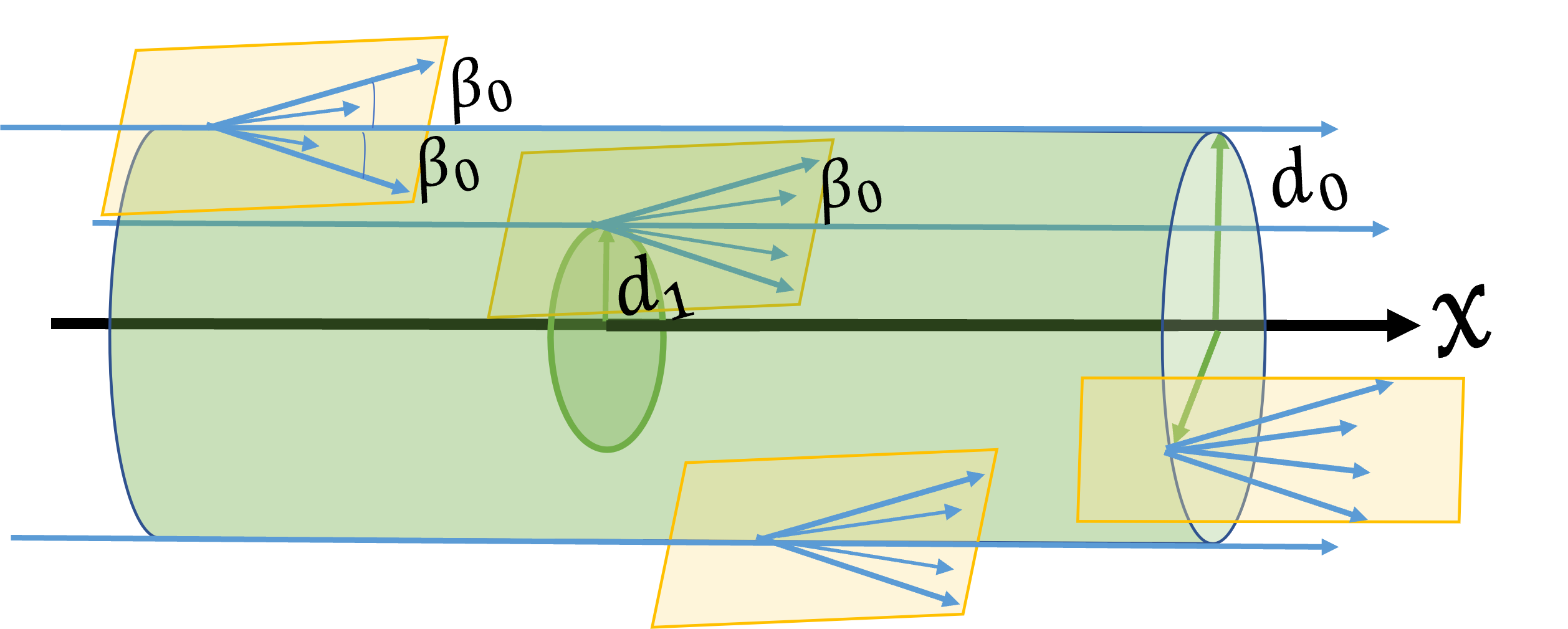}
         \caption{Neighborhood of a ray $x$ in the convolution. }
   \label{fig:lightconv}
\end{wrapfigure}
where $\text{h}(g)$ is the simplified form of the twist $\text{h}(g,\eta)$. Eq.\ref{paper:convolution} is equivariant to $SE(3)$ if and only if the convolution kernel $\kappa$ satisfies that $\kappa(hx)=\rho_{out}(h)\kappa(x)\rho_{in}(\text{h}^{-1}(h,x))$, where $\rho_{in}$ and $\rho_{out}$ are the group representations of $SO(2) \times \mathbb{R}$ corresponding to the input feature type $l_{in}$ and output feature type $l_{out}$, respectively. We derive the solutions of the kernel in the App. Ex. \ref{filter_solution}.
\vspace{-8pt}
\paragraph{Local kernel support} 
The equivariance stands even if we constrain the kernel to be local.
When $x=(\bm{d}_x,\bm{m}_x)$ meets the condition that $\angle(\bm{d}_x,[0,0,1]^T) \leq \beta_0$ and $d(x,\eta) \leq d_0$, $\kappa(x) \neq 0$, this local support  will not violate the constraint that $\kappa(hx)=\rho_{out}(h)\kappa(x)\rho_{in}(\text{h}^{-1}(h,x))$.
Then, convolution in Eq. \ref{paper:convolution} is accomplished over the neighbors only as visualized in Fig.~\ref{fig:lightconv}. In Fig.~\ref{fig:lightconv}, any ray $y=(\bm{d}_y, \bm{m}_y)$ (denoted in blue) in the neighborhood  of a ray $x=(\bm{d}_x,\bm{m}_x)$ will go through the cylinder with $x$ as the axis and $d_0$ as the radius since $d(x,y)\leq d_0$. Moreover, for any $y$, $\angle(\bm{d}_y,\bm{d}_x) \leq \beta_0$. Any ray $y \in \mathcal{N}(x)$ is on one tangent plane of a cylinder with $x$ as the axis and $d(x,y)$ as the radius when $d(x,y)>0$.

\vspace{-6pt}
\subsubsection{Convolution from Rays to Points}
\label{equifuse}
\vspace{-2pt}
In applications such as $3D$ reconstruction, key point detection, and $3D$ segmentation, we expect the output to be the field over $\mathbb{R}^3$. $\mathbb{R}^3$ is also a homogeneous space of $SE(3)$ like the ray space $\mathcal{R}$, with the stabilizer group as $SO(3)$, as stated in App. Ex. \ref{SE3_ex_homo_R3}. Using a convolution, we will define an equivariant map from light fields (fields on $\mathcal{R}$) to fields on $\mathbb{R}^3$.
We denote with $H_1$ and $H_2$ 
 the stabilizer groups for the input and output homogeneous spaces, respectively, i.e., $SO(2)\times \mathbb{R}$ and $SO(3)$ in this case. 
 As shown in the App. Ex. \ref{SE3_ex_sec_R3}, we can choose the section map $s_2: \mathbb{R}^3 \rightarrow SE(3)$: $s_2(\bm{x})=(I,\bm{x})$ for any $\bm{x} \in \mathbb{R}^3$ and I is the identity matrix.  Following \cite{cohen2019general}, the convolution from rays to points becomes: 
\begin{align*}
f_2^{l_{out}}(x)=\int_{\mathcal{R}}\kappa(s_2(x)^{-1}y)\rho_{in}(\text{h}_1(s_2(x)^{-1}s_1(y)))f^{l_{in}}_1(y)dy,
\end{align*}
where 
$\text{h}_1$ is the twist function corresponding to section $s_1:\mathcal{R} \rightarrow SE(3)$ defined aforementioned, $\rho_{in}$ is the group representation of $H_1$ ($SO(2) \times \mathbb{R}$)  corresponding to the feature type $l_{in}$. The subscripts of the input and output feature denote the homogeneous spaces they are defined on. The convolution is equivariant if and only if the kernel $\kappa$ satisfies that $\kappa(h_2x)=\rho_{out}(h_2)\kappa(x)\rho_{in}(\text{h}_1^{-1}(h_2,x))$ for any $h_2 \in H_2$, where $\rho_{out}$ is the group representation of $H_2$ ($SO(3)$) corresponding to the feature type $l_{out}$. Fig. \ref{figure:point_features} visualizes the scalar feature ($l_{out}=0$) and vector feature ($l_{out}=1$) attached to one point, 
offering an intuitive understanding of the feature field over $\mathbb{R}^3$.

In 3D reconstruction,
 $f^{l_{in}}$ is the scalar field over $\mathcal{R}$, i.e., $\rho_{in}=1$. 
 The convolution is simplified to  $f_2^{l_{out}}(x)=\int_{G/H_1}\kappa(s_2(x)^{-1}y)f^{l_{in}}_1(y)dy$ and the corresponding constraint becomes $\kappa(h_2x)=\rho_{out}(h_2)\kappa(x)$.  App. Ex. \ref{filter_solution2point} provides analytical kernel solutions.

\vspace{-8pt}
\subsection{Equivariant Transformer over Rays}
\label{equi_tr_over_rays}
\vspace{-4pt}
We can extend the equivariant convolution to the equivariant transformer model. In general, the equivariant transformer can be formulated as:
\begin{align}
f^{out}_2(x) = &\sum_{y \in \mathcal{N}(x)} \frac{exp(\langle f_q(x,f^{in}_2), f_k(x,y,f^{in}_1)\rangle)}{\sum_{y \in \mathcal{N}(x)}exp(\langle f_q(x,f^{in}_2) f_k(x,y,f^{in}_1)\rangle)}f_v(x,y,f^{in}_1), 
\label{paper:transformer}
\end{align}

where the subscript $1$ denotes the homogeneous space $M_1\cong G/H_1$ of the feature field $f^{in}_1$ that generates the key and value in the transformer; the subscript $2$ denotes the homogeneous space $M_2 \cong G/H_2$ of the  feature field $f^{in}_2$ that generates query in the transformer, which is also the homogeneous space of the output feature $f^{out}_2$; %
$x$ and $y$ represent elements in the homogeneous spaces $M_2$ and $M_1$, respectively, where $y \in \mathcal{N}(x)$ indicates that the attention model is applied over $y$, the neighbor of $x$ based on a defined metric. $f_k$, $f_q$, and $f_v$ are constructed equivariant keys, queries, and values in the transformer. $f_k$ and $f_v$ are constructed by equivariant kernel $\kappa_k$ and $\kappa_v$ while $f_q$ is constructed through an equivariant linear map, see App. Sec.\ref{construction of features in transformer} for detailed construction.

\begin{figure}[htbp]
\centering
\begin{minipage}[t]{0.55\textwidth}
\centering
\includegraphics[width=6cm]{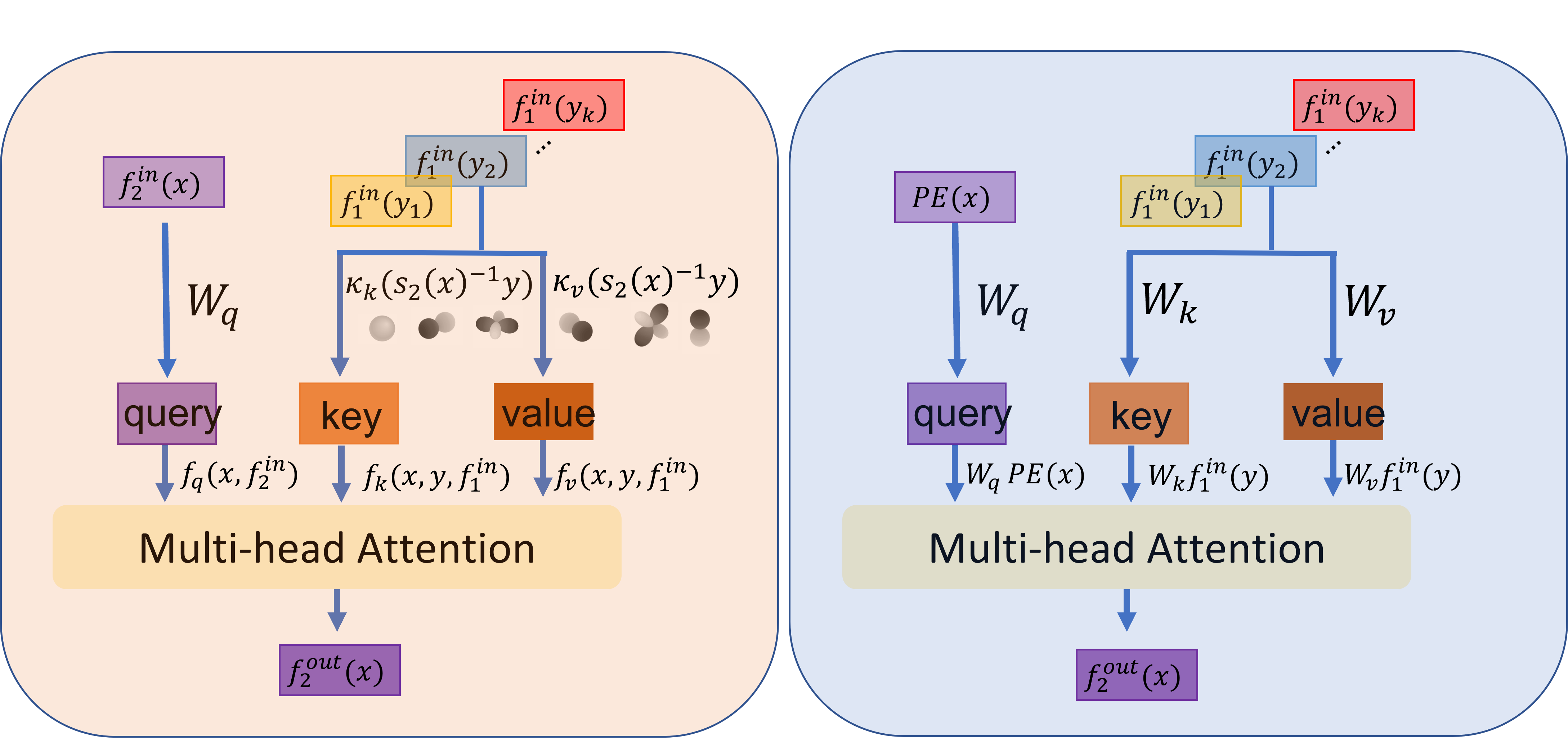}
\caption{ In the equivariant transformer (L), positional encoding is not directly used due to its lack of equivariance. Instead, the relative position within the kernel is utilized. %
To generate the query $f_q$, we multiply the feature $f^{in}_2(x)$ (pre-existing or yielded by convolution) attached to $x$ (in $\mathcal{R}$ or $\mathbb{R}^3$, depending on the task) by the designed equivariant linear matrix $W_q$ (see App. \ref{construction of features in transformer}). %
The key $f_k$ and value $f_v$ are constructed using designed equivariant kernels $\kappa_k$ and $\kappa_v$. 
The transformer is equivariant due to equivariant $f_k$, $f_q$, and $f_v$.%
The conventional transformer (R) uses point position encoding for the query feature and obtains the query, key, and value through nonequi-conventional linear mappings. %
}
\label{figure:comparison}
\end{minipage}
\hspace{0.3cm}
\begin{minipage}[t]{0.35\textwidth}
\centering
\includegraphics[width=4cm]{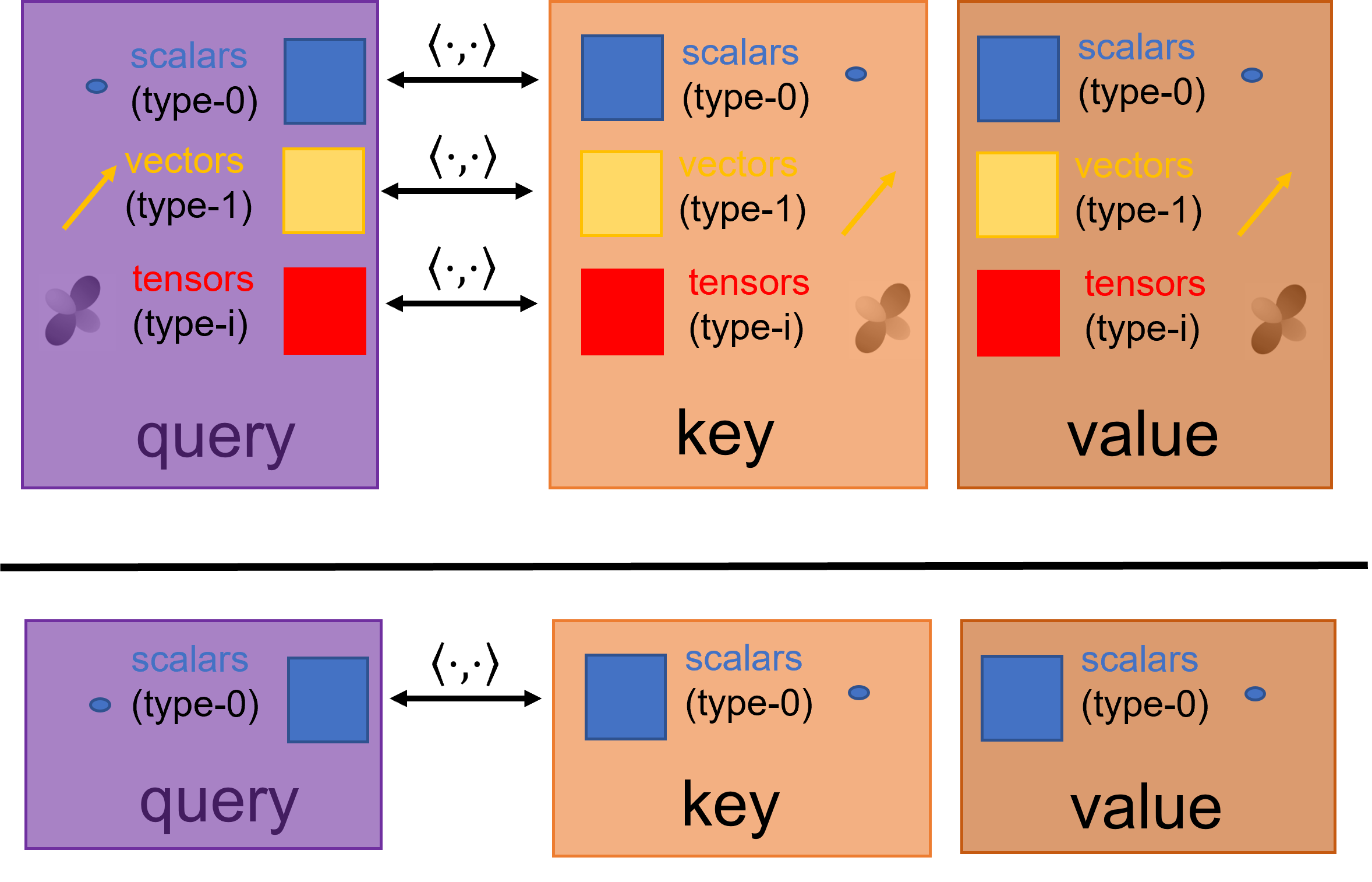}
\caption{%
In the equivariant transformer (U), the query, key, and value are equivariant and can be composed of different types of features; they can be scalars, vectors, or higher-order tensors. The inner product, determined by the feature type, should apply to the same type of features. In contrast, the feature in a conventional transformer (D) is not equivariant, it does not contain vectors and tensors, and the inner product is conventional.}
\label{figure:comparison_multihead}
\end{minipage}
\end{figure}

When the transformer is a self-attention model, homogeneous space $M_1$ and $M_2$ are the same since $f^{in}_2=f^{in}_1$. The above equivariant transformer could be applied to the other homogeneous space other than $\mathcal{R}$, $\mathbb{R}^3$, and acting group other than $SE(3)$. This paper presents the equivariant cross-attention model over rays, i.e., $M_1$ is $\mathcal{R}$. When the transformer is the cross-attention from rays to rays, $M_2$ is also $\mathcal{R}$, the equivariant kernel $\kappa_k$ and $\kappa_v$ is the convolution kernel we derived in convolution from rays to rays in Sec. \ref{equiconv}. %
When the transformer is the cross-attention from rays to points, $M_2$ is $\mathbb{R}^3$, the equivariant kernel $\kappa_k$ and $\kappa_v$ is the convolution kernel we derived in convolution from rays to points in Sec. \ref{equifuse}. %
With the construction in App. Sec. \ref{construction of features in transformer}, we claim that the transformer from rays to rays or from rays to points, as shown in the equation \ref{paper:transformer},  is equivariant. The proof is provided in App. Sec. \ref{proof_trans}.

To better understand the equivariant transformer in this paper, we visualize the comparison of our equivariant cross-attention transformer and conventional transformer shown in Fig. \ref{figure:comparison}. Meanwhile, as stated in App. Sec.\ref{construction of features in transformer}, key, query, and value are generally composed of different types of features and are multi-channel, allowing for the multi-head attention mechanism. In Fig. \ref{figure:comparison_multihead}, we visualize the comparison of the equivariant multi-head attention module from rays to points with the conventional multi-head attention module. The attention module from rays to rays follows a similar concept but with variations in the feature types due to the differing group representations of $SO(2) \times \mathbb{R}$ and $SO(3)$.

We will show two 3D multi-view applications of the proposed convolution and transformer: $3D$ reconstruction and generalized neural rendering.  For each application, we define the specific equivariance and present the corresponding pipeline. %

\vspace{-12pt}
\subsection{Equivariant 3D Reconstruction \label{def_equi}}
\begin{figure*}[t]
  \centering
   \includegraphics[width=0.9\linewidth]{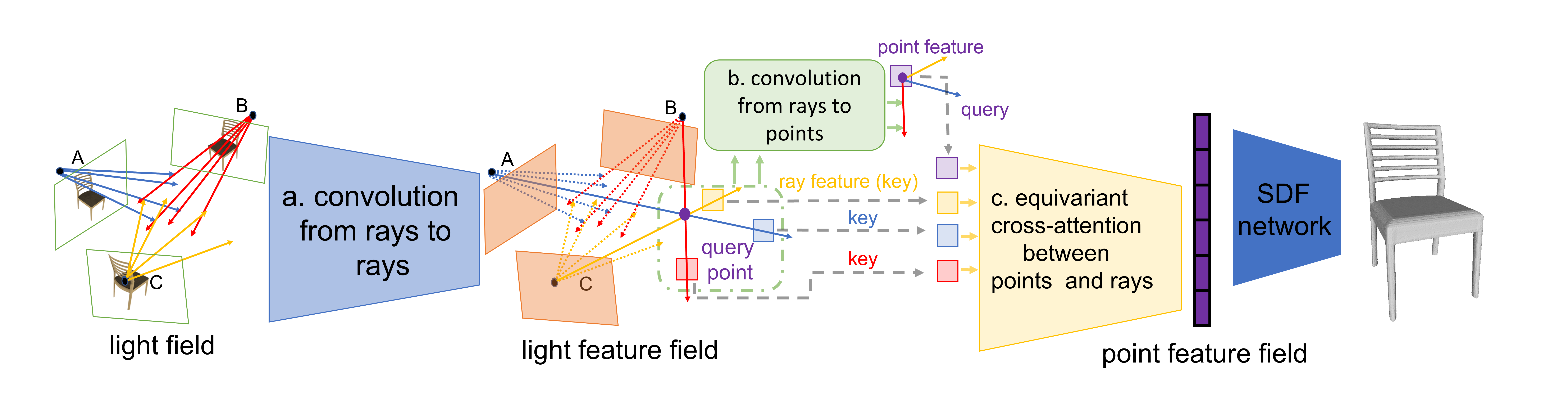}%
   \caption{The pipeline of equivariant $3D$ reconstruction: Firstly, we obtain the feature field over the ray space. Secondly, we perform an equivariant convolution from ray space to point space. Thirdly, we apply a $SE(3)$ equivariant cross-attention module to obtain a equivariant feature for a query.}
   \label{fig:pipeline}
\end{figure*}
\vspace{-4pt}

\vspace{-4pt}
The radiance field serves as the input for the $3D$ reconstruction, which ultimately generates a signed distance field (SDF) denoted by the function $e:\mathbb{R}^3 \rightarrow \mathbb{R}$. As aforementioned, the radiance field is the multi-channel scalar field over $\mathcal{R}$, while SDF is the scalar field over $\mathbb{R}^3$. A $3D$ reconstruction $\Phi: \mathcal{F} \rightarrow \mathcal{E}$, where $\mathcal{F}$ denotes the space of radiance fields and 
$\mathcal{E}$ denotes the space of signed distance fields, is equivariant when for any $g \in SE(3)$, any $x \in \mathbb{R}^3$, and any $f \in \mathcal{F}$, $\bm{\Phi( \mathcal{L}_gf)(x)= \mathcal{L}'_g(\Phi(f))(x)},$ where $\mathcal{L}_g$ and $\mathcal{L}'_g$ are group actions on the light field and the SDF, respectively. Specifically, as $f$ and $e$ are scalar fields, $(\mathcal{L}_gf)(x)=f(g^{-1}x)$ for any $x \in \mathcal{R}$, and $(\mathcal{L}'_ge)(x)=e(g^{-1}x)$ for any $x\in \mathbb{R}^3$.

 In practice, we have a finite sampling of the radiance field corresponding to the pixels of multiple views 
 $V=\left\{f(x)|x \in L_V\right\}$, where $L_V$ denotes the ray set of multi-views and $f \in \mathcal{F}$ is the radiance field induced by multi views sample from. The $3D$ reconstruction $\Phi$ is equivariant when for any $g \in SE(3)$ and any $x \in \mathbb{R}^3$: $\bm{\Phi(g \cdot V)(x)= \Phi(V)(g^{-1}x)}.$ If we denote $V$ as $(L_V,f)$, %
$g \cdot V=(g \cdot L_V, \mathcal{L}_gf)$, where $g \cdot L_V$ is  $g$ acting on the rays defined Eq. \ref{group action ray space}. 

We achieve equivariance using three steps as illustrated in Fig.~\ref{fig:pipeline}: (1) the transition from pixel colors to a feature-valued light field (equi-CNN over rays), (2) the computation of features in $\mathbb{R}^3$ from features on the ray space by equivariant convolution from ${\cal R}$ to $\mathbb{R}^3$, and (3) the equivariant transformer with the query generated by the feature on the point we want to compute SDF and key/value generated by features on rays. Note that we need (3) following (2) 
because the output feature of a single convolution layer is not expressive enough due to the constrained kernel.%
For the detailed practical adaption of the convolution and transformer in $3D$ reconstruction, please see the App. Sec. \ref{equivariant_reconstruction}, where we approximate the intra-view with $SE(2)$ equivariant convolution.

\vspace{-10pt}
\begin{figure}[htb]
  \centering
   \includegraphics[width=0.85\linewidth]{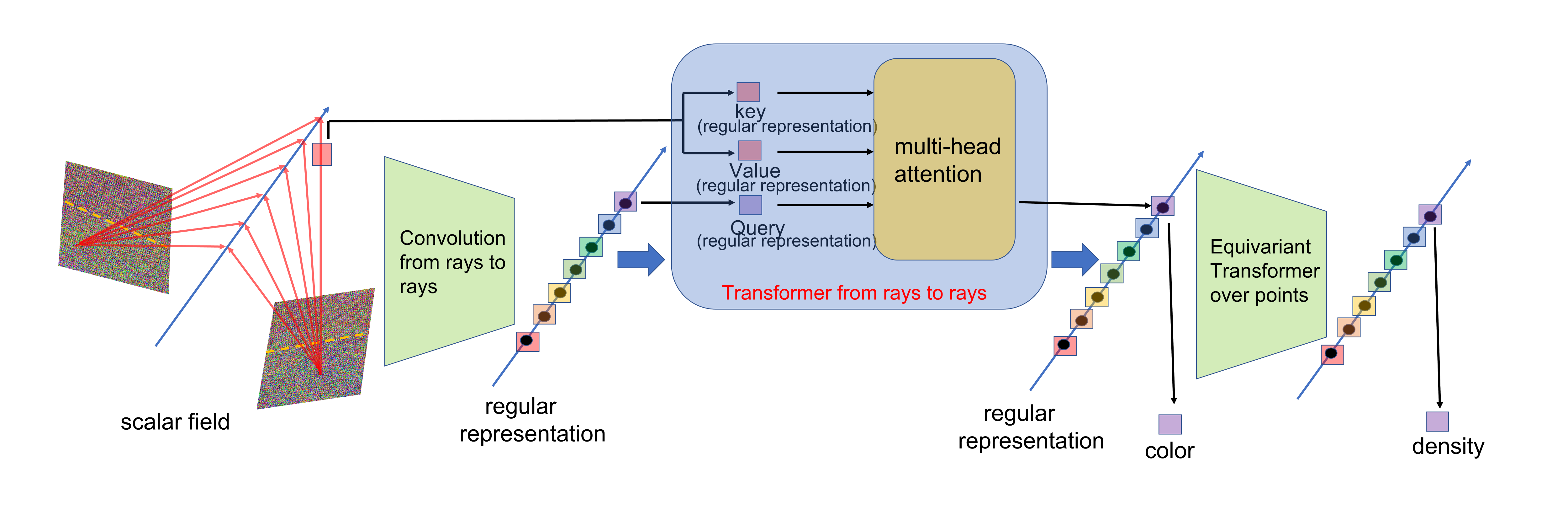}
   \caption{The pipeline of equivariant neural rendering. Firstly, we obtain the features of the points along the target ray through convolution over rays. Secondly, we apply the equivariant cross-attention module to obtain features for generating the color of the points. Finally, we use equivariant self-attention over the points along the ray to obtain features for generating the density of points.}
   \label{fig:renderpipeline}
\end{figure}
\vspace{-8pt}
\subsection{Generalized Neural Rendering}
\label{neural_rendering}
\vspace{-4pt}
The light feature field $f_{in}: \mathcal{R} \rightarrow V$ serves as the input for neural rendering, which ultimately generates the light field $f: \mathcal{R} \rightarrow \mathbb{R}^3$, a multi-channel scalar field over $\mathcal{R}$. A neural rendering $\Psi:\mathcal{I} \rightarrow \mathcal{F}$, where $\mathcal{I}$ denotes the space of the light feature fields and $\mathcal{E}$ denotes the space of the light field, is equivariant when for any $g \in SE(3)$, any $x \in \mathcal{R}$, and any $f_{in} \in \mathcal{I}$, $\bm{\Psi( \mathcal{L}_gf_{in})(x)= \Psi(f_{in})(g^{-1}x)},$%
where $\mathcal{L}_g$ is the group operator on the light feature field $f_{in}$, as shown in Eq. \ref{action on filed over ray} depending on the feature type. In the experiment of this paper, the input light feature field is scalar, i.e., $\mathcal{L}_gf_{in}(x)=f_{in}(g^{-1}x)$. Similar to reconstruction, in practice, the neural rendering $\Psi$ is equivariant when for any $g \in SE(3)$ and any $x \in \mathcal{R}$: $\bm{\Psi(g \cdot V)(x)= \Psi(V)(g^{-1}x)},$
where $V=\left\{f_{in}(x)|x \in L_V\right\}$, and if we denote $V$ as $(L_V,f_{in})$, 
then $g \cdot V=(g \cdot L_V, \mathcal{L}_gf_{in})$,

By restricting the field type of the output field over rays to have a group representation of $SO(2) \times \mathbb{R}$ as $\rho(\gamma,t) = \rho_1(\gamma)\otimes\rho_2(t)$, where $\rho_2$ is the regular representation, we can obtain the feature of points along the ray by convolution or transformer from $\mathcal{R}$ to $\mathcal{R}$. 
See App. Ex. \ref{regular_representation} for more explanation of the regular representation. Alternatively, we can obtain the desired feature by applying convolution or transformer from $\mathcal{R}$ to $\mathcal{R}$, with output features attached to the target ray corresponding to different irreducible representations of the stabilizer group. These features can be interpreted as Fourier coefficients of the function of the points along the ray. The Inverse Fourier Transform yields feature for the points along the ray. More details are in the App. Sec. \ref{rendering_ray2ray}.

The feature of the points along the ray can be used to generate density and color for volumetric rendering \cite{wang2021ibrnet, yu2021pixelnerf}, or fed into attention and pooling for the final ray feature  \cite{varma2022attention}. In this paper, we opt to generate the density and color and utilize volumetric rendering, which can be viewed as a specialized equivariant convolution from $\mathbb{R}$ to $\mathcal{R}$. Method details are available in App. Sec. \ref{equi_render}.

We achieve the equivariant rendering through three steps as shown in Fig. \ref{fig:renderpipeline}: (1) we apply equivariant convolution from rays to rays to get the equivariant feature for points along the rays, which is a specific field type over $\mathcal{R}$; %
2) to enhance the feature expressivity, we apply equivariant transformer from rays to rays to get  the color for each point; 
(3) we apply the equivariant self-attention over the points along the ray to reason over the points on the same ray, the output feature of the points will be fed to multiple perceptron layers to get the density of the points.%

\vspace{-10pt}
\section{Experiment}
\vspace{-4pt}
\subsection{3D Object Reconstruction from Multiple Views}
\vspace{-2pt}
\paragraph{Datasets and Implementation} We use the same train/val/test split of the Shapenet 
 Dataset \cite{chang2015shapenet}  and render ourselves for the equivariance test. In order to render the views for each camera, we fix eight cameras to one cube's eight corners. The cameras all point in the same  direction toward the object's center. We use the following notation to denote the variety of transformations in training and testing:  $I$ (no transformation), $Z$ (optical axis rotation), $R$ (bounded 3-dof camera rotation), $Y$ (vertical axis object rotation), $SO(3)$ (full object rotation). The details to generate the five settings are provided in App. Sec. \ref{data_generation}.  As described in App. Sec. \ref{equivariant_reconstruction}, we use $SE(2)$ equivariant CNNs to approximate the equivariant convolution over the rays. For the fusion from the ray space to the point space model, we use one layer of convolution and three combined blocks of updating ray features and $SE(3)$ transformers. For more details, please see the App. Sec. \ref{implementation_reconstruction}.
 \vspace{-10pt}
\paragraph{Results}
\begin{wrapfigure}{r}{0.4\linewidth}
\includegraphics[width=0.25\textheight]{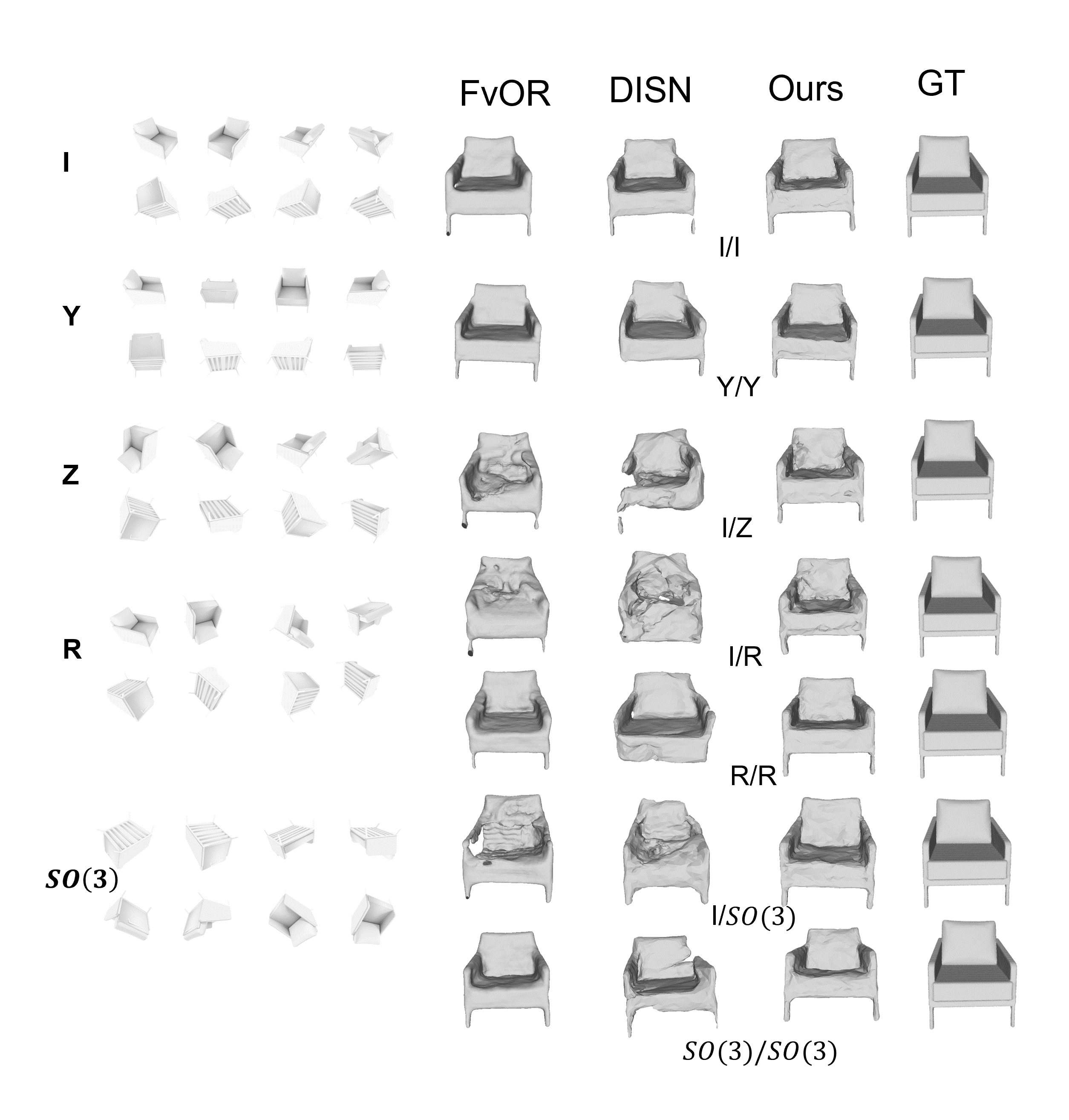}
         \caption{Qualitative results for equivariant reconstruction. Left: input views; Right: reconstruction meshes of different models and ground truth meshes the meshes show how the model is trained and tested, explained in the text.}
   \label{fig:qualitative_rec}
\end{wrapfigure}
We evaluate our model in seven experiment settings, $I/I$, $I/Z$, $I/R$, $R/R$, Y/$SO(3)$, $SO(3)/SO(3)$.  The setting A/B indicates training the model on the A setup of the dataset and evaluating it on the B setup. Following the previous works, we use \textit{IoU} and \textit{Chamfer-L1 Distance}  %
as the evaluation metric. Quantitative results are reported in table \ref{tab:equi_table}, and qualitative results are in Fig. \ref{fig:qualitative_rec}. 
We compare with two other approaches \cite{xu2019disn}, which follows a classic paradigm that queries 3D positions that are then back-projected %
to obtain image features for aggregation, %
and \cite{yang2022fvor}, which was state of the art in 3D object reconstruction from multi-views.
 Notably, we provide ground truth poses to the baselines, which originally estimate object poses. See 
App. Sec. \ref{qual_results} for more qualitative results.

\begin{table*}
  \centering
  \scalebox{0.77}{
  \begin{tabular}{|c|c|c|c|c|c|c|c|}
    \hline
    Method &\multicolumn{7}{|c|}{chair}\\
    \hline
    & I/I & I/Z & I/R & R/R&Y/Y&Y/SO(3)&SO(3)/SO(3) \\
    \hline
   Fvor w/ gt pose\cite{yang2022fvor}& 0.691/0.099& 0.409/0.253&0.398/0.257&0.669/0.113&0.687/0.103&0.518/0.194&0.664/0.114\\
   DISN w/ gt pose\cite{xu2019disn}& 0.725/0.094& 0.335/0.396&0.322/0.405 & 0.500/0.201&0.659/0.120&0.419/0.303&0.549/0.174 \\
   Ours& \textbf{0.731}/\textbf{0.090} &\textbf{0.631}/\textbf{0.130}& \textbf{0.592}/\textbf{0.137}& \textbf{0.689}/\textbf{0.105}&\textbf{0.698}/\textbf{0.102}&\textbf{0.589}/\textbf{0.142}& \textbf{0.674}/\textbf{0.113} \\
   \hline
   Method &\multicolumn{7}{|c|}{airplane}\\
   \hline
    & I/I & I/Z & I/R & R/R &Y/Y &Y/SO(3)&SO(3)/SO(3) \\
    \hline
   Fvor w/ gt pose\cite{yang2022fvor}&0.770/0.051&0.534/0.168 &0.533/0.174 & \textbf{0.766}/0.053&\textbf{0.760}/\textbf{0.052}&0.579/0.147&\textbf{0.746}/\textbf{0.056}\\
   DISN w/ gt pose\cite{xu2019disn}& 0.752/0.058& 0.465/0.173&0.462/0.171& 0.611/0.104&0.706/0.069&0.530/0.151&0.631/0.103\\
   Ours& \textbf{0.773}/\textbf{0.050}& \textbf{0.600}/\textbf{0.092}&\textbf{0.579}/\textbf{0.100} & 0.759/\textbf{0.051}&0.734/\textbf{0.052}&\textbf{0.597}/\textbf{0.101}&0.722/\textbf{0.056}\\
   \hline
    Method &\multicolumn{7}{|c|}{car}\\
   \hline
    & I/I & I/Z & I/R & R/R & Y/Y &Y/SO(3) &SO(3)/SO(3)\\
    \hline
   Fvor w/ gt pose\cite{yang2022fvor}&0.837/0.090& 0.466/0.254&0.484/0.258 & 0.816/0.107 &\textbf{0.830}/0.094&0.496/0.240&0.798/0.111\\
   DISN w/ gt pose\cite{xu2019disn}& 0.822/0.089& 0.610/0.232&0.567/0.236 &0.772/0.135 &0.802/0.098&0.614/0.205&0.769/0.123 \\
   Ours&\textbf{0.844}/\textbf{0.081}&\textbf{0.739}/\textbf{0.142} &\textbf{0.741}/\textbf{0.150}&\textbf{0.836}/\textbf{0.089} &\textbf{0.830}/\textbf{0.089}&\textbf{0.744}/\textbf{0.137}& \textbf{0.813}/\textbf{0.097}\\
   \hline
  \end{tabular}
  }
  \caption{The results for the seven experiments of 8-view $3D$ reconstruction for the ShapeNet dataset. The metrics in the cell are \textit{IoU}$\uparrow$ and \textit{Chamfer-L1 Distance}$\downarrow$. We implement \cite{yang2022fvor} and \cite{xu2019disn} ourselves on our equivariant dataset. For the performance of \cite{xu2019disn}, we follow their work to conduct the multi-view reconstruction by pooling over the feature of every view. The value of \textit{Chamfer-L1 Distance} is $\times 10$.}
  \label{tab:equi_table}
\end{table*}

In table \ref{tab:equi_table}, our model outperforms the \cite{yang2022fvor} and \cite{xu2019disn} by a large margin on $I/Z$, $I/R$, and $Y/SO(3)$ settings. Although theoretically,  our model is not equivariant to the arbitrary rotation of the object, $Y/SO(3)$ shows the robustness of our model to the object rotation and the generalization ability to some extent.

Our model outperforms other models %
for the chair and car categories in $R/R$ and $SO(3)/SO(3)$ settings while it is slightly inferior to \cite{yang2022fvor} in the airplane category. Notably, our model only requires relative camera poses, while \cite{yang2022fvor} and \cite{xu2019disn} utilize camera poses relative to the object frame, leveraging explicit positional encoding of the query point in the object frame, which is concatenated to the point feature. In addition, our model performs better in several experiments in $I/I$ and $Y/Y$ settings. It %
can be attributed to the $SE(3)$ equivariant attention model, which considers scalar features and ray directions. See App. Sec. \ref{discussion_result} for more discussion of the results.

We provide an ablation study of the effectiveness of $SE(2)$ CNNs, equivariant convolution, transformer, and type-1 feature (vector feature) in our model. Meanwhile, we compare our method with the model that explicitly encodes the direction of rays. Please see the App. Sec. \ref{ablation_study} for details.
 \vspace{-8pt}
\subsection{Neural Rendering}
\vspace{-8pt}
\paragraph{Datasets and Implementation} We use the same training and test dataset as in \cite{wang2021ibrnet}, which consists of both synthetic data and real data.
Two experiment settings illustrate our model's equivariance: $I/I$ and $I/SO(3)$. $I/I$ is the canonical setting, where we train and test the model in the same canonical frame defined in the dataset. In the $I/SO(3)$ setting, we test the model trained in the conical frame under arbitrarily rotated coordinate frames while preserving relative camera poses and the relative poses between the camera and the scene, thereby preserving the content of the multiple views. Each individual view itself is not transformed. Note that this experiment's $SO(3)$ setup differs from the $R$ and $SO(3)$ setups used in the reconstruction. Further details and discussions on this difference can be found in App. Sec. \ref{rendering_setting_discussion}. 
\begin{wrapfigure}{l}{0.4\linewidth}
        \includegraphics[width=0.25\textheight]{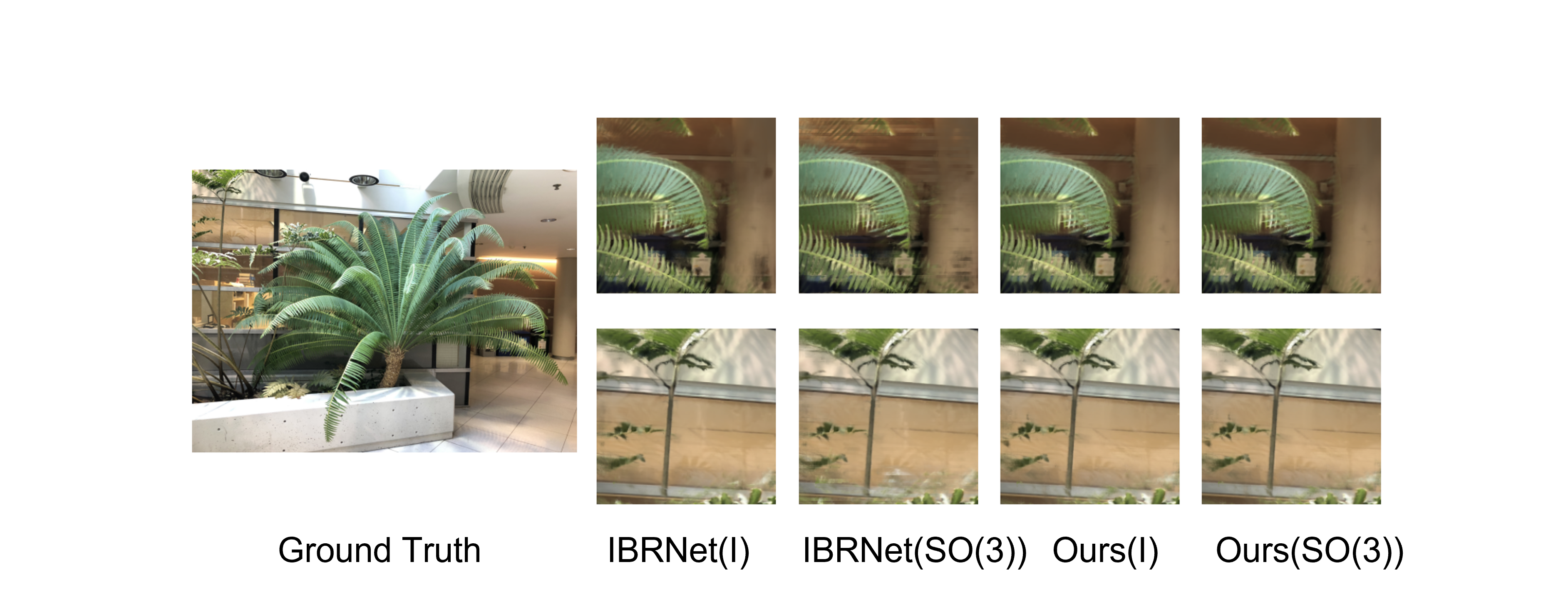}
         \caption{Qualitative Result for Generalized Rendering. We observe a performance drop for IBRNet from $I$ to $SO(3)$, while ours are robust to the rotation.}
   \label{fig:rendering_qua}
\end{wrapfigure}
Our model architecture is based on IBRNet\cite{wang2021ibrnet}, with view feature aggregation and ray transformer components modifications. Specifically, we replace the view feature aggregation in \cite{wang2021ibrnet} with the equivariant convolution and transformer over rays and the ray transformer part with the equivariant self-attention over the points along the ray. For more information of the implementation details, please refer to App. Sec. \ref{implementation_rendering}.

\vspace{-8pt}
\paragraph{Results}

We compare with IBRNet on $I/I$ and $I/SO(3)$ settings to show that our proposed models can be embedded in the existing rendering framework and achieve equivariance. Following previous works on novel view synthesis, our evaluation metrics are PSNR, SSIM, and LPIPS \cite{zhang2018unreasonable}.  In the $I/SO(3)$ test period, we randomly rotate each data six times and report the average metrics. Meanwhile,  we record the max pixel variance and report the average value. We show a qualitative result in Fig. \ref{fig:rendering_qua}. In table \ref{tab:rendering_table}, our model performs comparably with IBRNet\cite{wang2021ibrnet} in $I/I$ setting without performance drop in $I/SO(3)$ setting. The slight decrease in PSNR/SSIM/LPIPS for IBRNet from $I/I$ to $I/SO(3)$ can be attributed to the training process involving multiple datasets with different canonical frames, which includes transformation augmentation and makes the model more robust to coordinate frame changes. 
Additionally, conventional metrics like PSNR/SSIM may not directly capture image variations. Therefore, we introduce an additional metric, pixel variance, to illustrate better the changes. We observe that IBRNet \cite{wang2021ibrnet} exhibits pixel variance for different rotations, whereas our approach remains robust to rotation. Our method performs comparably with IBRNet in the $I/SO(3)$ setting in DeepVoxels \cite{sitzmann2019deepvoxels} because the synthetic data consists of Lambertian objects with simple geometry, where the ray directions do not significantly affect the radiance. For more qualitative results, see App. Sec.\ref{rendering_qualitive_result}.

\begin{table}
\scalebox{0.85}{
\begin{tabular}{|c|c|c|c|c|c|c|c|c|}
\hline
    Dataset&Method &\multicolumn{3}{|c|}{I/I}&\multicolumn{4}{|c|}{I/SO(3)}\\
    \hline
    \multicolumn{2}{|c|}{}&PSNR$\uparrow$&SSIM$\uparrow$&LPIPS$\downarrow$&PSNR$\uparrow$&SSIM$\uparrow$&LPIPS$\downarrow$&Pix- Var$\downarrow$\\ 
     \hline
    \multirow{2}{*}{\makecell[c]{Realistic Synthetic\\ $360^{\circ}$\cite{mildenhall2021nerf}}}&IBRNet\cite{wang2021ibrnet}&\textbf{26.91} & 0.928&\textbf{0.084}&26.77 &0.923 &0.091 & 66.58\\
    &Ours&26.90 & \textbf{0.929}&0.086&\textbf{26.90}&\textbf{0.929} &\textbf{0.086}&\textbf{0.00}\\
      \hline
   \multirow{2}{*}{\makecell[c]{Real\\Forward-Facing \cite{mildenhall2019local}}} &IBRNet\cite{wang2021ibrnet}& \textbf{25.13}&\textbf{0.817} &\textbf{0.205} &24.60 & 0.797&0.223 & 52.66\\
    &Ours&24.93 &0.808 &0.212 &\textbf{24.93} &\textbf{0.808} &\textbf{0.212} &\textbf{0.00}\\
    \hline
    \multirow{2}{*}{\makecell[c]{Diffuse Synthetic\\ $360^{\circ}$\cite{sitzmann2019deepvoxels}}}&IBRNet \cite{wang2021ibrnet} &\textbf{37.21} & \textbf{0.989}&\textbf{0.019} & 37.07& \textbf{0.988}& \textbf{0.019}&34.51 \\
    &Ours& 37.11& 0.987& \textbf{0.019}&\textbf{37.11} & 0.987&\textbf{0.019} &\textbf{0.00}\\
\bottomrule
\end{tabular}
}
\caption{The results for the experiments of generalized rendering without per-scene tuning. The metrics in the cell are PSNR$\uparrow$, SSIM$\uparrow$, and Pixel Variance$\downarrow$ (denoted as Pix-Var). The evaluation of IBRNet\cite{wang2021ibrnet} is performed by testing the released model on both canonical and rotated test datasets.}
  \label{tab:rendering_table}
\end{table}

\vspace{-10pt}
\section{Conclusion and Broader Impacts}
\vspace{-8pt}
To learn equivariant geometric priors from multi-views, we modeled the convolution on the light field as a generalized convolution on the homogeneous space of rays with $SE(3)$ as the acting group. To obtain  expressive point features, we extended convolution to equivariant attention over rays. The main limitation of the approach is the finite sampling of the light field. The sampling of the light field by sparse views cannot account for large object motions with drastic aspect change, leading to a breakdown of equivariance.
This novel general equivariant representation framework for light fields can inspire further work on 3D vision and graphics tasks. We don't see the direct negative impact of our work but it could have negative societal consequences if misused without authorization, especially regarding private information.

\vspace{-8pt}
\section{Acknowledgement}
\vspace{-8pt}
The authors gratefully acknowledge support by the support by the following grants: NSF FRR 2220868, NSF IIS-RI 2212433, NSF TRIPODS 1934960, NSF CPS 2038873.
\newpage
{\small
\bibliographystyle{plainnat}
\bibliography{example_paper}

\begin{thebibliography}{71}
\providecommand{\natexlab}[1]{#1}
\providecommand{\url}[1]{\texttt{#1}}
\expandafter\ifx\csname urlstyle\endcsname\relax
  \providecommand{\doi}[1]{doi: #1}\else
  \providecommand{\doi}{doi: \begingroup \urlstyle{rm}\Url}\fi

\bibitem[Aronsson(2022)]{aronsson2022homogeneous}
Jimmy Aronsson.
\newblock Homogeneous vector bundles and g-equivariant convolutional neural
  networks.
\newblock \emph{Sampling Theory, Signal Processing, and Data Analysis},
  20\penalty0 (2):\penalty0 1--35, 2022.

\bibitem[Attal et~al.(2022)Attal, Huang, Zollh{\"o}fer, Kopf, and
  Kim]{attal2022learning}
Benjamin Attal, Jia-Bin Huang, Michael Zollh{\"o}fer, Johannes Kopf, and
  Changil Kim.
\newblock Learning neural light fields with ray-space embedding.
\newblock In \emph{Proceedings of the IEEE/CVF Conference on Computer Vision
  and Pattern Recognition}, pages 19819--19829, 2022.

\bibitem[Bautista et~al.(2021)Bautista, Talbott, Zhai, Srivastava, and
  Susskind]{bautista2021generalization}
Miguel~Angel Bautista, Walter Talbott, Shuangfei Zhai, Nitish Srivastava, and
  Joshua~M Susskind.
\newblock On the generalization of learning-based 3d reconstruction.
\newblock In \emph{Proceedings of the IEEE/CVF Winter Conference on
  Applications of Computer Vision}, pages 2180--2189, 2021.

\bibitem[Bekkers(2019)]{bekkers2019b}
Erik~J Bekkers.
\newblock B-spline cnns on lie groups.
\newblock \emph{arXiv preprint arXiv:1909.12057}, 2019.

\bibitem[Bemana et~al.(2020)Bemana, Myszkowski, Seidel, and
  Ritschel]{bemana2020x}
Mojtaba Bemana, Karol Myszkowski, Hans-Peter Seidel, and Tobias Ritschel.
\newblock X-fields: Implicit neural view-, light-and time-image interpolation.
\newblock \emph{ACM Transactions on Graphics (TOG)}, 39\penalty0 (6):\penalty0
  1--15, 2020.

\bibitem[Bergen and Adelson(1991)]{bergen1991plenoptic}
James~R Bergen and Edward~H Adelson.
\newblock The plenoptic function and the elements of early vision.
\newblock \emph{Computational models of visual processing}, 1:\penalty0 8,
  1991.

\bibitem[Brandstetter et~al.(2021)Brandstetter, Hesselink, van~der Pol,
  Bekkers, and Welling]{brandstetter2021geometric}
Johannes Brandstetter, Rob Hesselink, Elise van~der Pol, Erik Bekkers, and Max
  Welling.
\newblock Geometric and physical quantities improve e (3) equivariant message
  passing.
\newblock \emph{arXiv preprint arXiv:2110.02905}, 2021.

\bibitem[Cesa et~al.(2021)Cesa, Lang, and Weiler]{cesa2021program}
Gabriele Cesa, Leon Lang, and Maurice Weiler.
\newblock A program to build e (n)-equivariant steerable cnns.
\newblock In \emph{International Conference on Learning Representations}, 2021.

\bibitem[Chang et~al.(2015)Chang, Funkhouser, Guibas, Hanrahan, Huang, Li,
  Savarese, Savva, Song, Su, et~al.]{chang2015shapenet}
Angel~X Chang, Thomas Funkhouser, Leonidas Guibas, Pat Hanrahan, Qixing Huang,
  Zimo Li, Silvio Savarese, Manolis Savva, Shuran Song, Hao Su, et~al.
\newblock Shapenet: An information-rich 3d model repository.
\newblock \emph{arXiv preprint arXiv:1512.03012}, 2015.

\bibitem[Chatzipantazis et~al.(2022)Chatzipantazis, Pertigkiozoglou, Dobriban,
  and Daniilidis]{chatzipantazis2022se}
Evangelos Chatzipantazis, Stefanos Pertigkiozoglou, Edgar Dobriban, and Kostas
  Daniilidis.
\newblock {SE(3)}-equivariant attention networks for shape reconstruction in
  function space.
\newblock \emph{arXiv preprint arXiv:2204.02394}, 2022.

\bibitem[Chen et~al.(2021{\natexlab{a}})Chen, Xu, Zhao, Zhang, Xiang, Yu, and
  Su]{chen2021mvsnerf}
Anpei Chen, Zexiang Xu, Fuqiang Zhao, Xiaoshuai Zhang, Fanbo Xiang, Jingyi Yu,
  and Hao Su.
\newblock Mvsnerf: Fast generalizable radiance field reconstruction from
  multi-view stereo.
\newblock In \emph{Proceedings of the IEEE/CVF International Conference on
  Computer Vision}, pages 14124--14133, 2021{\natexlab{a}}.

\bibitem[Chen et~al.(2021{\natexlab{b}})Chen, Liu, Chen, Li, and
  Hill]{chen2021equivariant}
Haiwei Chen, Shichen Liu, Weikai Chen, Hao Li, and Randall Hill.
\newblock Equivariant point network for 3d point cloud analysis.
\newblock In \emph{Proceedings of the IEEE/CVF Conference on Computer Vision
  and Pattern Recognition}, pages 14514--14523, 2021{\natexlab{b}}.

\bibitem[Chen et~al.(2023)Chen, Xu, Wu, Zheng, Cham, and Cai]{chen2023explicit}
Yuedong Chen, Haofei Xu, Qianyi Wu, Chuanxia Zheng, Tat-Jen Cham, and Jianfei
  Cai.
\newblock Explicit correspondence matching for generalizable neural radiance
  fields.
\newblock \emph{arXiv preprint arXiv:2304.12294}, 2023.

\bibitem[Choy et~al.(2016)Choy, Xu, Gwak, Chen, and Savarese]{choy20163d}
Christopher~B Choy, Danfei Xu, JunYoung Gwak, Kevin Chen, and Silvio Savarese.
\newblock 3d-r2n2: A unified approach for single and multi-view 3d object
  reconstruction.
\newblock In \emph{European conference on computer vision}, pages 628--644.
  Springer, 2016.

\bibitem[Cohen and Welling(2016)]{cohen2016group}
Taco Cohen and Max Welling.
\newblock Group equivariant convolutional networks.
\newblock In \emph{International conference on machine learning}, pages
  2990--2999. PMLR, 2016.

\bibitem[Cohen et~al.(2019{\natexlab{a}})Cohen, Weiler, Kicanaoglu, and
  Welling]{cohen2019gauge}
Taco Cohen, Maurice Weiler, Berkay Kicanaoglu, and Max Welling.
\newblock Gauge equivariant convolutional networks and the icosahedral cnn.
\newblock In \emph{International conference on Machine learning}, pages
  1321--1330. PMLR, 2019{\natexlab{a}}.

\bibitem[Cohen et~al.(2018)Cohen, Geiger, K{\"o}hler, and
  Welling]{cohen2018spherical}
Taco~S Cohen, Mario Geiger, Jonas K{\"o}hler, and Max Welling.
\newblock Spherical cnns.
\newblock \emph{arXiv preprint arXiv:1801.10130}, 2018.

\bibitem[Cohen et~al.(2019{\natexlab{b}})Cohen, Geiger, and
  Weiler]{cohen2019general}
Taco~S Cohen, Mario Geiger, and Maurice Weiler.
\newblock A general theory of equivariant cnns on homogeneous spaces.
\newblock \emph{Advances in neural information processing systems}, 32,
  2019{\natexlab{b}}.

\bibitem[Deng et~al.(2021)Deng, Litany, Duan, Poulenard, Tagliasacchi, and
  Guibas]{deng2021vector}
Congyue Deng, Or~Litany, Yueqi Duan, Adrien Poulenard, Andrea Tagliasacchi, and
  Leonidas Guibas.
\newblock Vector neurons: A general framework for so (3)-equivariant networks.
\newblock \emph{arXiv preprint arXiv:2104.12229}, 2021.

\bibitem[Du et~al.(2023)Du, Smith, Tewari, and Sitzmann]{du2023learning}
Yilun Du, Cameron Smith, Ayush Tewari, and Vincent Sitzmann.
\newblock Learning to render novel views from wide-baseline stereo pairs.
\newblock \emph{arXiv preprint arXiv:2304.08463}, 2023.

\bibitem[Esteves et~al.(2018)Esteves, Allen-Blanchette, Makadia, and
  Daniilidis]{esteves2018learning}
Carlos Esteves, Christine Allen-Blanchette, Ameesh Makadia, and Kostas
  Daniilidis.
\newblock Learning so (3) equivariant representations with spherical cnns.
\newblock In \emph{Proceedings of the European Conference on Computer Vision
  (ECCV)}, pages 52--68, 2018.

\bibitem[Esteves et~al.(2019)Esteves, Xu, Allen-Blanchette, and
  Daniilidis]{esteves2019equivariant}
Carlos Esteves, Yinshuang Xu, Christine Allen-Blanchette, and Kostas
  Daniilidis.
\newblock Equivariant multi-view networks.
\newblock In \emph{Proceedings of the IEEE/CVF International Conference on
  Computer Vision}, pages 1568--1577, 2019.

\bibitem[Esteves et~al.(2020)Esteves, Makadia, and Daniilidis]{esteves2020spin}
Carlos Esteves, Ameesh Makadia, and Kostas Daniilidis.
\newblock Spin-weighted spherical cnns.
\newblock \emph{arXiv preprint arXiv:2006.10731}, 2020.

\bibitem[Finzi et~al.(2020)Finzi, Stanton, Izmailov, and
  Wilson]{finzi2020generalizing}
Marc Finzi, Samuel Stanton, Pavel Izmailov, and Andrew~Gordon Wilson.
\newblock Generalizing convolutional neural networks for equivariance to lie
  groups on arbitrary continuous data.
\newblock In \emph{International Conference on Machine Learning}, pages
  3165--3176. PMLR, 2020.

\bibitem[Finzi et~al.(2021)Finzi, Welling, and Wilson]{finzi2021practical}
Marc Finzi, Max Welling, and Andrew~Gordon Wilson.
\newblock A practical method for constructing equivariant multilayer
  perceptrons for arbitrary matrix groups.
\newblock \emph{arXiv preprint arXiv:2104.09459}, 2021.

\bibitem[Fuchs et~al.(2020)Fuchs, Worrall, Fischer, and Welling]{fuchs2020se}
Fabian~B Fuchs, Daniel~E Worrall, Volker Fischer, and Max Welling.
\newblock Se (3)-transformers: 3d roto-translation equivariant attention
  networks.
\newblock \emph{arXiv preprint arXiv:2006.10503}, 2020.

\bibitem[Furukawa et~al.(2015)Furukawa, Hern{\'a}ndez,
  et~al.]{furukawa2015multi}
Yasutaka Furukawa, Carlos Hern{\'a}ndez, et~al.
\newblock Multi-view stereo: A tutorial.
\newblock \emph{Foundations and Trends{\textregistered} in Computer Graphics
  and Vision}, 9\penalty0 (1-2):\penalty0 1--148, 2015.

\bibitem[Gallier and Quaintance(2020)]{gallier2020differential}
Jean Gallier and Jocelyn Quaintance.
\newblock \emph{Differential geometry and Lie groups: a computational
  perspective}, volume~12.
\newblock Springer Nature, 2020.

\bibitem[Han et~al.(2021)Han, Ding, Xue, and Xia]{han2021redet}
Jiaming Han, Jian Ding, Nan Xue, and Gui-Song Xia.
\newblock Redet: A rotation-equivariant detector for aerial object detection.
\newblock In \emph{Proceedings of the IEEE/CVF Conference on Computer Vision
  and Pattern Recognition}, pages 2786--2795, 2021.

\bibitem[Huang et~al.(2023)Huang, Zhang, Feng, Li, Wang, and
  Wang]{huang2023local}
Xin Huang, Qi~Zhang, Ying Feng, Xiaoyu Li, Xuan Wang, and Qing Wang.
\newblock Local implicit ray function for generalizable radiance field
  representation.
\newblock \emph{arXiv preprint arXiv:2304.12746}, 2023.

\bibitem[Hutchinson et~al.(2021)Hutchinson, Le~Lan, Zaidi, Dupont, Teh, and
  Kim]{hutchinson2021lietransformer}
Michael~J Hutchinson, Charline Le~Lan, Sheheryar Zaidi, Emilien Dupont,
  Yee~Whye Teh, and Hyunjik Kim.
\newblock Lietransformer: Equivariant self-attention for lie groups.
\newblock In \emph{International Conference on Machine Learning}, pages
  4533--4543. PMLR, 2021.

\bibitem[Jiang et~al.(2022)Jiang, Jiang, Grauman, and Zhu]{jiang2022few}
Hanwen Jiang, Zhenyu Jiang, Kristen Grauman, and Yuke Zhu.
\newblock Few-view object reconstruction with unknown categories and camera
  poses.
\newblock \emph{arXiv preprint arXiv:2212.04492}, 2022.

\bibitem[Kalantari et~al.(2016)Kalantari, Wang, and
  Ramamoorthi]{kalantari2016learning}
Nima~Khademi Kalantari, Ting-Chun Wang, and Ravi Ramamoorthi.
\newblock Learning-based view synthesis for light field cameras.
\newblock \emph{ACM Transactions on Graphics (TOG)}, 35\penalty0 (6):\penalty0
  1--10, 2016.

\bibitem[Kar et~al.(2017)Kar, H{\"a}ne, and Malik]{kar2017learning}
Abhishek Kar, Christian H{\"a}ne, and Jitendra Malik.
\newblock Learning a multi-view stereo machine.
\newblock \emph{Advances in neural information processing systems}, 30, 2017.

\bibitem[Kondor and Trivedi(2018)]{kondor2018generalization}
Risi Kondor and Shubhendu Trivedi.
\newblock On the generalization of equivariance and convolution in neural
  networks to the action of compact groups.
\newblock In \emph{International Conference on Machine Learning}, pages
  2747--2755. PMLR, 2018.

\bibitem[Levoy and Hanrahan(1996)]{levoy1996light}
Marc Levoy and Pat Hanrahan.
\newblock Light field rendering.
\newblock In \emph{Proceedings of the 23rd annual conference on Computer
  graphics and interactive techniques}, pages 31--42, 1996.

\bibitem[Long et~al.(2022)Long, Lin, Wang, Komura, and
  Wang]{long2022sparseneus}
Xiaoxiao Long, Cheng Lin, Peng Wang, Taku Komura, and Wenping Wang.
\newblock Sparseneus: Fast generalizable neural surface reconstruction from
  sparse views.
\newblock \emph{arXiv preprint arXiv:2206.05737}, 2022.

\bibitem[MacDonald et~al.(2022)MacDonald, Ramasinghe, and
  Lucey]{macdonald2022enabling}
Lachlan~E MacDonald, Sameera Ramasinghe, and Simon Lucey.
\newblock Enabling equivariance for arbitrary lie groups.
\newblock In \emph{Proceedings of the IEEE/CVF Conference on Computer Vision
  and Pattern Recognition}, pages 8183--8192, 2022.

\bibitem[Mescheder et~al.(2019)Mescheder, Oechsle, Niemeyer, Nowozin, and
  Geiger]{mescheder2019occupancy}
Lars Mescheder, Michael Oechsle, Michael Niemeyer, Sebastian Nowozin, and
  Andreas Geiger.
\newblock Occupancy networks: Learning 3d reconstruction in function space.
\newblock In \emph{Proceedings of the IEEE/CVF conference on computer vision
  and pattern recognition}, pages 4460--4470, 2019.

\bibitem[Mildenhall et~al.(2019)Mildenhall, Srinivasan, Ortiz-Cayon, Kalantari,
  Ramamoorthi, Ng, and Kar]{mildenhall2019local}
Ben Mildenhall, Pratul~P Srinivasan, Rodrigo Ortiz-Cayon, Nima~Khademi
  Kalantari, Ravi Ramamoorthi, Ren Ng, and Abhishek Kar.
\newblock Local light field fusion: Practical view synthesis with prescriptive
  sampling guidelines.
\newblock \emph{ACM Transactions on Graphics (TOG)}, 38\penalty0 (4):\penalty0
  1--14, 2019.

\bibitem[Mildenhall et~al.(2021)Mildenhall, Srinivasan, Tancik, Barron,
  Ramamoorthi, and Ng]{mildenhall2021nerf}
Ben Mildenhall, Pratul~P Srinivasan, Matthew Tancik, Jonathan~T Barron, Ravi
  Ramamoorthi, and Ren Ng.
\newblock Nerf: Representing scenes as neural radiance fields for view
  synthesis.
\newblock \emph{Communications of the ACM}, 65\penalty0 (1):\penalty0 99--106,
  2021.

\bibitem[Romero et~al.(2020)Romero, Bekkers, Tomczak, and
  Hoogendoorn]{romero2020attentive}
David Romero, Erik Bekkers, Jakub Tomczak, and Mark Hoogendoorn.
\newblock Attentive group equivariant convolutional networks.
\newblock In \emph{International Conference on Machine Learning}, pages
  8188--8199. PMLR, 2020.

\bibitem[Romero and Cordonnier(2020)]{romero2020group}
David~W Romero and Jean-Baptiste Cordonnier.
\newblock Group equivariant stand-alone self-attention for vision.
\newblock \emph{arXiv preprint arXiv:2010.00977}, 2020.

\bibitem[Saito et~al.(2019)Saito, Huang, Natsume, Morishima, Kanazawa, and
  Li]{saito2019pifu}
Shunsuke Saito, Zeng Huang, Ryota Natsume, Shigeo Morishima, Angjoo Kanazawa,
  and Hao Li.
\newblock Pifu: Pixel-aligned implicit function for high-resolution clothed
  human digitization.
\newblock In \emph{Proceedings of the IEEE/CVF International Conference on
  Computer Vision}, pages 2304--2314, 2019.

\bibitem[Satorras et~al.(2021)Satorras, Hoogeboom, and Welling]{satorras2021n}
V{\i}ctor~Garcia Satorras, Emiel Hoogeboom, and Max Welling.
\newblock E (n) equivariant graph neural networks.
\newblock In \emph{International conference on machine learning}, pages
  9323--9332. PMLR, 2021.

\bibitem[Sitzmann et~al.(2019)Sitzmann, Thies, Heide, Nie{\ss}ner, Wetzstein,
  and Zollhofer]{sitzmann2019deepvoxels}
Vincent Sitzmann, Justus Thies, Felix Heide, Matthias Nie{\ss}ner, Gordon
  Wetzstein, and Michael Zollhofer.
\newblock Deepvoxels: Learning persistent 3d feature embeddings.
\newblock In \emph{Proceedings of the IEEE/CVF Conference on Computer Vision
  and Pattern Recognition}, pages 2437--2446, 2019.

\bibitem[Sitzmann et~al.(2021)Sitzmann, Rezchikov, Freeman, Tenenbaum, and
  Durand]{sitzmann2021light}
Vincent Sitzmann, Semon Rezchikov, Bill Freeman, Josh Tenenbaum, and Fredo
  Durand.
\newblock Light field networks: Neural scene representations with
  single-evaluation rendering.
\newblock \emph{Advances in Neural Information Processing Systems},
  34:\penalty0 19313--19325, 2021.

\bibitem[Srinivasan et~al.(2017)Srinivasan, Wang, Sreelal, Ramamoorthi, and
  Ng]{srinivasan2017learning}
Pratul~P Srinivasan, Tongzhou Wang, Ashwin Sreelal, Ravi Ramamoorthi, and Ren
  Ng.
\newblock Learning to synthesize a 4d rgbd light field from a single image.
\newblock In \emph{Proceedings of the IEEE International Conference on Computer
  Vision}, pages 2243--2251, 2017.

\bibitem[Steenrod(1999)]{steenrod1999topology}
Norman Steenrod.
\newblock \emph{The topology of fibre bundles}, volume~27.
\newblock Princeton university press, 1999.

\bibitem[Suhail et~al.(2022{\natexlab{a}})Suhail, Esteves, Sigal, and
  Makadia]{suhail2022generalizable}
Mohammed Suhail, Carlos Esteves, Leonid Sigal, and Ameesh Makadia.
\newblock Generalizable patch-based neural rendering.
\newblock \emph{arXiv preprint arXiv:2207.10662}, 2022{\natexlab{a}}.

\bibitem[Suhail et~al.(2022{\natexlab{b}})Suhail, Esteves, Sigal, and
  Makadia]{suhail2022light}
Mohammed Suhail, Carlos Esteves, Leonid Sigal, and Ameesh Makadia.
\newblock Light field neural rendering.
\newblock In \emph{Proceedings of the IEEE/CVF Conference on Computer Vision
  and Pattern Recognition}, pages 8269--8279, 2022{\natexlab{b}}.

\bibitem[Thomas et~al.(2018)Thomas, Smidt, Kearnes, Yang, Li, Kohlhoff, and
  Riley]{thomas2018tensor}
Nathaniel Thomas, Tess Smidt, Steven Kearnes, Lusann Yang, Li~Li, Kai Kohlhoff,
  and Patrick Riley.
\newblock Tensor field networks: Rotation-and translation-equivariant neural
  networks for 3d point clouds.
\newblock \emph{arXiv preprint arXiv:1802.08219}, 2018.

\bibitem[Tulsiani et~al.(2018)Tulsiani, Efros, and Malik]{tulsiani2018multi}
Shubham Tulsiani, Alexei~A Efros, and Jitendra Malik.
\newblock Multi-view consistency as supervisory signal for learning shape and
  pose prediction.
\newblock In \emph{Proceedings of the IEEE conference on computer vision and
  pattern recognition}, pages 2897--2905, 2018.

\bibitem[Tyszkiewicz et~al.(2022)Tyszkiewicz, Maninis, Popov, and
  Ferrari]{tyszkiewicz2022raytran}
Micha{\l}~J Tyszkiewicz, Kevis-Kokitsi Maninis, Stefan Popov, and Vittorio
  Ferrari.
\newblock Raytran: 3d pose estimation and shape reconstruction of multiple
  objects from videos with ray-traced transformers.
\newblock \emph{arXiv preprint arXiv:2203.13296}, 2022.

\bibitem[Varma et~al.(2022)Varma, Wang, Chen, Chen, Venugopalan, and
  Wang]{varma2022attention}
Mukund Varma, Peihao Wang, Xuxi Chen, Tianlong Chen, Subhashini Venugopalan,
  and Zhangyang Wang.
\newblock Is attention all that nerf needs?
\newblock In \emph{The Eleventh International Conference on Learning
  Representations}, 2022.

\bibitem[Villar et~al.(2021)Villar, Hogg, Storey-Fisher, Yao, and
  Blum-Smith]{villar2021scalars}
Soledad Villar, David~W Hogg, Kate Storey-Fisher, Weichi Yao, and Ben
  Blum-Smith.
\newblock Scalars are universal: Equivariant machine learning, structured like
  classical physics.
\newblock \emph{Advances in Neural Information Processing Systems},
  34:\penalty0 28848--28863, 2021.

\bibitem[Wang et~al.(2018)Wang, Zhang, Li, Fu, Liu, and
  Jiang]{wang2018pixel2mesh}
Nanyang Wang, Yinda Zhang, Zhuwen Li, Yanwei Fu, Wei Liu, and Yu-Gang Jiang.
\newblock Pixel2mesh: Generating 3d mesh models from single rgb images.
\newblock In \emph{Proceedings of the European conference on computer vision
  (ECCV)}, pages 52--67, 2018.

\bibitem[Wang et~al.(2021)Wang, Wang, Genova, Srinivasan, Zhou, Barron,
  Martin-Brualla, Snavely, and Funkhouser]{wang2021ibrnet}
Qianqian Wang, Zhicheng Wang, Kyle Genova, Pratul~P Srinivasan, Howard Zhou,
  Jonathan~T Barron, Ricardo Martin-Brualla, Noah Snavely, and Thomas
  Funkhouser.
\newblock Ibrnet: Learning multi-view image-based rendering.
\newblock In \emph{Proceedings of the IEEE/CVF Conference on Computer Vision
  and Pattern Recognition}, pages 4690--4699, 2021.

\bibitem[Weiler and Cesa(2019)]{weiler2019general}
Maurice Weiler and Gabriele Cesa.
\newblock General $ e (2) $-equivariant steerable cnns.
\newblock \emph{arXiv preprint arXiv:1911.08251}, 2019.

\bibitem[Weiler et~al.(2018)Weiler, Geiger, Welling, Boomsma, and
  Cohen]{weiler20183d}
Maurice Weiler, Mario Geiger, Max Welling, Wouter Boomsma, and Taco Cohen.
\newblock 3d steerable cnns: Learning rotationally equivariant features in
  volumetric data.
\newblock \emph{arXiv preprint arXiv:1807.02547}, 2018.

\bibitem[Weiler et~al.(2021)Weiler, Forr{\'e}, Verlinde, and
  Welling]{weiler2021coordinate}
Maurice Weiler, Patrick Forr{\'e}, Erik Verlinde, and Max Welling.
\newblock Coordinate independent convolutional networks--isometry and gauge
  equivariant convolutions on riemannian manifolds.
\newblock \emph{arXiv preprint arXiv:2106.06020}, 2021.

\bibitem[Worrall et~al.(2017)Worrall, Garbin, Turmukhambetov, and
  Brostow]{worrall2017harmonic}
Daniel~E Worrall, Stephan~J Garbin, Daniyar Turmukhambetov, and Gabriel~J
  Brostow.
\newblock Harmonic networks: Deep translation and rotation equivariance.
\newblock In \emph{Proceedings of the IEEE Conference on Computer Vision and
  Pattern Recognition}, pages 5028--5037, 2017.

\bibitem[Wu et~al.(2021)Wu, Liu, Fang, and Chai]{wu2021revisiting}
Gaochang Wu, Yebin Liu, Lu~Fang, and Tianyou Chai.
\newblock Revisiting light field rendering with deep anti-aliasing neural
  network.
\newblock \emph{IEEE Transactions on Pattern Analysis and Machine
  Intelligence}, 2021.

\bibitem[Xie et~al.(2019)Xie, Yao, Sun, Zhou, and Zhang]{xie2019pix2vox}
Haozhe Xie, Hongxun Yao, Xiaoshuai Sun, Shangchen Zhou, and Shengping Zhang.
\newblock Pix2vox: Context-aware 3d reconstruction from single and multi-view
  images.
\newblock In \emph{Proceedings of the IEEE/CVF international conference on
  computer vision}, pages 2690--2698, 2019.

\bibitem[Xie et~al.(2020)Xie, Yao, Zhang, Zhou, and Sun]{xie2020pix2vox++}
Haozhe Xie, Hongxun Yao, Shengping Zhang, Shangchen Zhou, and Wenxiu Sun.
\newblock Pix2vox++: Multi-scale context-aware 3d object reconstruction from
  single and multiple images.
\newblock \emph{International Journal of Computer Vision}, 128\penalty0
  (12):\penalty0 2919--2935, 2020.

\bibitem[Xu et~al.(2019)Xu, Wang, Ceylan, Mech, and Neumann]{xu2019disn}
Qiangeng Xu, Weiyue Wang, Duygu Ceylan, Radomir Mech, and Ulrich Neumann.
\newblock Disn: Deep implicit surface network for high-quality single-view 3d
  reconstruction.
\newblock \emph{Advances in Neural Information Processing Systems}, 32, 2019.

\bibitem[Xu et~al.(2022)Xu, Lei, Dobriban, and Daniilidis]{xu2022unified}
Yinshuang Xu, Jiahui Lei, Edgar Dobriban, and Kostas Daniilidis.
\newblock Unified fourier-based kernel and nonlinearity design for equivariant
  networks on homogeneous spaces.
\newblock In \emph{International Conference on Machine Learning}, pages
  24596--24614. PMLR, 2022.

\bibitem[Yang et~al.(2021)Yang, Wen, Chen, Chen, and Jia]{yang2021deep}
Mingyue Yang, Yuxin Wen, Weikai Chen, Yongwei Chen, and Kui Jia.
\newblock Deep optimized priors for 3d shape modeling and reconstruction.
\newblock In \emph{Proceedings of the IEEE/CVF Conference on Computer Vision
  and Pattern Recognition}, pages 3269--3278, 2021.

\bibitem[Yang et~al.(2022)Yang, Ren, Bautista, Zhang, Shan, and
  Huang]{yang2022fvor}
Zhenpei Yang, Zhile Ren, Miguel~Angel Bautista, Zaiwei Zhang, Qi~Shan, and
  Qixing Huang.
\newblock Fvor: Robust joint shape and pose optimization for few-view object
  reconstruction.
\newblock In \emph{Proceedings of the IEEE/CVF Conference on Computer Vision
  and Pattern Recognition}, pages 2497--2507, 2022.

\bibitem[Yu et~al.(2021)Yu, Ye, Tancik, and Kanazawa]{yu2021pixelnerf}
Alex Yu, Vickie Ye, Matthew Tancik, and Angjoo Kanazawa.
\newblock pixelnerf: Neural radiance fields from one or few images.
\newblock In \emph{Proceedings of the IEEE/CVF Conference on Computer Vision
  and Pattern Recognition}, pages 4578--4587, 2021.

\bibitem[Zhang et~al.(2018)Zhang, Isola, Efros, Shechtman, and
  Wang]{zhang2018unreasonable}
Richard Zhang, Phillip Isola, Alexei~A Efros, Eli Shechtman, and Oliver Wang.
\newblock The unreasonable effectiveness of deep features as a perceptual
  metric.
\newblock In \emph{Proceedings of the IEEE conference on computer vision and
  pattern recognition}, pages 586--595, 2018.

\end{thebibliography}
}

\newpage
 \appendix
 \onecolumn
\section*{Supplemental Material}
The introduction of convolution and attention on the space of rays in 3D required  additional geometric representations  for which there was no space in the main paper to elaborate on. We will introduce here all the necessary notations and definitions. We have accompanied this presentation with examples of specific groups in order to elucidate the abstract concepts needed in the definitions.
\section{Preliminary}

\label{appendix}
\subsection{Group Actions and Homogeneous Spaces}
\begin{wrapfigure}{l}{0.5\linewidth}
        \includegraphics[width=0.25\textheight]{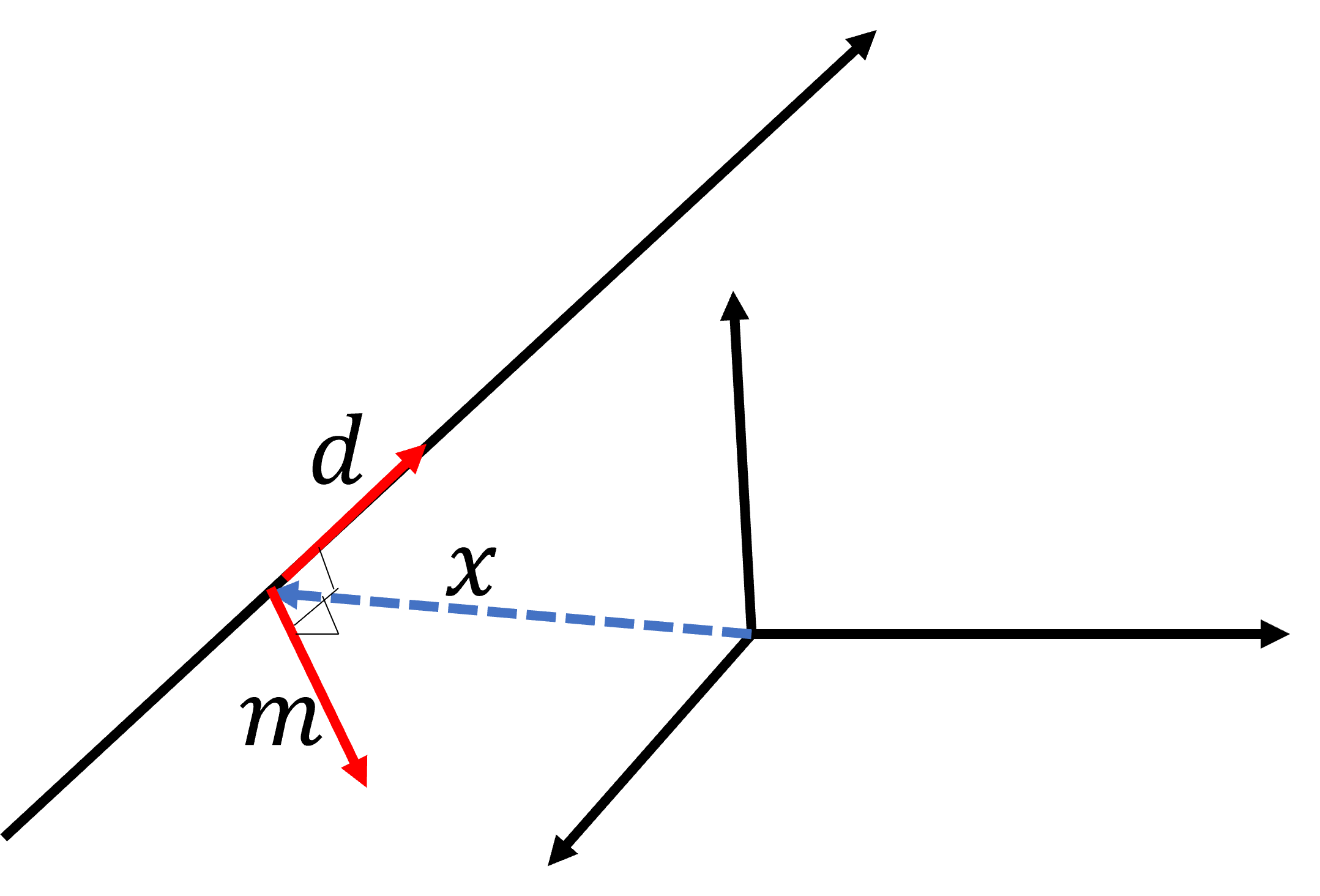}
         \caption{The visualization of Pl{\"u}cker coordinates: A  ray $x$ can be denoted as $(\bm{d}, \bm{m})$ where $\bm{x}$ is any point on the ray $x$, and $\bm{d}$ is the direction of the ray $x$.  $\bm{m}$ is defined as $ \bm{x} \times \bm{d}$. }
   \label{fig:plucker}
\end{wrapfigure}

Given the action of the group $G$ on a homogeneous space $X$, and given $x_0$ as the origin of $X$, the stabilizer group $H$ of $x_0$ in $G$ is the group that leaves $x_0$ intact, i.e., $H=\{h \in G| hx_0=x_0\}$. The group $G$ can be partitioned into the quotient space (the set of left cosets) $G/H$, and $X$ is isomorphic to $G/H$ since all group elements in the same coset transform $x_0$ to the same element in $X$, that is, for any element $g' \in gH$ we have $g'x_0=gx_0$. 

\begin{exg}
\label{SE3_ex_homo_ray_space}
\textcolor{red}{$SE(3)$ acting on the ray space $\mathcal{R}$:}
Take $SE(3)$ as the acting group and the ray space $\mathcal{R}$ as its homogeneous space. We use Pl{\"u}cker coordinates to parameterize the ray space $\mathcal{R}$:  any $x \in \mathcal{R}$ can be denoted as $(\bm{d},\bm{m})$,  where $\bm{d} \in \mathbb{S}^2$ is the direction of the ray, and $\bm{m} = \bm{x} \times \bm{d}$ where $\bm{x}$ is any point on the ray, as shown in figure $\ref{fig:plucker}$. A group element $g=(R,\bm{t}) \in SE(3)$  acts on the the ray space as:
\begin{align}
     \label{group action}
    gx=g(\bm{d},\bm{m})=(R\bm{d},R\bm{m}+\bm{t}\times (R\bm{d})).
\end{align}
 We can choose the fixed origin of the homogeneous space to be $\eta =  ([0,0,1]^T, [0,0,0]^T)$, the line identical with the $z$-axis of the coordinate system. Then, the stabilizer group $H$ (the rotation around and translation along the ray) can be parameterized as $H=\left\{(R_Z(\gamma),t[0,0,1]^T)|\gamma \in [0, 2\pi), t \in \mathbb{R} \right\}$, i.e.,  $H \simeq SO(2) \times \mathbb{R}$. We can simplify $H$ as $H=\left\{(\gamma,t)|\gamma \in [0, 2\pi), t \in \mathbb{R} \right\}$. $\mathcal{R}$ is the quotient space $SE(3)/(SO(2)\times \mathbb{R})$ up to isomorphism.

\begin{figure}[htb]
  \centering
   \includegraphics[width=0.5\linewidth]{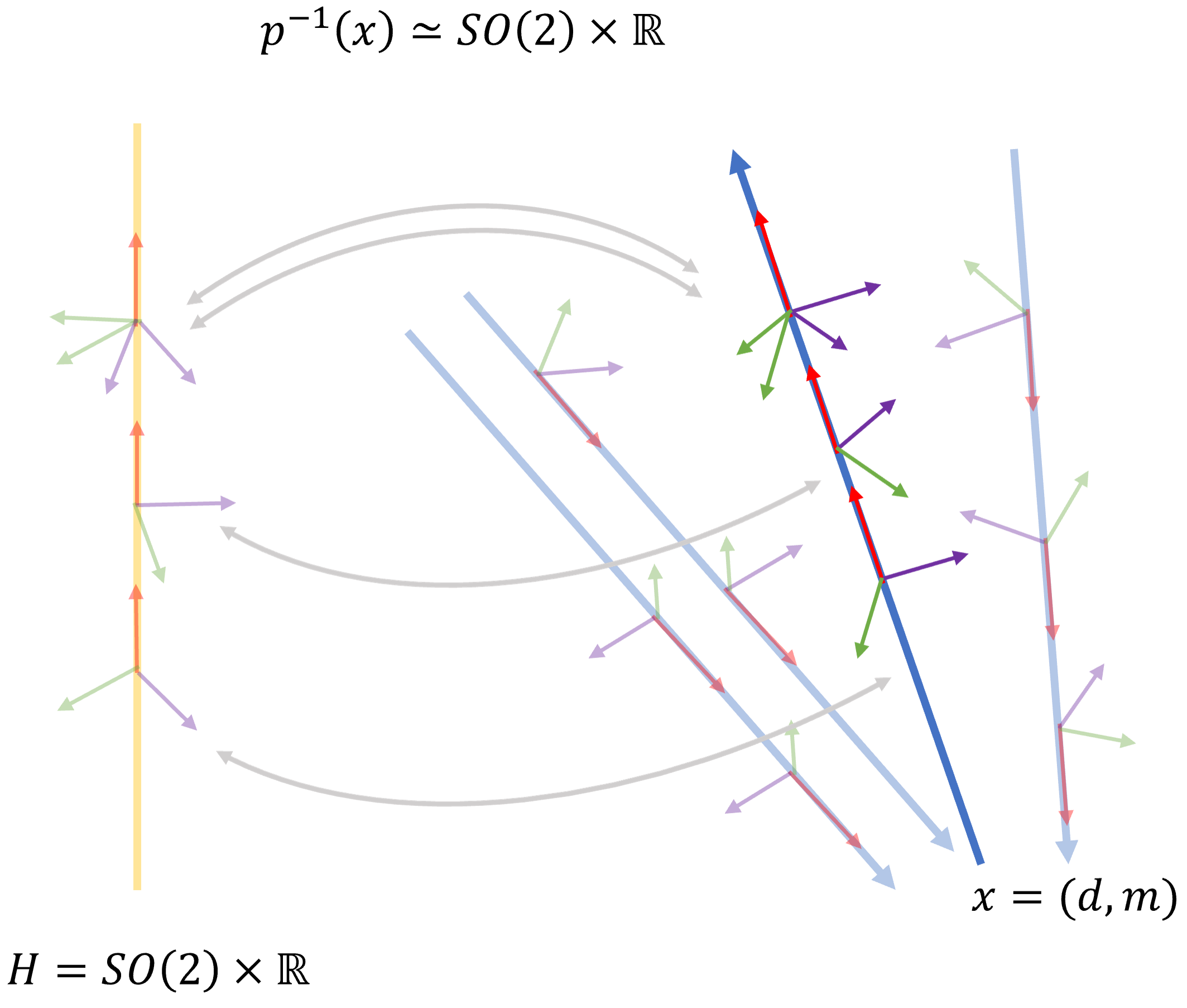}
   \caption{We can view $SE(3)$ as an $SO(2) \times \mathbb{R}$-principal bundle, where the projection map $p: SE(3) \rightarrow \mathcal{R}$ is $p(R,\bm{t})=(R[0,0,1]^T,\bm{t} \times (R[0,0,1]^T)$, and the inverse of $p$ is  $p^{-1}(x)=\left\{(R,\bm{t})|(R,\bm{t})\eta=x\right \}$.
   We use the coordinate frames (red axis denotes $Z$-axis, green axis denotes $X$-axis, and purple axis denotes $Y$-axis) to denote the element in $SE(3)$ because we can use the position of the coordinate origin to denote the translation $\bm{t}$ and use $X$-axis, $Y$-axis, and $Z$-axis to represent the first, second and third columns in rotation $R$. When we say next ``the coordinate frame on the line/ray'' we will mean that its origin is on the line/ray.
   By this convention, the coordinate frames representing the element in $H=SO(2)\times \mathbb{R}$ are the frames whose $Z$-axis aligns with $[0,0,1]^T$ and whose origin is $[0,0,t]^T$ for any $t \in \mathbb{R}$, i.e., frames on the yellow line in the left of the figure. For one ray $x=(\bm{d},\bm{m})$ (illustrated as the chosen blue ray),  the coordinate frames on the ray $x$ whose $Z$-axis aligns with the ray $\bm{d}_x$ are in $p^{-1}(x)$. As shown in the figure, there exists a bijection (gray double arrow line ) between $p^{-1}(x)$ and 
   $H=SO(2)\times \mathbb{R}$. 
   $p^{-1}(x)$ is isomorphic to $H=SO(2)\times \mathbb{R}$.}
   \label{fig:bundle struc}
\end{figure}
\end{exg}

\begin{exg}
\label{SE3_ex_homo_R3}
\textcolor{red}{$SE(3)$ acting on the $3D$ Euclidean space $\mathbb{R}^3$:}
$\mathbb{R}^3$ is isomorphic to $SE(3)/SO(3)$. 
Consider another case when $SE(3)$ acts on the homogeneous space $\mathbb{R}^3$; for any $g =(R,\bm{t}) \in SE(3)$ and $\bm{x} \in \mathbb{R}^3$, $g\bm{x}=R\bm{x} +\bm{t}$. If the fixed origin is $[0,0,0]^T$, the stabilizer subgroup is $H=SO(3)$ since any rotation $g=(R, \bm{0})$ leaves $[0,0,0]^T$ unchanged. 
\end{exg}

\begin{exg}
\label{SO3_ex_homo}
\textcolor{red}{$SO(3)$ acting on the sphere $\mathbb{S}^2$:}
$\mathbb{S}^2$ is isomorphic to $SO(3)/SO(2)$.
The last example is $SO(3)$ acting on the homogeneous space sphere $\mathbb{S}^2$. Given the fixed origin point as $[0,0,1]^T$, the stabilizer group is $SO(2)$.  
\end{exg}

\subsection{Principal Bundle}
\label{principal_bundle}
As stated in \cite{gallier2020differential, cohen2019general}, the partition of the group $G$ into cosets allows us to treat the group $G$ as the principal bundle where the \textbf{total space} is $G$,
the \textbf{base space} is the homogeneous space $G/H$\footnote{We use $G/H$ to denote the homogeneous space since the homogeneous space $X$ can be identified with $G/H$ up to an isomorphism, i.e., $X \simeq G/H$.}, the canonical \textbf{fiber} is the stabilizer group $H$, the \textbf{projection map} $p: G \rightarrow G/H$ reads $p(g)=gH=gx_0=x$.  The \textbf{section} $s: G/H \rightarrow G$ of $p$ should satisfy that  $p\circ s =id_{G/H}$, where $id_{G/H}$ is the identity map on $G/H$. 
Note that non-trivial principal bundles do not have a continuous global section, but we can define a continuous section locally on the open set $U \subseteq G/H$. 
The action of $G$ causes a twist of the fiber, i.e., $gs(x)$ might not be equal to $s(gx)$ though they are in the same coset. We use the \textbf{twist function} $\text{h}: G \times G/H \rightarrow H$  to denote the twist: $gs(x)=s(gx)\text{h}(g,x)$. Same as \cite{cohen2019general}, we simplify $\text{h}(g, eH)$ to be $\text{h}(g)$, where $e$ is the identity element in $G$ and $eH=x_0$.

\begin{exg}
\label{SE3_ex_sec_R3}
\textcolor{red}{Projection, section map and twist function for $\mathbb{R}^3$ and $SE(3)$:}
According to Ex. \ref{SE3_ex_homo_R3}, we can consider a bundle with total space as $SE(3)$, base space as $\mathbb{R}^3$, and the fiber as $SO(3)$. For any $g=(R, \bm{t}) \in SE(3)$, the projection map $p: SE(3) \rightarrow \mathbb{R}^3$ projects $g$ as $p(R, \bm{t})=\bm{t}$. For any $\bm{x} \in \mathbb{R}^3$, we can define the section map  $s: \mathbb{R}^3 \rightarrow SE(3)$ as $s(\bm{x})= (I,\bm{x})$. The twist function ${\rm h}: SE(3) \times \mathbb{R}^3 \rightarrow SO(3)$ is that ${\rm h}(g,\bm{x})=s(g\bm{x})^{-1}gs(\bm{x})=R$ for any $\bm{x} \in \mathbb{R}^3$ and any $g=(R,\bm{t}) \in SE(3)$. This twist function is independent of $\bm{x}$ due to the fact that $SE(3)= \mathbb{R}^3 \rtimes SO(3)$ is a semidirect product group as stated in \cite{cohen2019general}.

\end{exg}

\begin{exg}
\label{SO3_ex_sec}
\textcolor{red}{Projection, section map, and twist function for $\mathbb{S}^2$ and $SO(3)$:}
As shown in Ex. \ref{SO3_ex_homo}, $SO(3)$ can be viewed as a principal bundle with the base space as $\mathbb{S}^2$ and the fiber as $SO(2)$. With the rotation $R \in SO(3)$ parameterized as $R=R_Z(\alpha)R_Y(\beta)R_Z(\gamma)$, the projection $p: G \rightarrow G/H$ maps $R$ as follows:
\begin{align*}
p(R)&=R_Z(\alpha)R_Y(\beta)R_Z(\gamma)[0,0,1]^T \\
&=R_Z(\alpha)R_Y(\beta)[0,0,1]^T\\
&=[sin(\beta)cos(\alpha),sin(\beta)sin(\alpha),cos(\beta)]^T.
\end{align*}
For any $\bm{d} \in \mathbb{S}^2$, the section map $s: \mathbb{S}^2 \rightarrow SO(3)$ of $p$ should satisfy that $p\circ s =id_{\mathbb{S}^2}$ as mentioned above, i.e., $s(\bm{d})[0,0,1]^T=\bm{d}$. For instance, we could define the section map $s$ as:
\begin{align*}
s(\bm{d})=R_Z(\alpha_{\bm{d}})R_Y(\beta_{\bm{d}}),
\end{align*} 
where $\alpha_{\bm{d}}$ and $\beta_{\bm{d}}$ satisfies that
$$\bm{d}=[sin(\beta_{\bm{d}})cos(\alpha_{\bm{d}}),sin(\beta_{\bm{d}})sin(\alpha_{\bm{d}}),cos(\beta_{\bm{d}})]^T.$$ Specifically, when $\bm{d}=[0,0,1]^T$, $\alpha_{\bm{d}}=0$ and $\beta_{\bm{d}}=0$; when $\bm{d}=-[0,0,1]^T$, $\alpha_{\bm{d}}=0$ and $\beta_{\bm{d}}=\pi$.

As defined, the twist function ${\rm h}: SO(3) \times \mathbb{S}^2 \rightarrow SO(2)$ is that ${\rm h}(R,\bm{d})=s(R\bm{d})^{-1}Rs(\bm{d})$. 
\end{exg}

\begin{exg}
\label{SE3_ex_sec_ray_space}
\textcolor{red}{Projection, section map, and twist function for $\mathcal{R}$ and $SE(3)$:}
The final example is $SE(3)$ with $\mathcal{R}$ as the base space and  $SO(2) \times \mathbb{R}$ as the fiber, which is the focus of this work, as shown in figure \ref{fig:bundle struc}.  According to the group action defined in Eq. \ref{group action}, the projection map $p: SE(3) \rightarrow \mathcal{R}$ is: 
$$p((R,\bm{t}))=(R,\bm{t})\eta=(R[0,0,1]^T, \bm{t}\times (R[0,0,1]^T)).$$
This represents a ray direction $\bm{d}$ with the 3rd column of a rotation matrix and  the moment $\bm{m}$ with the cross product of the translation and the ray direction. We can construct a section $s: G/H \rightarrow G$ using the Pl{\"u}cker coordinate:
$$s((\bm{d},\bm{m})) = (s_a(\bm{d}), s_b(\bm{d},\bm{m})),$$
where $s_a(\bm{d}) \in SO(3)$ is a rotation that $s_a(\bm{d})[0,0,1]^T=\bm{d}$, i.e., $s_a$ is a section map from $\mathbb{S}^2$ to $SO(3)$ as shown in Ex. \ref{SO3_ex_sec}; and $s_b(\bm{d}, \bm{m}) \in \mathbb{R}^3$ is a point on the ray $(\bm{d}, \bm{m})$. In this paper, we define the section map as $s((\bm{d},\bm{m})) = (R_Z(\alpha_{\bm{d}})R_Y(\beta_{\bm{d}}), \bm{d}\times\bm{m})$,
where $\alpha_{\bm{d}}$ and  $\beta_{\bm{d}}$ satisfy that $\bm{d}=R_Z(\alpha_{\bm{d}})R_Y(\beta_{\bm{d}})[0,0,1]^T$, which is the same as Ex. \ref{SO3_ex_sec}. Figure \ref{fig: section of ray} displays the visualization of the section map.

Given the section map, for any $g=(R_g,\bm{t}_g) \in SE(3)$ and $x= (\bm{d}_x, \bm{m}_x) \in \mathcal{R}$, we  have
the twist function ${\rm h}:SE(3)\times \mathcal{R} \rightarrow SO(2) \times \mathbb{R}$ is ${\rm h}(g,x)=s^{-1}(gx)gs(x) = ({\rm h}_a(R_g,\bm{d}_x), {\rm h}_b(g,x))$, where ${\rm h}_a:SO(3) \times \mathbb{S}^2 \rightarrow SO(2)$ is the twist function corresponding to $s_a$, as shown in Ex. \ref{SO3_ex_sec}, and ${\rm h}_b(g,x)=\langle R_gs_b(x)+\bm{t}_g-s_b(gx), R_g\bm{d}_x \rangle$. With the above section $s$ defined in this paper, the twist function  ${\rm h}:SE(3)\times \mathcal{R} \rightarrow SO(2) \times \mathbb{R}$ is $${\rm h}(g,x)=s^{-1}(gx)gs(x)= (R_Z(R_g,\bm{d}_x), \langle \bm{t}_g, (R_g\bm{d}_x) \rangle),$$ where $R_Z(R_g,\bm{d}_x) = R_Y^{-1}(\beta_{R_g\bm{d}_x}) R_Z^{-1}(\alpha_{R_g\bm{d}_x})R_gR_Z(\alpha_{\bm{d}_x})R_Y(\beta_{\bm{d}_x})$.

To understand the twist function clearly, we visualize a twist induced by a translation in $SE(3)$ in figure \ref{fig: visualization of twist},

\begin{figure}[t]
  \centering
\includegraphics[width=0.7\linewidth]{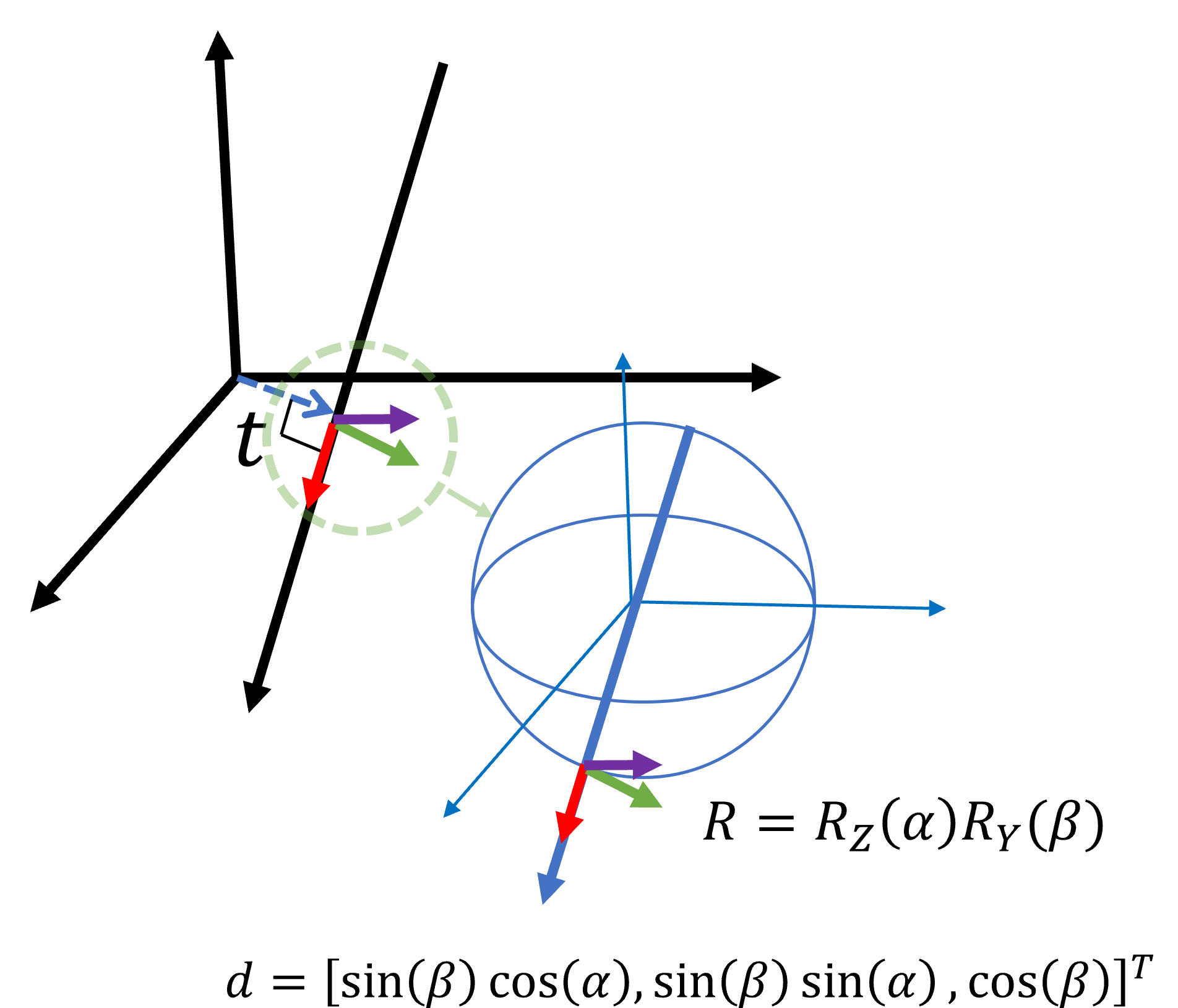}
   \caption{For a ray $x=(\bm{d}, \bm{m})$, we need to choose an element $(R,\bm{t}) \in SE(3)$ as the representative element $s(x)$ such that $s(x) ([0,0,1]^T, [0,0,0]^T) =x$. This figure shows one example of the section map $s$ from ray space to $SE(3)$. This map also serves as the section provided in this paper. 
   The axes of the coordinate frame in the figure represent $R = s_a(\bm{d})=R_Z(\alpha_{\bm{d}})R_Y(\beta_{\bm{d}})$, where the green axis, purple axis, and red axis represent 1st, 2nd and 3rd column in the rotation matrix $R$, respectively.The origin of the frame,$\bm{t} = s_b(\bm{d}, \bm{m})=\bm{d} \times \bm{m}$, denotes the translation.}
   \label{fig: section of ray}
\end{figure}

\begin{figure}[t]
  \centering
\includegraphics[width=0.7\linewidth]{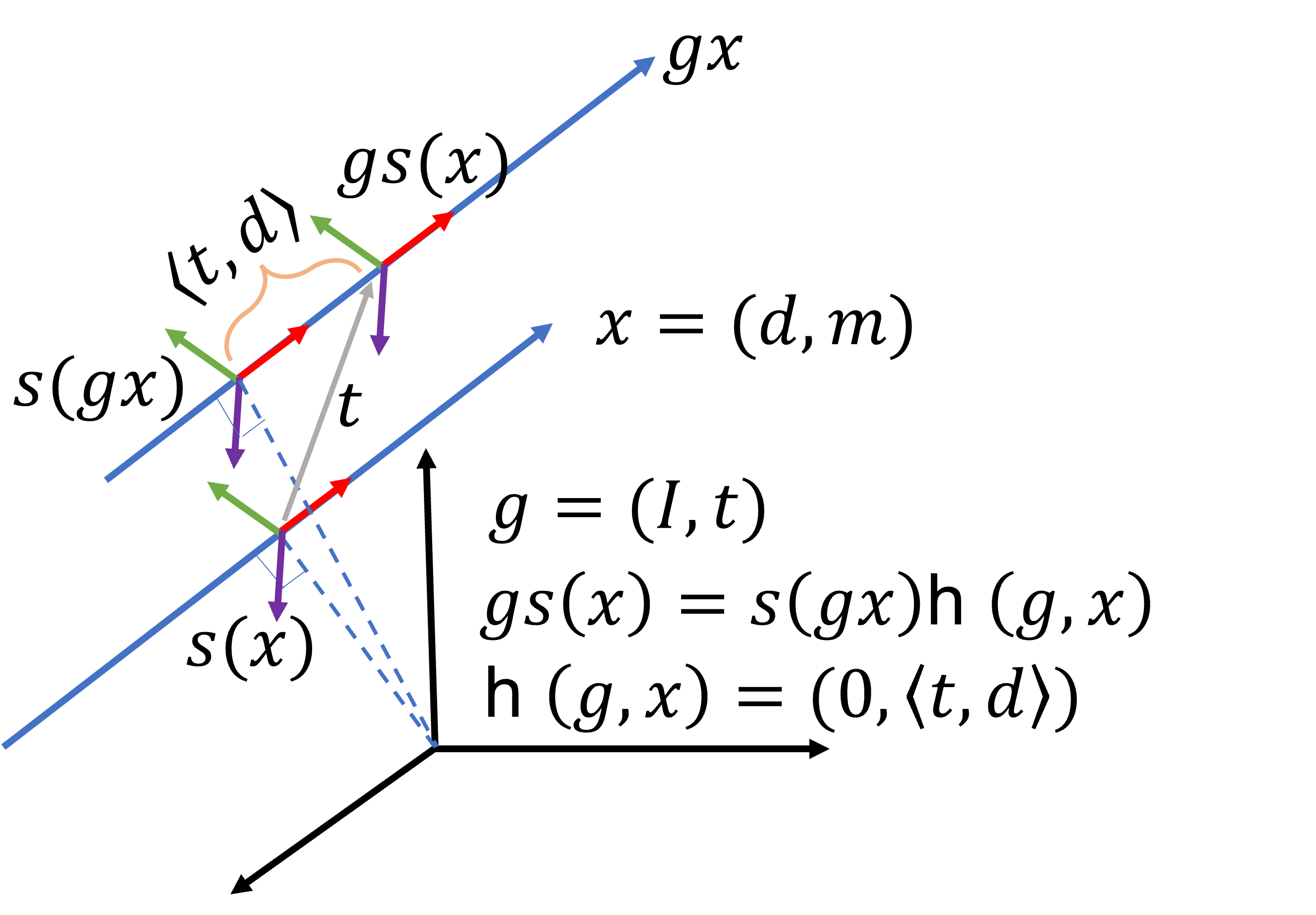}
   \caption{When we translate a ray $x=(\bm{d}, \bm{m})$ with $g =(I,\bm{t}) \in SE(3)$, we will find that $gs(x)$ does not agree with $s(gx)$. As defined in figure \ref{fig: section of ray}, we have $s_b(x) \perp \bm{d}$ and $s_b(gx) \perp \bm{d}$. 
  Following the geometry of the figure,  we obtain that ${\rm h_b}(g,x)=\langle t, \bm{d} \rangle [0,0,1]^T$, i.e.,${\rm h}(g,x) = s(gx)^{-1}gs(x) = (I, \langle t, \bm{d} \rangle [0,0,1]^T)=(0,\langle t, \bm{d} \rangle).$}
   \label{fig: visualization of twist}
\end{figure}

\end{exg}

\subsection{Associated Vector Bundle}
\label{associated vector}
Given the principal bundle $G$, we can construct the associated vector bundle by replacing the fiber $H$ with the vector space $V$, where $V \simeq \mathbb{R}^n$ and $H$ acts on $V$ through a group representation $\rho: H \rightarrow GL(V)$.
The group representation corresponds to the type of geometric quantity in the vector space $V$, for example, the scalar, the vector, or the higher-order tensor.

The quotient space $E = G \times_{\rho} V /H$ is defined through the right action of $H$ on $G \times V$: $(g,v)h=(gh, \rho(h)^{-1}v)$ for any $h \in H$, $g \in G$ and $v \in V$. With the defined projection map $p: G\times_{\rho} V \rightarrow G/H$: $p([g,v])=gH$, where $[g,v] = \left \{(gh,\rho(h)^{-1}v)| h \in H \right \}$, the element in $G\times_{\rho} V$, we obtain the fiber bundle $E = G\times_{\rho} V $ associated to the principal bundle $G$. For more background and details of the associated vector bundle, we recommend referring to the following sources: \cite{steenrod1999topology} and \cite{cohen2019general}.

The feature function $f: U \subseteq G/H \rightarrow V$ can encode the local section of the associated vector bundle $s_v: U \subseteq G/H \rightarrow G\times_{\rho} V$: $s_v(x)=[s(x),f(x)]$, where $s$ is the section map of the principal bundle as defined in Sec. \ref{principal_bundle}. 
The group $G$ acting on the field $f$ as shown in \cite{cohen2019general}:
\begin{align}
(\mathcal{L}_gf)(x)=\rho(\text{h}(g^{-1},x))^{-1}f(g^{-1}x),
\label{action on field}
\end{align}
where $\text{h}:G \times G/H \rightarrow H$ is the twist function as defined in Sec. \ref{principal_bundle}.

\subsection{Equivariant Convolution Over Homogeneous Space}
\label{convolution sec}
The generalized equivariant convolution over homogeneous space, as stated in \cite{cohen2019general}, that maps a feature field $f^{l_{in}}$ over homogeneous space $G/H_1$ to a feature $f^{l'_{out}}$ over homogeneous space $G/H_2$ by convolving with a kernel $\kappa$ is defined as:
\begin{align}
f^{l'_{out}}(x)=\int_{G/H_1}\kappa(s_2(x)^{-1}y)\rho_{in}(\text{h}_1(s_2(x)^{-1}s_1(y)))f^{l_{in}}(y)dy,
    \label{convolution}
\end{align}
where $l_{in}$ and $l'_{out}$ \footnote{%
In this context, the feature type indicates the specific geometric quantity in vector spaces $V_{in}$ and $V_{out}$. $V_{in}$ corresponds to the stabilizer $H_1$ and $V_{out}$ corresponds to the stabilizer $H_2$. It is possible for $H_1$ and $H_2$ to be distinct; therefore, to differentiate the types of features corresponding to different stabilizers, we utilize $l$ and $l'$ as notations for the feature types.} denote the input and output feature types, respectively. $\rho_{in}$ is the group representation of $H_1$ corresponding to the feature type $l_{in}$, $s_1$ is the section map from $G/H_1$ to $G$ (see Sec. \ref{principal_bundle}), $s_2$ is the section map from $G/H_2$ to $G$ (see Sec. \ref{principal_bundle}), $\text{h}_1$ is the twist function corresponding to $s_1$ (see Sec. \ref{principal_bundle}). 

The convolution is equivariant with respect to $G$, that is $$\mathcal{L}^{out}_gf^{l'_{out}}= \kappa*\mathcal{L}^{in}_gf^{l_{in}},$$
if and only if $\kappa(h_2x)=\rho_{out}(h_2)\kappa(x)\rho_{in}(\text{h}_1^{-1}(h_2,x))$ for any $h_2 \in H_2$, where $\rho_{out}$ is the group representation of $H_2$ corresponding to the feature type $l'_{out}$.

In the following examples, we will illustrate three instances where the input and output homogeneous spaces, denoted as $G/H_1$ and $G/H_2$, respectively, are identical, meaning that $H_1=H_2$. These examples involve convolutions from $\mathbb{R}^3$ to $\mathbb{R}^3$, from $\mathbb{S}^2$ to $\mathbb{S}^2$, and from $\mathcal{R}$ to $\mathcal{R}$. Furthermore, we will show an example where $H_1$ and $H_2$ differ, explicitly focusing on the convolution from $\mathcal{R}$ to $\mathbb{R}^3$.
 
\begin{exg}
\textcolor{red}{$SE(3)$ equivariant convolution from $\mathbb{R}^3$ to $\mathbb{R}^3$:}
 If we use the section map as stated in Ex.  \ref{SE3_ex_sec_R3}, we will find that ${\rm h}(s(x)^{-1}s(y))=I$, therefore convolution \ref{convolution} becomes:
\begin{align*}
f^{l_{out}}(x)&=\int_{\mathbb{R}^3}\kappa(s(x)^{-1}y)f^{l_{in}}(y)dy\\
&=\int_{\mathbb{R}^3}\kappa(y-x)f^{l_{in}}(y)dy
\end{align*}
and
$\kappa$ should satisfy  
\begin{align*}
\kappa(Rx)&=\rho_{out}(R)\kappa(x)\rho_{in}({\rm h}^{-1}(R,x))\\
&=\rho_{out}(R)\kappa(x)\rho_{in}({\rm h}^{-1}(R))\\
&=\rho_{out}(R)\kappa(x)\rho^{-1}_{in}(R)%
\end{align*}
for any $R \in SO(3)$.
When the feature type $l_{in}$ and $l_{out}$ corresponds to the irreducible representation, we have $$ 
\kappa(Rx) = D_{l_{out}}(R)\kappa(x)D_{l_{in}}(R)^{-1}
$$
where $D^{l_{in}}$ and $D^{l_{out}}$ are the Wigner-D matrices, i.e. irreducible representations corresponding to the feature types $l_{in}$ and $l_{out}$, which is the same as the analytical result in \cite{weiler20183d}. 
\end{exg}

\begin{exg}
\textcolor{red}{$SO(3)$ equivariant spherical convolution from $\mathbb{S}^2$ to $\mathbb{S}^2$:}
For spherical convolution, when we substitute the section in Eq. \ref{convolution} with the section we defined in Ex. \ref{SO3_ex_sec}, the convolution integral takes the following form:
\begin{align*}
&f^{l_{out}}(\alpha, \beta)\\
&=\int_{\alpha' \in [0,2\pi),\beta' \in [0,\pi)} \kappa(R^{-1}_Y(\beta)R^{-1}_Z(\alpha)R_Z(\alpha')R_Y(\beta')[0,0,1]^T)\\
&\rho_{in}({\rm h}(R^{-1}_Y(\beta)R^{-1}_Z(\alpha)R_Z(\alpha')R_Y(\beta')
)f^{l_{in}}(\alpha',\beta')d\alpha'sin(\beta')d\beta'
\end{align*}

where $[0,0,1]^T$  is the fixed original point as stated in Ex. \ref{SO3_ex_homo}, $\rho_{in}$ is the group representation of $SO(2)$ corresponding to the feature type $l_{in}$. When $\rho_{in}$ and $\rho_{out}$ are the irreducible representations of $SO(2)$, $\rho_{in}$ and $\rho_{out}$ can be denoted as 
$\rho_{in}(\theta)=e^{-il_{in}\theta}$ and $\rho_{out}(\theta)=e^{-il_{out}\theta}$.

To simplify the notation, we utilize $R(\theta)$ to represent $R_Z(\theta) \in SO(2)$, where $\theta \in [0, 2\pi)$. When considering the cases where $x = [0,0,1]^T$, $h(R(\theta)x) = R(\theta)$; when $x = -[0,0,1]^T$, $h(R(\theta)x) = R(-\theta)$; and when $x \in \mathbb{S}^2 - \left\{ [0,0,1]^T, -[0,0,1]^T \right\}$, $h(R(\theta)x) = R(-\theta) = I$. Therefore, the kernel $\kappa$ should satisfy the following conditions:
$\kappa (R(\theta)x)=e^{-il_{out}\theta}\kappa(x)$ for any $R(\theta) \in SO(2)$ and any $x \in \mathbb{S}^2-\left\{[0,0,1]^T, -[0,0,1]^T\right\}$; $\kappa(x)= e^{-i(l_{out}-l_{in})\theta}\kappa(x)$ for $x=[0,0,1]^T$; and $\kappa(x)= e^{-i(l_{out}+l_{in})\theta}\kappa(x)$ for $x=-[0,0,1]^T$.

Specifically, when the input and output are scalar feature fields over the sphere,  convolution reads
\begin{align*}
&f^{out}(\alpha, \beta)\\
&=\int_{\alpha' \in [0,2\pi),\beta' \in [0,\pi)} \kappa(R^{-1}_Y(\beta)R^{-1}_Z(\alpha)R_Z(\alpha')R_Y(\beta')\eta)\\
&f^{in}(\alpha',\beta')d\alpha'sin(\beta')d\beta'
\end{align*}

$\kappa$ has such constraint:
$$\kappa(R(\theta)x)=\kappa(x)$$ for any $R(\theta) \in SO(2)$, which is consistent with the isotropic kernel of the convolution in \cite{esteves2018learning}.

\end{exg}

\begin{exg}
\label{filter_solution}
\textcolor{red}{$SE(3)$ equivariant convolution from $\mathcal{R}$ to $\mathcal{R}$:}
In our case, the equivariant convolution from ray space to ray space is also based on the generalized equivariant convolution over a homogeneous space.  See Sec. \ref{equiconv}
for the details. We solve the constraint of the kernel here:
\begin{align}
    \kappa(hx)=\rho_{out}(h)\kappa(x)\rho_{in}(\text{h}^{-1}(h,x)),
\end{align}
for any $h \in SO(2) \times \mathbb{R}$.

The irreducible group representation $\rho_{in}$ 
for the corresponding feature type $l_{in}=(\omega^1_{in}, \omega^2_{in})$, where $\omega^1_{in} \in \mathbb{N}$ and $\omega^2_{in} \in \mathbb{R}$, can be written as
$\rho_{in}(\gamma,t)= e^{-i(\omega^1_{in}\gamma+\omega^2_{in}t)}$ for any $h =(\gamma ,t) \in  SO(2) \times \mathbb{R}$; and the irreducible group representation $\rho_{out}(\gamma,t) =e^{-i(\omega^1_{out}\gamma+\omega^2_{out}t)} $  for the feature type $l_{out}=(\omega^1_{out}, \omega^2_{out})$, where $\omega^1_{out} \in \mathbb{N}$ and $\omega^2_{out} \in \mathbb{R}$, for any $h =(\gamma ,t) \in  SO(2) \times \mathbb{R}$. 

To simplify the notation, we utilize $R(\gamma)$ to represent $R_Z(\gamma) \in SO(2)$, where $\gamma \in [0, 2\pi)$. For any $h =(\gamma,t) \in SO(2) \times \mathbb{R}$ and any $x =(\bm{d}_x,\bm{m}_x) \in \mathcal{R}$, we have $\text{h}(h,x)=s(hx)^{-1}hs(x)= (R_Z(R(\gamma),\bm{d}_x), \langle t[0,0,1]^T, \bm{d}_x \rangle)$ according to 
Ex. \ref{SE3_ex_sec_ray_space}. Since $SO(2) \times \mathbb{R}$ is a product group, we can have $\kappa(x) = \kappa_1(x)\kappa_2(x)$, where 
\begin{align}
\kappa_1((\gamma,t)x) = \rho_{out}((\gamma, 0))\kappa_1(x) \rho^{-1}_{in}((R_Z(R(\gamma),\bm{d}_x),0))
\label{constraint_k1}
\end{align}

\begin{align}
\kappa_2((\gamma,t)x) = \rho_{out}((0, t))\kappa_2(x) \rho^{-1}_{in}((0,\langle t[0,0,1]^T, \bm{d}_x \rangle ))
\label{constraint_k2}
\end{align}

Now we solve the constraint for the kernel $\kappa_1$:

One can check that for any $\bm{d}_x \in \mathbb{S}^2 -\left\{ [0,0,1]^T, -[0,0,1]^T\right\}$, $R_Z(R(\gamma),\bm{d}_x)=I$; when $\bm{d}_x=[0,0,1]^T$, $R_Z(R(\gamma),\bm{d}_x)=R(\gamma)$; and when 
$\bm{d}_x=-[0,0,1]^T$, $R_Z(R(\gamma),\bm{d}_x)=R(-\gamma)$.

Therefore, we obtain the constraint that 
\begin{align}
\label{con_k1}
    \kappa_1((\gamma,t)x) = e^{-i\omega^1_{out}\gamma}\kappa_1(x)
\end{align}
when $\bm{d}_x \in \mathbb{S}^2 -\left\{ [0,0,1]^T, -[0,0,1]^T\right\}$; 

\begin{align}
\label{con_k1_2}
    \kappa_1((\gamma,t)x) = e^{-i(\omega^1_{out}-\omega^1_{in})\gamma}\kappa_1(x)
\end{align}
when $\bm{d}_x = [0,0,1]^T $;

\begin{align}
\label{con_k1_3}
    \kappa_1((\gamma,t)x) = e^{-i(\omega^1_{out}+\omega^1_{in})\gamma}\kappa_1(x)
\end{align}
when $\bm{d}_x = -[0,0,1]^T $;

The solution for Eq. \ref{con_k1} is that $\kappa_1(x)=f(d(\eta, x), \angle([0,0,1]^T,\bm{d}_x))e^{-i\omega^{1}_{out}atan2([0,1,0]\bm{d}_x, [1,0,0]\bm{d}_x)} $, where $atan2$ is the 2-argument arctangent function, and $f$ is an arbitrary function that maps $(d(\eta, x), \angle([0,0,1]^T,\bm{d}_x))$ to the complex domain.

The solution for Eq. \ref{con_k1_2} is that when $\omega^1_{out} = \omega^1_{in}$, $\kappa_1(x)=C$, where $C$ is any constant value;  when $\omega^1_{out} \neq \omega^1_{in}$ and $x =\eta$, $\kappa_1(x)=0$; when $\omega^1_{out} \neq \omega^1_{in}$ and $x \neq \eta$, $\kappa_1(x)=f(d(\eta, x))e^{-i(\omega^{1}_{out}-\omega^{1}_{in})atan2([0,1,0]\bm{m}_x, [1,0,0]\bm{m}_x)}$,where $f$ is an arbitrary function that maps $d(x,\eta)$ to the complex domain.

The solution for Eq. \ref{con_k1_3} is that when $\omega^1_{out} = -\omega^1_{in}$, $\kappa_1(x)=C$, where $C$ is any constant value;  when $\omega^1_{out} \neq -\omega^1_{in}$ and $x =-\eta$, $\kappa_1(x)=0$; when $\omega^1_{out} \neq -\omega^1_{in}$ and $x \neq -\eta$, $\kappa_1(x)=f(d(\eta, x))e^{-i(\omega^{1}_{out}+\omega^{1}_{in})atan2([0,1,0]^T\bm{m}_x, [1,0,0]^T\bm{m}_x)}$,where $f$ is an arbitrary function that maps $d(x,\eta)$ to the complex domain.

Next, we will solve the constraint for the kernel $\kappa_2$, which is that $\kappa_2((\gamma,t)x)= e^{-i(\omega^2_{out}-\omega^2_{in}\langle [0,0,1]^T, \bm{d}_x \rangle)t}\kappa_2(x)$.

When $\bm{d}_x =[0,0,1]^T$, and $\omega^2_{out} \neq \omega^2_{in}$, $\kappa_2(x)=0$; When $\bm{d}_x=-[0,0,1]^T$ and $\omega^2_{out} \neq -\omega^2_{in}$, $\kappa_2(x)=0$; When $\bm{d}_x =[0,0,1]^T$, and $\omega^2_{out} = \omega^2_{in}$, $\kappa_2(x)=f(d(x,\eta))$, where $f$ is an arbitrary function that maps $d(x,\eta)$ to the complex domain; When $\bm{d}_x =-[0,0,1]^T$, and $\omega^2_{out} = -\omega^2_{in}$, $\kappa_2(x)=f(d(x,\eta))$, where $f$ is an arbitrary function that maps $d(x,\eta)$ to the complex domain; when $\bm{d}_x \in \mathbb{S}^2-\left\{[0,0,1]^T, -[0,0,1]^T\right\}$,  
\begin{align}
\kappa_2(x)=f(d(\eta, x), \angle([0,0,1]^T,\bm{d}_x))e^{-i(\omega^2_{out}-\omega^2_{in}\langle [0,0,1]^T, \bm{d}_x \rangle)g(x)}, \label{solution_for_nonre}
\end{align}
where  
$f$ is an arbitrary function that maps $(d(\eta, x), \angle([0,0,1]^T,\bm{d}_x))$ to the complex domain; $g(x)=[0,0,1](\bm{x}_Q-[0,0,0]^T)$, where $\bm{x}_Q$ represents the 3D coordinates of a point $Q$. This point $Q$ can be defined as the intersection of $x$ and $\eta$ if $x$ and $\eta$ intersect. Alternatively, if $x$ and $\eta$ do not intersect, $Q$ is determined as the intersection of $\eta$ and the ray $y$, which is perpendicular to both $x$ and $\eta$, and intersects with both $x$ and $\eta$. Refer to Figure \ref{fig:constraint} for a visual representation. One can easily check that $g((\gamma, t)x) = t+g(x)$, as shown in figure \ref{fig:constraint}, which makes the solution valid. 

If $x$ and $\eta$ are intersected, i.e., $[0,0,1]\bm{m}_x=0$, 
$$g(x) = [0,0,1](\bm{d}_x \times \bm{m}_x -\frac{[1,0,0](\bm{d}_x \times \bm{m}_x)}{[1,0,0]\bm{d}_x}\bm{d}_x)$$ when $[1,0,0]\bm{d}_x \neq 0$;  
$$g(x) = [0,0,1](\bm{d}_x \times \bm{m}_x -\frac{[0,1,0](\bm{d}_x \times \bm{m}_x)}{[0,1,0]\bm{d}_x}\bm{d}_x)$$ when $[1,0,0]\bm{d}_x = 0$;  

When $x$ and $\eta$ are not intersected, 
$$g(x) = [0,0,1](\bm{d}_x \times \bm{m}_x -\frac{[1,0,0](\bm{d}_x \times \bm{m}_x)[1,0,0]\bm{d}_x+[0,1,0](\bm{d}_x \times \bm{m}_x)[0,1,0]\bm{d}_x  }{([1,0,0]\bm{d}_x)^2+ ([0,1,0]\bm{d}_x)^2 }\bm{d}_x).$$

\paragraph{Regular Representation} 
\label{regular_representation}
Here, we delve into the case where the output field type corresponds to the group representation of $SO(2)\times \mathbb{R}$ that $\rho(\gamma,t) =\rho_1(\gamma)\otimes \rho_2(t)$ for any $(\gamma, t) \in SO(2)\times \mathbb{R}$, where $\rho_2$ is the regular representation. The regular representation of a group G is a linear representation that arises from the group action of G on itself by translation, that is when $\rho_2: \mathbb{R} \rightarrow GL(V)$ is the regular representation, for any $v \in V$, for any $t,t' \in \mathbb{R}$, we have $(\rho_2(t')v)_{t}=v_{t-t'}$, in other words, $v \in V$ can be viewed as a function defined on $\mathbb{R}$ or an infinite dimensional vector.  Then  according to Ex. \ref{SE3_ex_sec_ray_space}, the  group $SE(3)$ acting on the the field $f$ would be: 
\begin{align*}
(\mathcal{L}_gf)(x)_t&=(\rho(\text{h}(g^{-1},x))^{-1}f(g^{-1}x))_t\\
&=\rho_1({\rm h}_a(R_{g^{-1}},\bm{d}_x))^{-1}
f(g^{-1}x)_{t+{\rm h}_b(g^{-1},x)}\\
&=\rho_1(R_Z(R_{g^{-1}},\bm{d}_x))^{-1}
f(g^{-1}x)_{t+ \langle \bm{t}_{g^{-1}}, (R_{g^{-1}}\bm{d}_x) \rangle}
\end{align*}
for any $t \in \mathbb{R}$, $x \in \mathcal{R}$ and $g \in SE(3)$.

The points $\bm{x}$ on the ray $x =(\bm{d}_x, \bm{m}_x)$ can be uniquely expressed as $\bm{x} = s_b(x)+t_{\bm{x}}\bm{d}_x = \bm{d}_x\times \bm{m}_x + t_{\bm{x}}\bm{d}_x$, therefore for any $x \in \mathcal{R}$, any $t \in \mathbb{R}$, $f(x)_t$ can be expressed as a feature attached to the point $s_b(x)+t\bm{d}_x$ along the ray $x$,i.e., $f(x)_t = f'(s_b(x)+t\bm{d}_x, \bm{d}_x)$ as shown in figure \ref{fig:regular_and_points}.

\begin{figure}[t]
  \centering
   \includegraphics[width=0.9\linewidth]{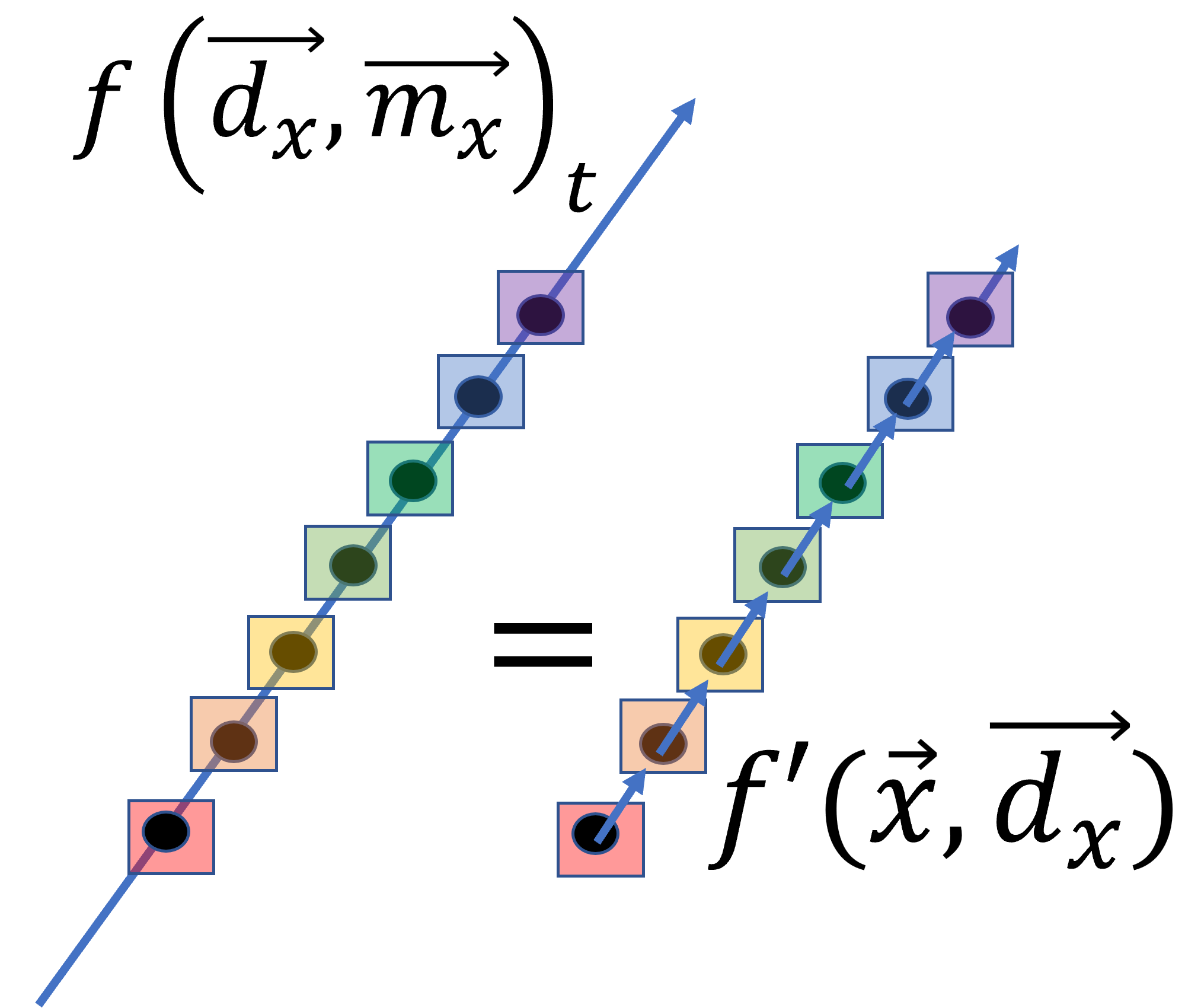}
   \caption{The feature attached to the ray, which corresponds to the regular representation of translation, can also be treated as the features attached to the points along the ray.}
   \label{fig:regular_and_points}
\end{figure}

Therefore, we have $f'(\bm{x},\bm{d})=f((\bm{d},\bm{x}\times\bm{d}))_{\langle\bm{x}-\bm{d}\times (\bm{x}\times \bm{d}), \bm{d}\rangle}$, one can easily check:
\begin{align}
(\mathcal{L}_gf')(\bm{x},\bm{d})&= \rho_1({\rm h}_a(R_{g^{-1}},\bm{d}))^{-1}
f'(R_{g^{-1}}\bm{x}+\bm{t}_{g^{-1}},R_{g^{-1}}\bm{d}) = \rho_1(R_Z(R_{g^{-1}},\bm{d}))^{-1}
f'(g^{-1}\bm{x},R_{g^{-1}}\bm{d})
\label{regular_action}
\end{align}

We should note the difference of the point $\bm{x}$ along the ray and the independent point $\bm{x}$, as shown in the above equation, the point $\bm{x}$ along the ray $x=(\bm{d},\bm{x}\times\bm{d})$ is denoted as $(\bm{x},\bm{d})$ instead of $\bm{x}$. Actually, it can be viewed as a  homogeneous space of $SE(3)$ larger than $\mathbb{R}^3$, whose elements are in $\mathbb{R}^3 \times \mathbb{S}^2$, as shown in figure \ref{fig:tensors_over_points_along_ray}.

\begin{figure}[t]
  \centering
   \includegraphics[width=0.9\linewidth]{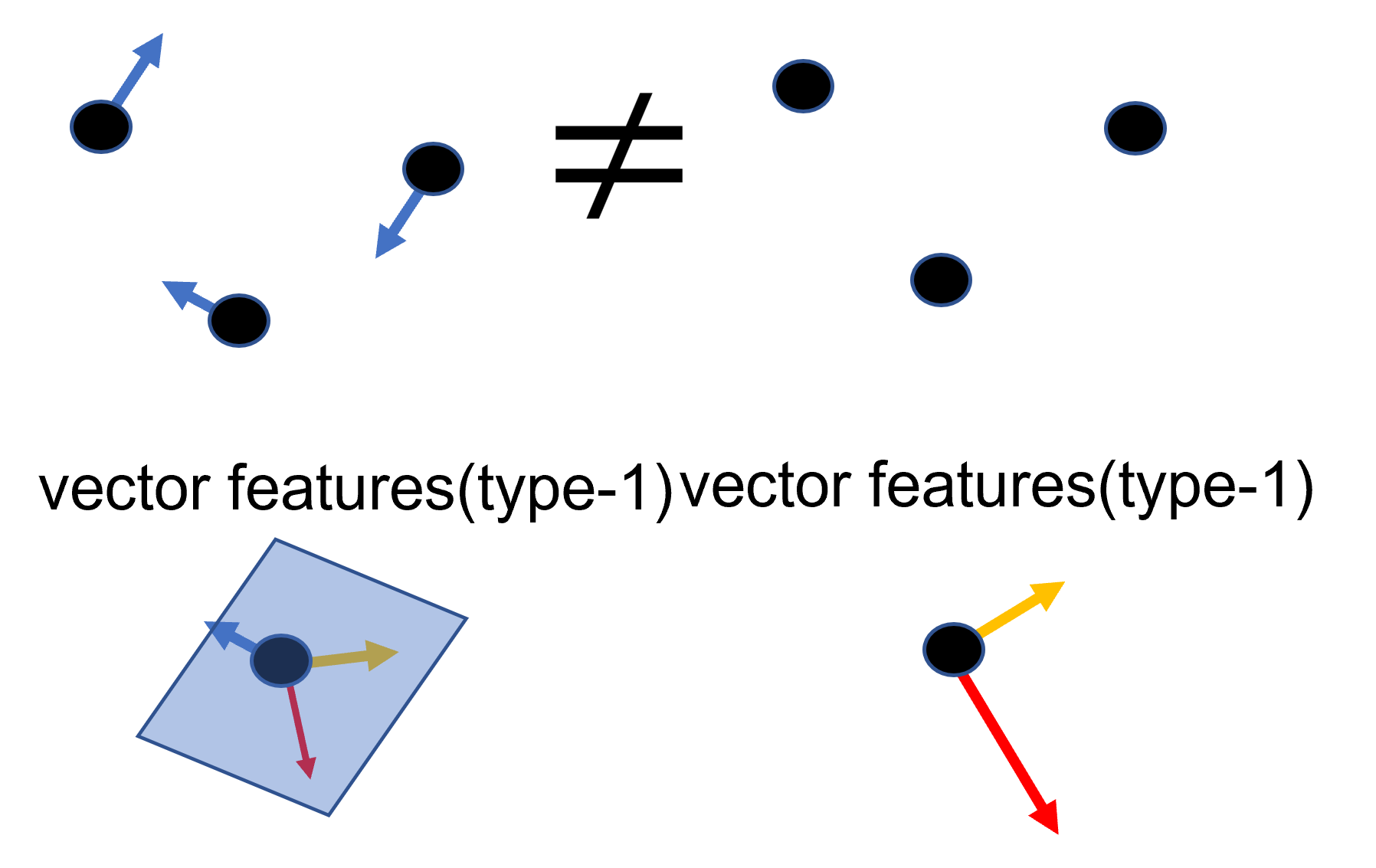}
   \caption{As shown in the figure, the point along the ray is distinct from the independent point. Moreover, we can observe that the type-1 feature of the point along the ray differs from that of the independent point. Specifically, the type-1 feature for the point along the ray can be interpreted as a vector lying on the plane orthogonal to the ray direction. In contrast, the type-1 feature for the independent point can be interpreted as a three-dimensional vector.}
   \label{fig:tensors_over_points_along_ray}
\end{figure}

To summarize, the features attached to the ray, whose type corresponds to the regular representation of translation, can be considered as the features attached to the points along the ray. The action of $SE(3)$ on features attached to these points can be expressed as shown in Eq. \ref{regular_action}.

\begin{figure}[t]
  \centering
   \includegraphics[width=0.9\linewidth]{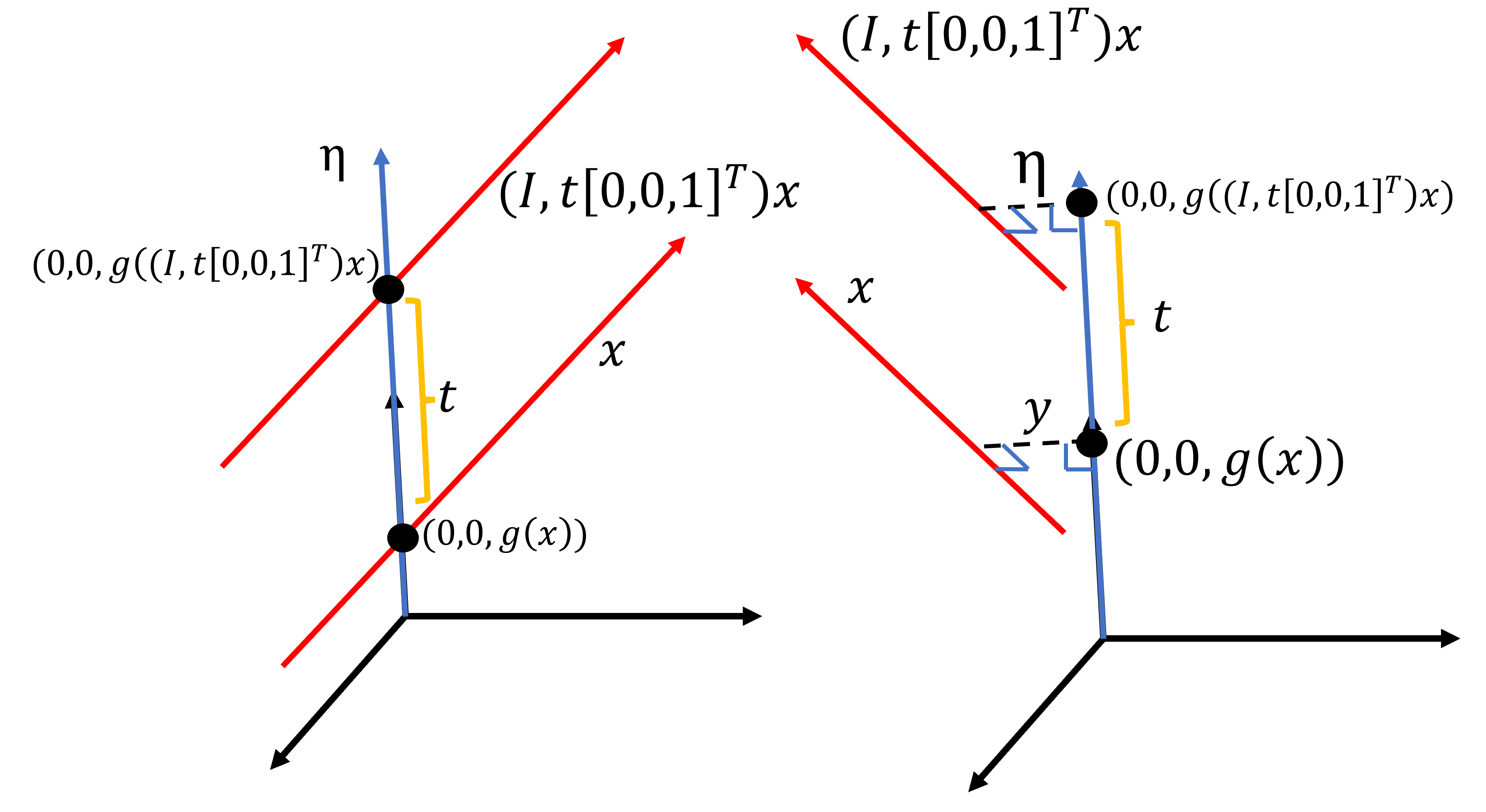}
   \caption{Visualization of $g(x)$. The left is the case that the ray $x$ and the ray $\eta$ are intersected, and the right is the case that the ray $x$ and the ray $\eta$ are not intersected. For the left, the point $Q$ is the intersection of $x$ and $\eta$, and $Q=(0,0,g(x))$; for the right, the point $Q$ is the intersection of the line $y$ and the ray $\eta$, where $y$ is perpendicular to both $\eta$ and $x$, and intersects with both $\eta$ and $x$. From the figure, in both cases, we can see that for any $t \in \mathbb{R}$, $g((0,t)x)=t+g(x)$. In general, we actually have for any $(\gamma,t) \in SO(2) \times \mathbb{R}$, $g((\gamma, t)x)=t+g(x)$. }
   \label{fig:constraint}
\end{figure}

The solution $\kappa$ also can be expressed as 
\begin{align}
\kappa(x)_t=\kappa_1(x)\kappa_2(x)_t
\label{whole_so_re}
\end{align}
for any $t \in \mathbb{R}$, and their constraint is also the same as Eq. \ref{constraint_k1} and Eq. \ref{constraint_k2}. As a result, the solution for $\kappa_1$ should be the same. We only need to solve $\kappa_2$:
\begin{align}
\kappa_2((\gamma,t')x)_t= e^{i\omega^2_{in}\langle [0,0,1]^T, \bm{d}_x \rangle t'}\kappa_2(x)_{t-t'}
\end{align} 
for any $(\gamma,t')\in SO(2) \times \mathbb{R}$.

When $\bm{d}_x \in \mathbb{S}^2-\left\{[0,0,1]^T, -[0,0,1]^T\right\}$,  

\begin{align}
\kappa_2(x)_t=f(d(\eta, x), \angle([0,0,1]^T,\bm{d}_x))e^{i\omega^2_{in}\langle [0,0,1]^T, \bm{d}_x \rangle g(x)}\delta(t-g(x)),
\label{solution_for_re}
\end{align}
where $f$ and $g$ are the same function as defined in \ref{solution_for_nonre}, and $\delta(t)=1$ only when $t=0$.

when $\bm{d}_x \in \left\{[0,0,1]^T, -[0,0,1]^T\right\}$, $\kappa_2(x)_t=0$ for any $t \in \mathbb{R}$.

\end{exg}

\begin{exg}
\label{filter_solution2point}
\textcolor{red}{$SE(3)$ equivariant convolution from $\mathcal{R}$ to $\mathbb{R}^3$:}
Following \cite{cohen2019general}, the convolution from rays to points becomes: 
\begin{align}
f_2^{l_{out}}(x)=\int_{\mathcal{R}}\kappa(s_2(x)^{-1}y)\rho_{in}(\text{h}_1(s_2(x)^{-1}s_1(y)))f^{l_{in}}_1(y)dy,
\label{light2r3}
\end{align}
where 
$\text{h}_1$ is the twist function corresponding to section $s_1:\mathcal{R} \rightarrow SE(3)$ defined aforementioned, $\rho_{in}$ is the group representation of $SO(2) \times \mathbb{R}$, corresponding to the feature type $l_{in}$, $s_2: \mathbb{R}^3 \rightarrow SE(3)$ is the section map defined in paper as $s_2(\bm{x})=(I,\bm{x})$.

In this paper, we give the analysis and solutions for the kernel 
where the input is the scalar field over the ray space, i.e.,$\rho_{in}=1$, the trivial group representation, which is also the case of our application in reconstruction.

The convolution is equivariant if and only if 
$$\kappa(h_2x)=\rho_{out}(h_2)\kappa(x),$$
for any $h_2 \in SO(3)$, 
where $\rho_{out}$ is the group representation of $SO(3)$ corresponding to the feature type $l_{out}$.

We can derive $\kappa(h_2x)=\rho_{out}(h_2)\kappa(x)$ analytically. For irreducible representation $\rho_{out}$ and any $x =(\bm{d}_x,\bm{m}_x) \in \mathcal{R}$, if $\|\bm{m}_x\|=0$, $\kappa(x)=cY^{l_{out}}(\bm{d}_x)$, where $c$ is an arbitrary constant and $Y^{l_{out}}$ is the spherical harmonics and $l_{out}$ is the 
order (type) of 
output tensor corresponding to the representation $\rho_{out}$;  With $\|\bm{m}_x\| \neq 0$, 
$\kappa(x)$ becomes $\rho_{out}(\hat{x})f(\|\bm{m}\|_x)$,  where $\hat{x}$ denotes the element $(\bm{d}_x,\frac{\bm{m}_x}{\|\bm{m}_x\|},\bm{d}_x\times \frac{\bm{m}_x}{\|\bm{m}_x\|})$ in $SO(3)$ and $f: \mathbb{R} \rightarrow \mathbb{R}^{(2l_{out}+1)\times 1}$. 

Similar to the convolution from rays to rays, we also can have the local support of the kernel. We set $\kappa(x) \neq 0$ when $\|\bm{m}_x\| \leq d_0$, otherwise $\kappa(x) = 0$. One can easily check that it doesn't break the equivariant constraint for the kernel. 

Specifically, when we set $d_0=0$,  the neighborhood of the target points in the convolution only includes the rays from all views going through the point. Hence,  we can simplify the convolution to $f_2^{l_{out}}(x)=\int_{d(y,x)=0}Y^{l_{out}}(\bm{d}_{s_2(x)^{-1}y})f^{in}_1(y)dy$.
This equation shows 
 that for every point $x$, we can treat the ray $y$ going through $x$ with feature $f^{in}_1$ as a point $y'$, where $y'-x =\bm{d}_{s_2(x)^{-1}y}$, as shown in figure \ref{figure:convertion}. 
 \begin{figure}[t]
  \centering
   \includegraphics[width=0.9\linewidth]{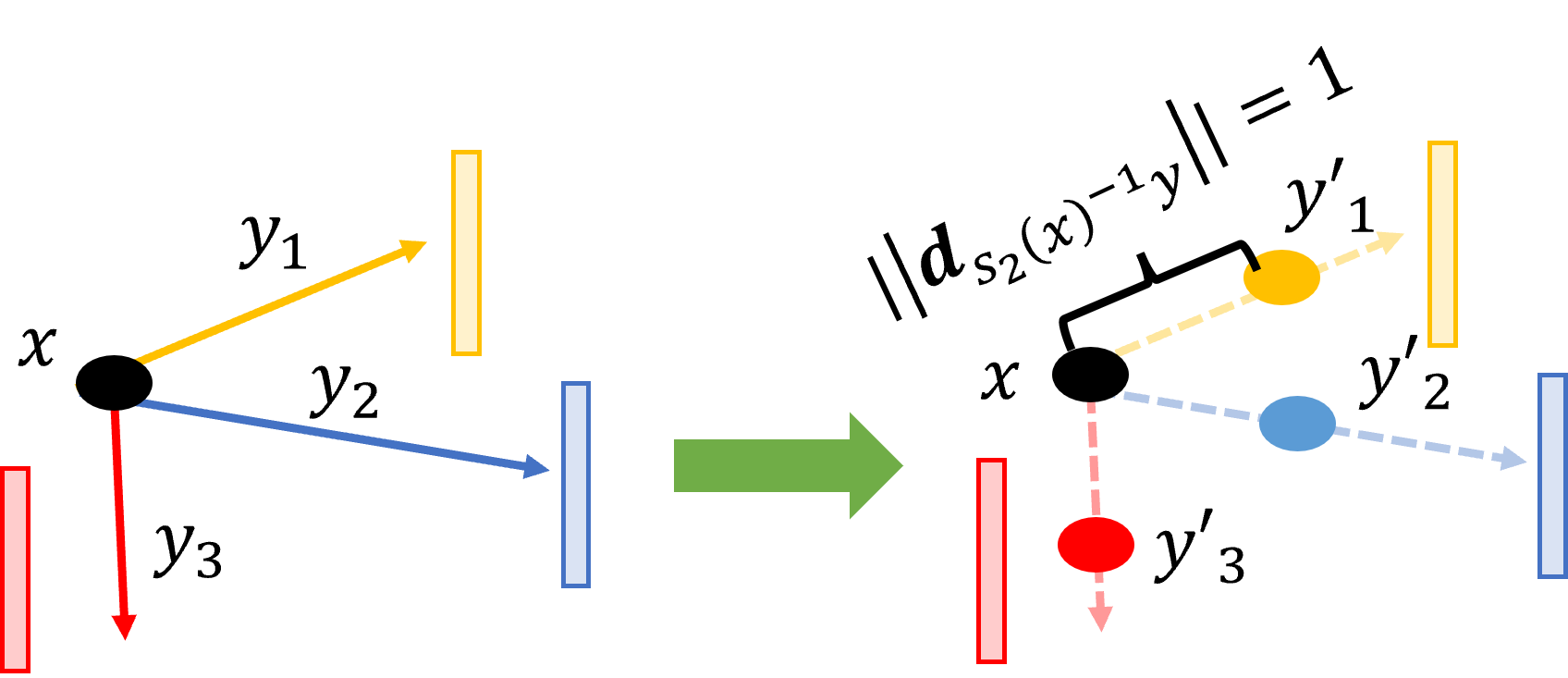}
   \caption{Interpreting rays $y_i$ as points $y'_i$}
   \label{figure:convertion}
\end{figure}
\end{exg}

\section{Equivariant 3D Reconstruction}
\label{equivariant_reconstruction}
\subsection{Approximation of the  Equivariant Convolution from Rays to Rays}

In practical $3D$ reconstruction, we have multiple views instead of the whole light field. Although the aforementioned convolution is defined on the continuous ray space, the equivariance still strictly holds when the ray sampling (pixels from camera views) is the same up to coordinate change. In this case, we will show how we adjust the equivariant convolution from rays to rays in this case and approximate it by an intra-view $SE(2)$-convolution.

\subsubsection{From light field to intra-view convolution} 
Following Fig.~\ref{fig:mvconv}, neighboring rays are composed of two parts: a set of rays from the same view and another set of rays from different views. For one ray $x$ in view $A$, the neighboring rays from view $B$ are in the neighborhood of the epipolar line of $x$ in view $B$. When the two views are close to each other, the neighborhood in the view $B$ would be very large.

The kernel solution in Ex. \ref{filter_solution} suggests that $\kappa(x)$ is related to $\angle (\bm{d}_x,[0,0,1]^T)$ and $d((x,\eta)$, where $\eta =([0,0,1]^T,[0,0,0]^T)$ as mentioned before. It would be memory- and time-consuming to memorize the two metrics beforehand or to compute the angles and distances on the fly.  Practically, the light field is only sampled from a few sparse viewpoints, which causes the relative angles of the rays in different views to be large and allows them to be excluded from the kernel neighborhood; therefore, in our implementation, the ray neighborhood is composed of only rays in the same view.
\begin{figure}[t]
  \centering
\includegraphics[width=0.9\linewidth]{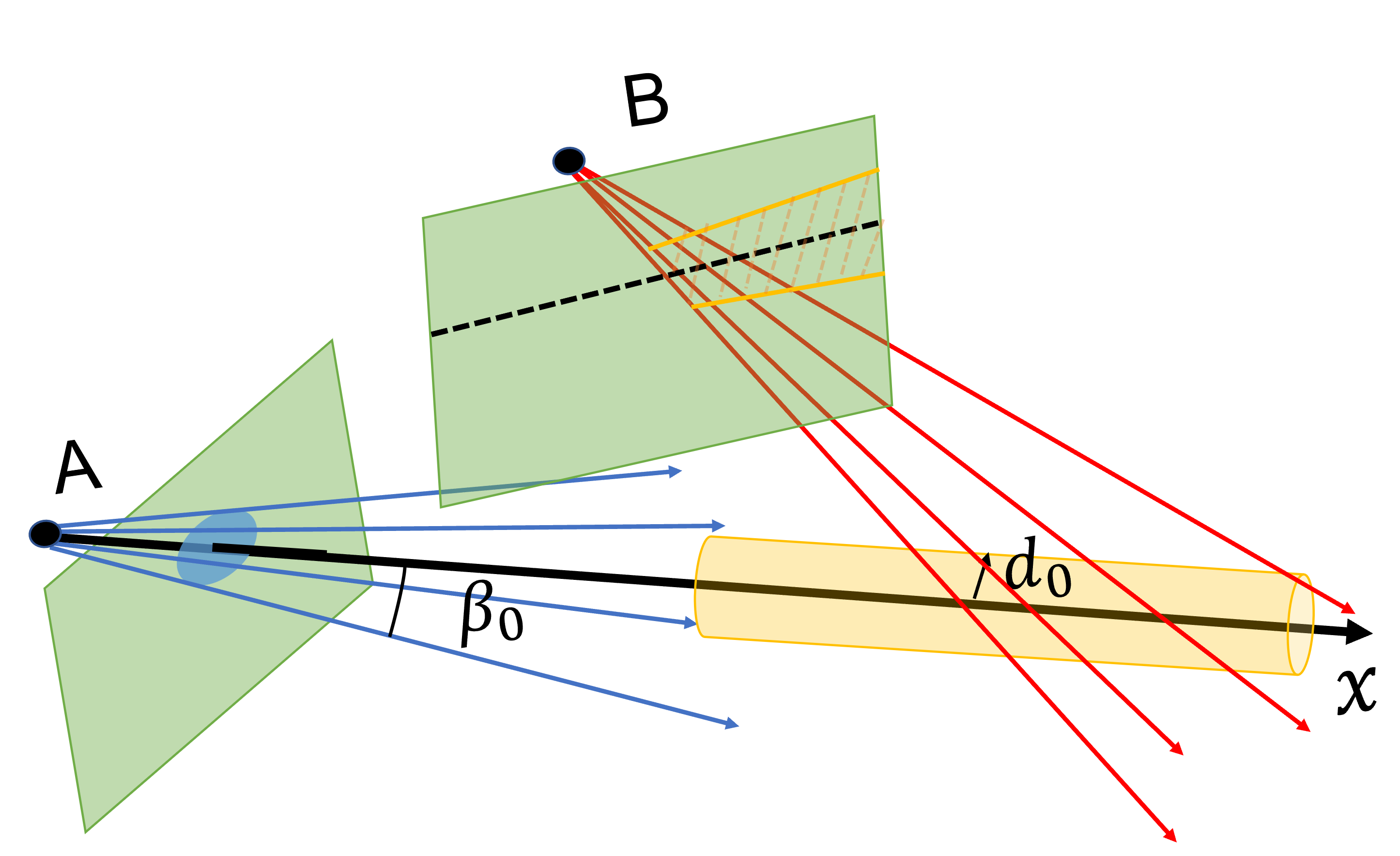}
   \caption{For simplification, we show a situation of two views. For a ray $x$ from view A, one part of the neighboring rays is from view A (the blue rays in the figure), $\mathcal{N}_A(x)$. For any ray $y \in \mathcal{N}_A(x)$, we have $d(y,x)=0$, and we require $\angle(\bm{d}_y,\bm{d}_x) \leq \beta_0$. The other part is from the other view B (the red rays in the figure). As illustrated in figure %
   \ref{fig:lightconv}, the neighboring rays always cross a cylinder around $x$; therefore, the neighboring rays from view B are the projection of the cylinder with radius $r=d_0$ in view B, that is, $\mathcal{N}_B$ is composed of the neighboring pixels of the epipolar line (the black dotted dash) corresponding to $x$ in view B. For any ray $y$ in the projection of the cylinder, we have $d(y,x) \neq d_0$. Since we require that $\angle(\bm{d}_y,\bm{d}_x) \leq \beta_0$ for any ray $y \in \mathcal{N}_B(x)$, $\mathcal{N}_B(x)$ is part of the projection of the cylinder, denoted as the shaded yellow part in view B.}
   \label{fig:mvconv}
\end{figure}

\subsubsection{From intra-view light field to spherical convolution}
After showing that a small kernel support in the case of sparse views affects only intra-view rays, we can prove that an intra-view light-field convolution is equivalent to a spherical convolution when we constrain the feature field types over $\mathcal{R}$.

We exploit the desired property that a feature defined on a ray is constant along the ray. This means that the translation part of the stabilizer group (translation along the ray) leaves the feature as is. In math terms, 
the irreducible representation for the translation $\mathbb{R}$ is the identity, which means that the field function is a scalar field for the translation group, with the formula $(\mathcal{L}_t f)(x) = f(t^{-1}x)$. We prove that, in this case, the intra-view convolution over rays is equivalent to the spherical convolution; please see Sec. \ref{light2sphere}.

\subsubsection{From SO(3)- to SE(2)-convolution}
While there is an established framework for 
spherical convolution using a Fourier transform \cite{cohen2018spherical,esteves2018learning, esteves2020spin}
it is not applicable in our case because the boundaries of the constrained field of view cause an explosion in the high frequencies of the spherical harmonics. We will make a compromise here and approximate the SO(3) convolution with an SE(2) convolution on the image plane by making the assumption that the field of view is small. One can see the rationale behind this approximation by keeping only the first order terms in the optical flow equation: the rotational term is only due to $\Omega_z$ while the translational term is $(-T_x-\Omega_y,-T_y+\Omega_x)$ with $(\Omega_x, \Omega_y,\Omega_z)$ as the angular velocity. We  provide  a  justification using the formalism of the previous paragraphs in appendix Sec. \ref{just_se2}.

\subsection{Ray Fusion: Equivariant Convolution and Transformer}
 To reconstruct a 3D object, we use an implicit function known as the signed distance function (SDF) defined on $\mathbb{R}^3$. As a result, we require an equivariant model that can transform features from rays to points to obtain the SDF. This can be achieved using the equivariant convolution in Sec. \ref{equifuse}
 and transformer in Sec. 3.2 in the paper \ref{equi_tr_over_rays}, 
 which allow us to transform features from the ray space to points in 3D space while maintaining equivariance.

\subsubsection{Equivariant Convolution from Rays to Points}
In this paper, we obtain the scalar feature field over rays after the SE(2)-equivaraint CNNs. As illustrated in figure %
\ref{fig:pipeline}
, we utilize the equivariant convolution (discussed in Sec. %
\ref{equifuse}
) to compute features for a query point by convolving over neighboring rays. Our experiments have shown that convolving only over rays that go through the point achieves the best results, and the equivariant kernel used for this convolution is provided in Ex.\ref{filter_solution2point}. Moreover, in the implementation, we can concatenate the input feature $f^{in}_1$ with the depth embedding of the query point $x$. While this theoretically breaks the ideal equivariance for continuous light fields, it does not affect the practical equivariance, as it is rare for two cameras to share the same ray.

\subsubsection{Equivariant Transformer from Rays to Points}
For the third step, we introduce an equivariant transformer in order to  alleviate the loss of expressivity due to the constrained 
 kernel $\kappa$ in Eq. \ref{light2r3}. Again, the attention key and values are generated from the feature attached to rays, while the query is generated from the feature attached to points.

In the implementation, we apply a transformer over the
rays going through the query point. We can continue to use the interpretation that treats any ray $y$ passing through the point $x$ as a point $y'$ such that $y'-x =\bm{d}_{s_2(x)^{-1}y}$, as shown in figure \ref{figure:convertion}. Since $y$ becomes point $y'$, the ray feature $f^{in}_1$ becomes the feature over $\mathbb{R}^3$ attached to ``points" $y'$. We can update the neighboring ray feature by directly concatenating the equivariant feature of the point to every ray feature before through a $SO(3)$ equivariant MLP. The transformer in Eq. %
\ref{paper:transformer} 
would be converted to the transformer in \cite{fuchs2020se} over $\mathbb{R}^3$. See appendix Sec. \ref{r3_trans} for details. The composition of the ray updating block and transformer block are shown in figure \ref{fig:transformer}.

\section{Proof of Equivalence of Intra-view Light Field Convolution and Spherical Convolution}
\label{light2sphere}
The property that a feature defined on a ray is constant along the ray means that the translation part of the stabilizer group (translation along the ray) leaves the feature as is. In math terms, 
the irreducible representation for the translation $\mathbb{R}$ is the identity, which means that the field function is a scalar field for the translation group, with the formula $(\mathcal{L}_t f)(x) = f(t^{-1}x)$. The equivariant condition on the kernel can then be simplified as 
$$\kappa((h,t)x)=\rho_{out}(h)\kappa(x)\rho_{in}(\text{h}_a^{-1}(h,\bm{d}_x)),$$
where $h \in SO(2)$ and $t \in \mathbb{R}$, $\rho_{in}$ and $\rho_{out}$ are irreducible representations for $SO(2)$, and $\text{h}_a$ is the twist function as shown in Ex. \ref{SE3_ex_sec_ray_space} that $\text{h}(g,x) =(\text{h}_a(R_g,\bm{d}_x), \text{h}_b(g,x))$,i.e., the twist of the fiber introduced by action of $SO(3)$ corresponding to the section map $s_a$ of $SO(3)$ in Ex. \ref{SO3_ex_sec} and Ex. \ref{SE3_ex_sec_ray_space}.
Now we describe the relationship between the intra-view light-field convolution and the spherical convolution:
\begin{proposition}
When the translation group acts on feature $f: \mathcal{R} \rightarrow V$ as $(\mathcal{L}_tf)(x)=f(t^{-1}x)$ for any $x \in \mathcal{R}$, the equivariant intra-view light-field convolution: 
$$f^{l_{out}}(x)=\int_{y \in \mathcal{N}(x)} \kappa(s(x)^{-1}y)\rho_{in}(\text{h}(s(x)^{-1}s(y)))f^{l_{in}}(y)dy$$
becomes a spherical convolution:
\begin{align}
    &f^{l_{out}}(x)\nonumber=
     \int_{\bm{d}_y \in \mathbb{S}^2} \kappa'(s_a(\bm{d}_x)^{-1}\bm{d}_y)\rho_{in}(\text{h}_a(s_a(\bm{d}_x)^{-1}s_a(\bm{d}_y)))\\
    & f'^{l_{in}}(\bm{d}_y)d\bm{d}_y,
    \label{spherical_conv}
\end{align}
where 
 $f'^{l_{in}}(\bm{d}_y) = f^{l_{in}}(\bm{d}_y,\bm{c}_x \times\bm{d}_y)$, $\bm{c}_x$ denotes the camera center that $x$ goes through, $s_a$ is the section map of $SO(3)$ as defined in appendix Ex. \ref{SO3_ex_sec}, and $\kappa'(s_a(\bm{d}_x)^{-1}\bm{d}_y)=\kappa(s_a(\bm{d}_x)^{-1}\bm{d}_y, (s(x)^{-1}\bm{x}_c) \times (s_a(\bm{d}_x)^{-1}\bm{d}_y))$.

\end{proposition}

\begin{proof}
The $SE(3)$ equivariant convolution over rays transforms into intra-view convolution when the neighboring lights are in the same view. Moreover, the simplified kernel constraint derived in the paper is that for any $(h,t) \in SO(2)\times \mathbb{R}$ and $x= (\bm{d}_x,\bm{m}_x) \in \mathcal{R}$ : 
$$\kappa((h,t)x)=\rho_{out}(h)\kappa(x)\rho_{in}(\text{h}_a^{-1}(h,\bm{d}_x)),$$ where $\text{h}_a: SO(3) \times \mathbb{S}^2 \rightarrow SO(2)$ is the twist function: $\text{h}_{a}(g,\bm{d})=s_a(g\bm{d})^{-1}gs_a(\bm{d})$ for any $g \in SO(3)$ and $\bm{d} \in \mathbb{S}^2$.

With the simplified kernel constraint, we can prove that intra-view light field convolution is equivalent to spherical convolution:

\begin{align}
    &f^{l_{out}}(x)\nonumber\\
    &=\int_{d(y, \bm{c}_x)=0} \kappa(s(x)^{-1}y)\rho_{in}(\text{h}(s(x)^{-1}s(y)))f^{l_{in}}(y)dy \label{ray_conv}\\
    &=\int_{d(y, \bm{c}_x)=0} \kappa(s(x)^{-1}y)\rho_{in}(\text{h}_a(s_a(\bm{d}_x)^{-1}s_a(\bm{d}_y)))f^{l_{in}}(y)dy
    \label{simple_ray_conv}\\
     &=\int_{\bm{d}_y \in \mathbb{S}^2} \kappa(s_a(\bm{d}_x)^{-1}\bm{d}_y, s(x)^{-1}\bm{x}_c \times (s_a(\bm{d}_x)^{-1}\bm{d}_y) )\nonumber\\
    &\qquad \rho_{in}(\text{h}_a(s_a(\bm{d}_x)^{-1}s_a(\bm{d}_y)))f^{l_{in}}(\bm{d}_y,\bm{c}_x \times\bm{d}_y )d\bm{d}_y
    \label{simple_ray_conv2}\\
    &= \int_{\bm{d}_y \in \mathbb{S}^2} \kappa'(s_a(\bm{d}(x))^{-1}\bm{d}_y)\rho_{in}(\text{h}_a(s_a(\bm{d}_x)^{-1}s_a(\bm{d}_y))) \nonumber\\
    &\qquad f'^{l_{in}}(\bm{d}_y)d\bm{d}_y.
    \label{simple_ray_conv3}
\end{align}
In line \ref{ray_conv}, $\bm{c}_x$ is the camera center that $x$ goes through.

The line \ref{ray_conv} is equal to the line \ref{simple_ray_conv} because we assume that 
 the irreducible
representation for the translation $\mathbb{R}$ is the identity as mentioned in the paper.

From line \ref{simple_ray_conv} to line \ref{simple_ray_conv2}, We can replace $s(x)^{-1}y$ with $$(s_a(\bm{d}_x)^{-1}\bm{d}_y, (s(x)^{-1}\bm{x}_c) \times (s_a(\bm{d}_x)^{-1}\bm{d}_y))$$
due to the facts that $s_a(\bm{d}_x)^{-1}\bm{d}_y =\bm{d}_{s(x)^{-1}y}$ and point $s(x)^{-1}\bm{x}_c$ is on the ray $s(x)^{-1}y$. Since $y$ goes through $\bm{c}_x$, we can replace $y$ with $(\bm{d}_y,\bm{c}_x \times\bm{d}_y)$. 

From line \ref{simple_ray_conv2} to \ref{simple_ray_conv3}, we have $f'^{l_{in}}(\bm{d}_y) = f^{l_{in}}(\bm{d}_y,\bm{c}_x \times\bm{d}_y)$ because  $\bm{c}_x$ is fixed for any view. Additionally, from line \ref{simple_ray_conv2} to \ref{simple_ray_conv3} we replace $$\kappa(s_a(\bm{d}_x)^{-1}\bm{d}_y, (s(x)^{-1}\bm{x}_c) \times (s_a(\bm{d}_x)^{-1}\bm{d}_y))$$ with $\kappa'(s_a(\bm{d}_x)^{-1}\bm{d}_y)$. 
It is because according to 
$$\kappa((h,t)x)=\rho_{out}(h)\kappa(x)\rho_{in}(\text{h}_a^{-1}(h,\bm{d}_x)),$$
we have $\kappa((e,t)x)=\kappa(x)$ for any $t \in \mathbb{R}$, where $e$ is the identity element in $SO(2)$; thus when $t= ((-s(x)^{-1} \bm{x}_c))^T[0,0,1]^T$, we have 
\begin{align}
&\kappa(s_a(x)^{-1}\bm{d}_y, s(x)^{-1}\bm{x}_c \times (s_a(x)^{-1}\bm{d}_y)) \nonumber\\
&= \kappa(s_a(x)^{-1}\bm{d}_y, (s(x)^{-1}\bm{x}_c +t[0,0,1]^T)\times (s_a(x)^{-1}\bm{d}_y))\label{equi_a}\\
&=\kappa((s_a(x)^{-1}\bm{d}_y, [0,0,0]^T)\label{equi_b}\\
&=\kappa'((s_a(x)^{-1}\bm{d}_y).\nonumber
\end{align}
Line \ref{equi_a} is equal to \ref{equi_b} because $s(x)^{-1}\bm{x_c}$ is always on the $z$ axis, and thus $s(x)^{-1}\bm{x}_c +t[0,0,1]^T=[0,0,0]^T$. 
\end{proof}

\section{Spherical Convolution Expressed in Gauge Equivariant Convolution Format }
\label{gauge}
Group convolution is a special case of gauge equivariant convolution \cite{weiler2021coordinate}, where gauge equivariant means the equivariance with respect to the transformation of the section map (transformation of the tangent frame). In the following paragraph we give the elaborated definition of gauge equivariance for the sphere.

 Suppose $f: \mathbb{S}^2 \rightarrow V$ is the field function corresponding to the section choice $s_a:\mathbb{S}^2 \rightarrow SO(3)$, we use $\mathcal{L}_{s_a \rightarrow  s'_a}$ acting on $f$ to denote the change of section map from $s_a$ to $s'_a$: $(\mathcal{L}_{s_a \rightarrow s'_a}f)(x)=\rho(s_a(x)^{-1}s'_a(x))^{-1}f(x)$, where $\rho$ is the irreducible representation of $SO(2)$ corresponding to the field type of $f$. The convolution $\Phi$ is gauge equivariant when $\Phi(\mathcal{L}_{s_a \rightarrow s'_a}f) =\mathcal{L}_{s_a \rightarrow s'_a}(\Phi(f))$.
 
In this section, we show that the spherical convolution can be expressed in terms of the gauge equivariant convolution \cite{cohen2019gauge} , which provides the convenience for us to verify the approximation of spherical convolution through the $SE(2)$ convolution:
$$f^{l_{out}}(x) =\int_{y \in \mathcal{N}(x)}\kappa'(s(x)^{-1}y)\rho_{in}(h_{y \rightarrow x})^{-1}f^{l_{in}}(y)dy,$$
 where $\kappa'(hx)=\rho_{out}(h)\kappa'(x) \rho_{in}^{-1}(h)$ for any $h \in SO(2)$.

Since the focus of this section's discussion is spherical convolution, here we use $s(x)$ to denote $s_a(x)$ for any $x \in \mathbb{S}^2$.

For any $x, y\in \mathbb{S}^2$, $s(x)[1,0,0]^T$, $s(x)[0,1,0]^T$ attached to $x$ are tangent vectors on $x$, we parallel transport $s(x)[1,0,0]^T$ and  $s(x)[0,1,0]^T$ along the geodesic between $x$ and $y$ and get two tangent vectors on $y$, denoted as $s(x \rightarrow y)_1$ and $s(x \rightarrow y)_2$ as shown in the figure \ref{fig:gauge_equi}, where the parallel transport along a smooth curve  is a way
to translate a vector ``parallelly" based on the affine connection, that is, for a smooth curve $\gamma: [0,1] \rightarrow \mathbb{S}^2$, the  parallel transport $X: \text{Im}(\gamma) \rightarrow \mathcal{T}\mathbb{S}^2$ along the curve $\gamma$ satisfies that $\nabla_{\dot{\gamma}(t)}X=0$, where $\text{Im}(\gamma)=\left\{\gamma(t)|t \in [0,1]\right\}$ and $\nabla$ is the affine connection.

$s(x \rightarrow y)_1$ and $s(x \rightarrow y)_2$ need to undergo a transformation in $SO(2)$ to align with $s(y)[1,0,0]^T$ and $s(y)[0,1,0]^T$ on y as shown in the figure \ref{fig:gauge_equi}.  We denote the transformation as $h_{x \rightarrow y}$.
\begin{figure}[t]
  \centering
   \includegraphics[width=0.9\linewidth]{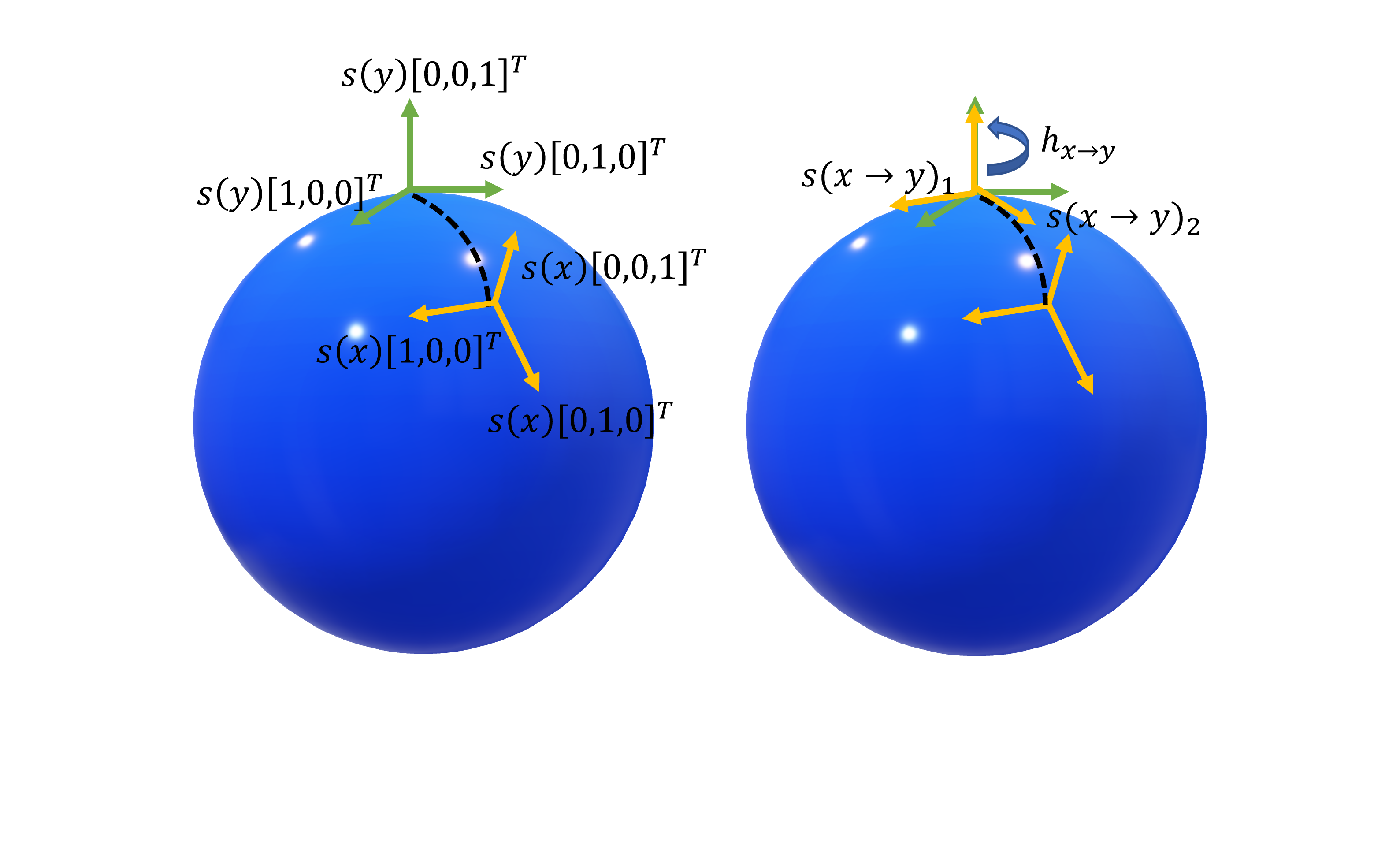}
   \caption{Illustration of $h_{x \rightarrow y}$. $s(x)[1,0,0]^T$ and  $s(x)[0,1,0]^T$ (yellow) attached to $x$ are tangent vectors on $x$.  We parallel transport $s(x)[1,0,0]^T$ and  $s(x)[0,1,0]^T$ along the geodesic (black dashed line) between $x$ and $y$. The transported tangent vectors need to undergo a transformation $h_{x \rightarrow y}$ in $SO(2)$ to align with the vectors $s(y)[1,0,0]^T$ and  $s(y)[0,1,0]^T$ (green) attached to $y$.}
   \label{fig:gauge_equi}
\end{figure}

With the above notation, the spherical convolution can be expressed as: 
\begin{align*}
f^{l_{out}}(x) 
&=\int_{y \in \mathcal{N}(x)}\kappa(s(x)^{-1}y)\rho_{in}(\text{h}(s(x)^{-1}s(y)))f^{l_{in}}(y)dy\\
&= \int_{y \in \mathcal{N}(x)} \kappa(s(x)^{-1}y)\rho_{in}(h_{s(x)^{-1}y \rightarrow\eta})\\
&\rho_{in}(h_{s(x)^{-1}y \rightarrow\eta})
^{-1}
\rho_{in}(\text{h}(s(x)^{-1}s(y))f^{l_{in}}(y)dy\\
&=\int_{y \in \mathcal{N}(x)} \kappa(s(x)^{-1}y)\rho_{in}(h_{s(x)^{-1}y \rightarrow\eta})\\
&\rho_{in}(h_{y \rightarrow x})^{-1}f^{l_{in}}(y)dy\\
&=\int_{y \in \mathcal{N}(x)}\kappa'(s(x)^{-1}y)\rho_{in}(h_{y \rightarrow x})^{-1}f^{l_{in}}(y)dy,
\end{align*}
where $\eta =[0,0,1]^T$, the fixed origin point in $\mathbb{S}^2$, and $\kappa'(x)=\kappa(x)\rho_{in}(h_{x \rightarrow\eta})^{-1}$ for any $x \in \mathcal{N}(\eta)$.

 We can derive the equivariant condition that $\kappa'$ should satisfy: 
 \begin{align*}
 \kappa'(hx)&=\kappa(hx)\rho_{in}(h_{hx \rightarrow \eta})^{-1}\\
 &=\rho_{out}(h)\kappa(x)\rho_{in}(\text{h}(h,x))^{-1}\rho_{in}(h_{hx \rightarrow \eta})\\
 &=\rho_{out}(h)\kappa(x) \rho_{in} (h_{x\rightarrow \eta})^{-1}\rho_{in}(h^{-1})\\
 &= \rho_{out}(h)\kappa'(x) \rho_{in}^{-1}(h).
 \end{align*}
Therefore, the spherical convolution can be expressed as the gauge equivariant convolution format:

$$f^{l_{out}}(x) =\int_{y \in \mathcal{N}(x)}\kappa'(s(x)^{-1}y)\rho_{in}(h_{y \rightarrow x})^{-1}f^{l_{in}}(y)dy,$$
 where $\kappa'(hx)=\rho_{out}(h)\kappa'(x) \rho_{in}^{-1}(h)$ for any $h \in SO(2)$.

\section{Converting Spherical Convolution to $SE(2)$ Equivariant Convolution}
\label{just_se2}
As stated in Sec. \ref{gauge}, spherical convolution is gauge equivariant with respect to the choice of section map $s_a$, and the spherical convolution can be written as gauge equivariant convolution.  In this section, we use the gauge equivariant convolution to analyze the $SE(2)$ equivariant convolution's approximation of spherical convolution.

Since each view performs spherical convolution on its own, we only analyze the convolution for one view for the sake of simplicity. We use $V$ to denote the space of the rays in the same view, where $V \subset \mathbb{S}^2$. For any $x \in V$, 
we can choose the section map $s_a$ such that $h_{x\rightarrow o}=e$, where $o \in \mathbb{S}^2$  that $o$ aligns with the optical axis as shown in the figure \ref{fig:section_choice}. Again, we use $s(x)$ to denote $s_a(x)$ for any $x \in \mathbb{S}^2$ in this section.

\begin{figure}[t]
  \centering
   \includegraphics[width=0.8\linewidth]{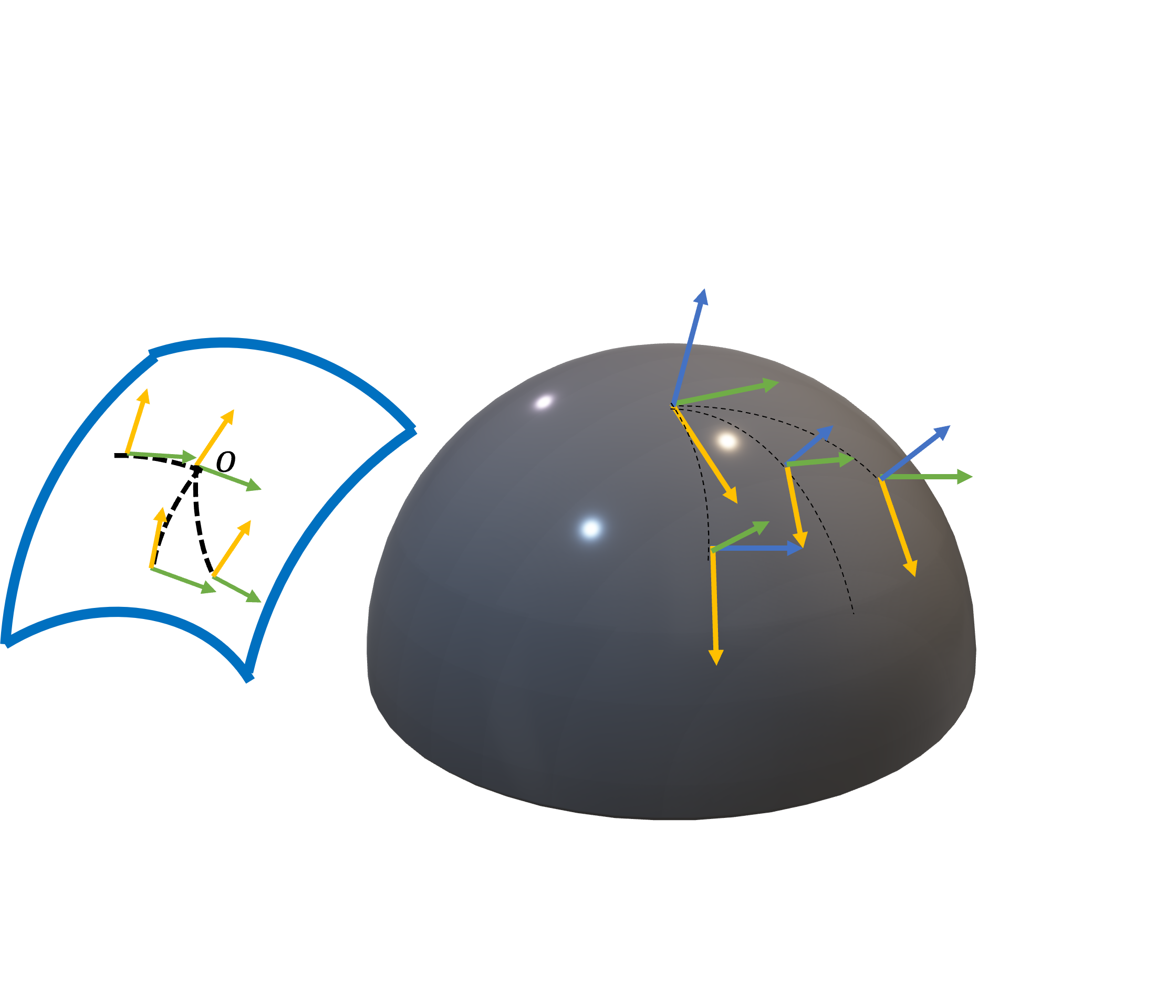}
   \caption{Section choice for every view}
   \label{fig:section_choice}
\end{figure}

When $FOV$ is small, for any $x,y \in V$, we can have such approximation: $h_{x \rightarrow y}=e$. Then the above gauge equivariant convolution in Sec. \ref{gauge} can be approximated as 
\begin{align*}
f^{l_{out}}(x)& =\int_{y \in \mathcal{N}(x)}\kappa'(s(x)^{-1}y)f^{l_{in}}(y)dy\\
&\xlongequal{t=s(x)^{-1}y}\int_{t \in \mathcal{N}(\eta)}\kappa'(t)f^{l_{in}}(s(x)t)dt,
\end{align*}
 where $\eta=[0,0,1]^{T}$, the fixed origin in $\mathbb{S}^2$, and $\kappa'(hx)=\rho_{out}(h)\kappa'(x) \rho_{in}^{-1}(h)$ for any $h \in SO(2)$.

Additionally, as illustrated in figure \ref{fig:projection_map}, we have a map from $V$ to the projection points on the picture plane represented as $\omega: V \rightarrow \mathbb{R}^2$, where $\omega(o)$ is defined as $[0, 0]^T$.   When $FOV$ is small, we have such approximation that for any $h \in SO(2)$,
$t \in \mathcal{N}(\eta)$, and $x \in V$, 
$$\omega(s(x)t) \approx \omega(x)+\omega(s(o)t).$$
It is because
\begin{align*}
&\omega(s(x)t)= \omega(x)+\omega(s(o)t)\\
&+ r(\frac{sin\beta_t}{cos\beta_t}-\frac{sin\beta_t}{cos\beta_xcos(\beta_x+\beta_t)}), 
\end{align*}
and we have
\begin{align*}
&lim_{t \rightarrow \eta}r(\frac{sin\beta_t}{cos\beta_t}-\frac{sin\beta_t}{cos\beta_xcos(\beta_x+\beta_t)}) \\
&=r(tan\beta_x)^2\beta_t +o(\beta^2_t), 
\end{align*}
when $\beta_x$ is small (FOV is small), the approximation stands.

Then $f^{l_{out}}(x) =\kappa'(t)f^{l_{in}}(s(x)t)dt$ can be approximately conducted in the image plane:
\begin{align}
&f'^{l_{out}}(\omega(x)) \nonumber\\
&=\int_{\omega(s(o)t) \in \mathcal{N}([0,0]^T)}\kappa''(\omega(s(o)t)) f'^{l_{in}}(\omega(x)+\omega(s(o)t)) \nonumber \\
&\qquad d(\omega(s(o)t)),
\label{SE(2)}
\end{align}
where for any $x \in \mathbb{S}^2$, $f'(\omega(x))=f(x)$, and for any $t \in \mathcal{N}(\eta)$,
 $\kappa''(\omega(s(o)t))=\kappa'(t)$.

Since for any $h \in SO(2)$  and any $t \in \mathcal{N}(\eta)$, $\omega(s(o)ht)=h\omega(s(o)t)$, we have for any  $h \in SO(2)$  and any $t \in \mathcal{N}(\eta)$,
\begin{align*}
&\kappa''(h\omega(s(o)t))=\kappa''(\omega(s(o)ht))=\kappa'(ht)\\
&=\rho_{out}(h)\kappa'(t)\rho_{in}^{-1}(h)=\rho_{out}(h) \kappa''(s(o)t)\rho_{in}^{-1}(h)\\
&\xlongequal{p=\omega(s(o)t) \in \mathbb{R}^2}k''(hp)\\
&=\rho_{out}(h) \kappa''(p)\rho_{in}^{-1}(h).
\end{align*}
Therefore, convolution \ref{SE(2)} is exactly $SE(2)$ equivariant convolution and it can be used to approximate the spherical convolution.

 In other words, we can intuitively approximate the equivariant convolution over the partial sphere using the $SE(2)$ equivariant network when the distortion of the sphere and the tangent plane of the optical axis is modest.

\begin{figure}[t]
  \centering
   \includegraphics[width=0.4\linewidth]{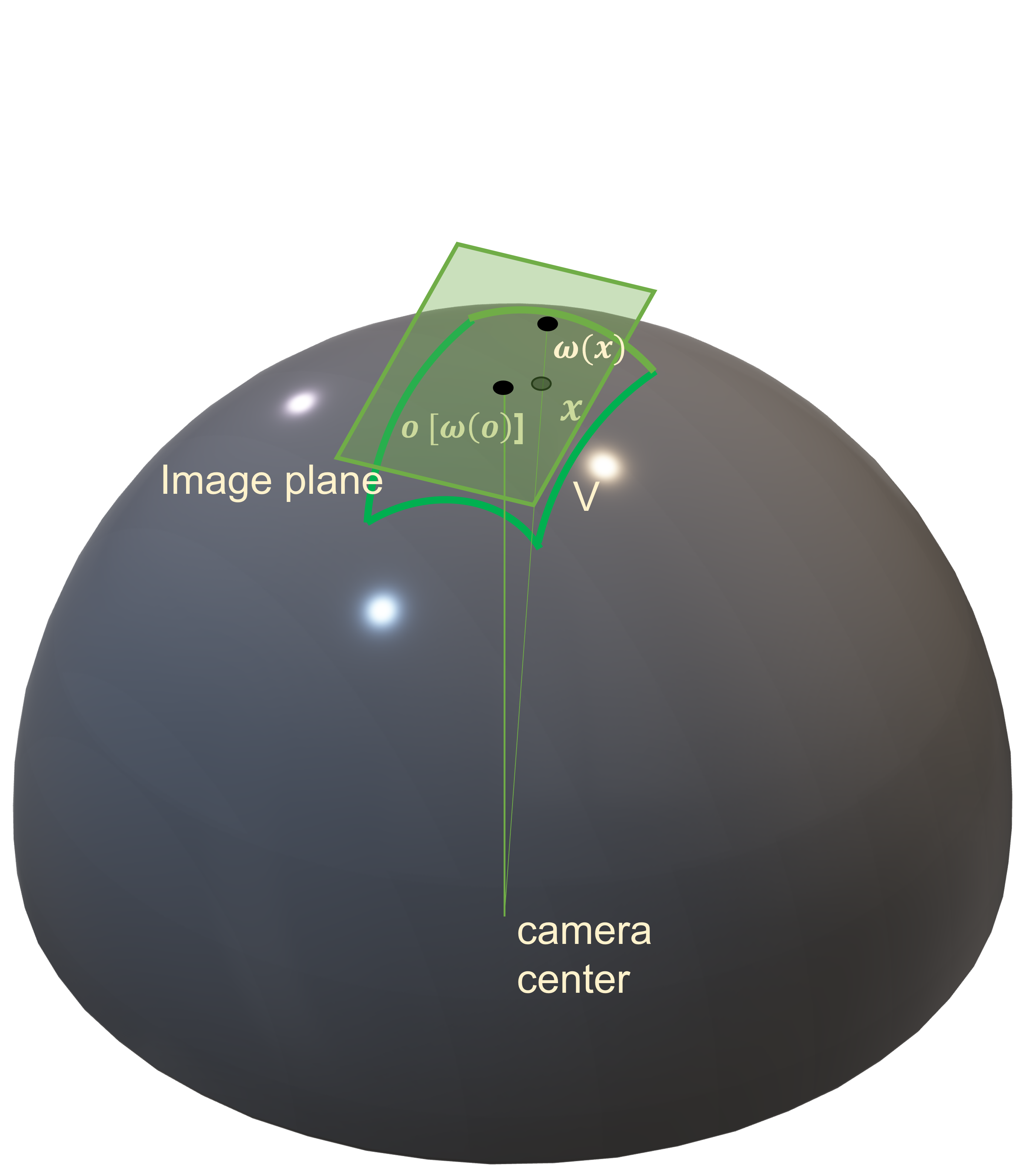}
   \caption{Illustration of projection map $\omega$}
   \label{fig:projection_map}
\end{figure}

\section{Construction of Features in Equivariant Light Field Transformer}
\label{construction of features in transformer}
Noted that $f_2^{out}$, $f_2^{in}$ and $f_1^{in}$ are features that are composed of fields of different types, denoted as $f_2^{out}=\oplus_{i}f_2^{l_{out_i}}$, $f_2^{in}=\oplus_{i}f_2^{l_{in_i}}$, and $f_1^{in}=\oplus_{i}f_1^{l'_{in_i}}$\footnote{Since here the homogeneous spaces of input and output might be different, so as the stabilizer groups, we use $l$ and $l'$ to denote the representations of different stabilizer groups.}. $f_k$, $f_q$, and $f_v$ are constructed equivariant key features, query features, and value features, respectively, which are composed of fields of different types as well. 

 We use $f_k = \oplus_{i} f_k^{l_{k_i}}$, $f_q = \oplus_{i} f_q^{l_{k_i}}$, and  $f_v = \oplus_{i} f_v^{l_{v_i}}$ to denote $f_k$, $f_q$ and $f_v$, respectively.
We construct the features $f_k$,$f_q$ and $f_v$ through the equivariant kernels $\kappa_k =\oplus_{j,i}\kappa_k^{l_{k_j},l'_{in_i}}$, $\kappa_v=\oplus_{j,i}\kappa_v^{l_{v_j},l'_{in_i}}$ and equivariant matrix $W_q=\oplus_{j,i} W_q^{l_{k_j},l_{in_i}}$:
\begin{align}
&f^{l_{k_j}}_k(x,y,f_1^{in}) \nonumber\\
&=\sum_{i}\kappa_k^{l_{k_j},l'_{in_i}}(s_2(x)^{-1}y)\rho_1^{l'_{in_i}}(\text{h}_1(s_2(x)^{-1}s_1(y)))f_1^{l'_{in_i}}(y);\\
&f^{l_{v_j}}_v(x,y,f_1^{in}) \nonumber \\
&=\sum_{i}\kappa_v^{l_{v_j},l'_{in_i}}(s_2(x)^{-1}y)\rho_1^{l'_{in_i}}(\text{h}_1(s_2(x)^{-1}s_1(y)))f_1^{l'_{in_i}}(y);\\
&f^{l_{k_j}}_q(x,f_2^{in})=\sum_i W_q^{l_{k_j},l_{in_i}}f_2^{l_{in_i}}(x),
\label{construction_of_fq}
\end{align} 

where for any $i,j$, any $h_2 \in SO(3)$, and any $x \in \mathcal{R}$
$\kappa_k^{l_{k_j},l'_{in_i}}$ and 
$\kappa_v^{l_{v_j},l'_{in_i}}$ should satisfy that:
$$\kappa_k^{l_{k_j},l'_{in_i}}(h_2x)=\rho_2^{l_{k_j}}(h_2)\kappa_k^{l_{k_j},l'_{in_i}}(x)\rho_1^{l'_{in_i}}(\text{h}^{-1}_1(h_2,x));$$
$$\kappa_v^{l_{v_j},l'_{in_i}}(h_2x)=\rho_2^{l_{v_j}}(h_2)\kappa_v^{l_{v_j},l'_{in_i}}(x)\rho_1^{l'_{in_i}}(\text{h}^{-1}_1(h_2,x)),$$

where $\text{h}_1(h_2,x)=s_1(h_2x)^{-1}h_2s_1(x)$ is the twist function, and for any $i,j$ and any $h_2 \in SO(3)$, $W_q^{l_{k_j},l_{in_i}}$ satisfies that:

\begin{align}
\rho_2^{l_{k_j}}(h_2)W_q^{l_{k_j},l_{in_i}}=W_q^{l_{k_j},l_{in_i}}\rho_1^{l_{in_i}}(h_2).
\label{Q construction}
\end{align}

When the group representation is irreducible representation, due to Schur's Lemma, we have  $W_q^{l_{k_j},l_{in_i}}=cI$ when $l_{k_j}=l_{in_i}$, where $c$ is an arbitrary real number, otherwise $W_q^{l_{k_j},l_{in_i}}=\bm{0}$.

\section{Proof for Equivariance of Light Field Transformer}
\label{proof_trans}
The equivariant light field transformer defined in the paper reads:
\begin{align}
&f^{out}_2(x) \nonumber \\
= &\sum_{y \in \mathcal{N}(x)} \frac{exp(\langle f_q(x,f^{in}_2), f_k(x,y,f^{in}_1)\rangle)}{\sum_{y \in \mathcal{N}(x)}exp(\langle f_q(x,f^{in}_2) f_k(x,y,f^{in}_1)\rangle} \nonumber \\
&f_v(x,y,f^{in}_1)) 
\label{transformer1}
\end{align}
is in a general form. %

According to \cite{cohen2019general}, one can prove that 
$f_q$, $f_k$ and $f_v$ are equivariant, 
that is, for any $g \in SE(3)$, $x \in \mathbb{R}^3$ and $y \in \mathcal{R}$, 

$$f^{l_{k_j}}_q(g \cdot x, \mathcal{L}^{in}_g(f_2^{in}))= \rho_2^{l_{k_j}}(\text{h}_2(g^{-1},g\cdot x)^{-1})f^{l_{k_j}}_q(x, f_2^{in});$$

$$f^{l_{k_j}}_k(g \cdot x, g\cdot y, \mathcal{L'}^{in}_g(f_1^{in}))= \rho_2^{l_{k_j}}(\text{h}_2(g^{-1},g\cdot x)^{-1})f^{l_{k_j}}_k(x,y, f_1^{in});$$

$$f^{l_{v_j}}_v(g \cdot x, g\cdot y, \mathcal{L'}^{in}_g(f_1^{in}))= \rho_2^{l_{v_j}}(\text{h}_2(g^{-1},g\cdot x)^{-1})f^{l_{v_j}}_v(x,y, f_1^{in}),$$

where $\mathcal{L}^{in}$ and $\mathcal{L'}^{in}$
 are group action of $SE(3)$ on $f_2^{in}$ and $f_1^{in}$, respectively.

The inner product $\langle f_q, f_k \rangle =
\sum_i (\overline{f_q^{l_{k_i}}})^T f_k^{l_{k_i}}$ is invariant due to the property of unitary representation, which results in the equivariance of the transformer.

 \section{From $SE(3)$ Equivariant Transformer in Ray Space to $SE(3)$ Equivariant Transformer in Euclidean Space}
 \label{r3_trans}
 In our implementation for the reconstruction task, the attention model is always only applied over the rays going through the points. We can continue to use the interpretation in the convolution from ray space to $\mathbb{R}^3$ in Ex. \ref{filter_solution2point} 
 that treats any ray $y$ passing through the point $x$ as a point $y'$ such that $y'-x =\bm{d}_{s_2(x)^{-1}y}$ as shown in the figure \ref{figure:convertion}.

After we get the initial feature of query points through equivariant convolution from $\mathcal{R}$ to $\mathbb{R}^3$, we update the neighboring ray feature by directly concatenating the query point feature to every ray feature before through a $SO(3)$ equivariant MLP as shown in the figure \ref{fig:transformer}. $SO(3)$ equivariant MLP is composed of an equivariant nonlinear layer and self-interaction layer as in the tensor field networks \cite{thomas2018tensor}.

\begin{figure}[t]
  \centering
   \includegraphics[width=0.6\linewidth]{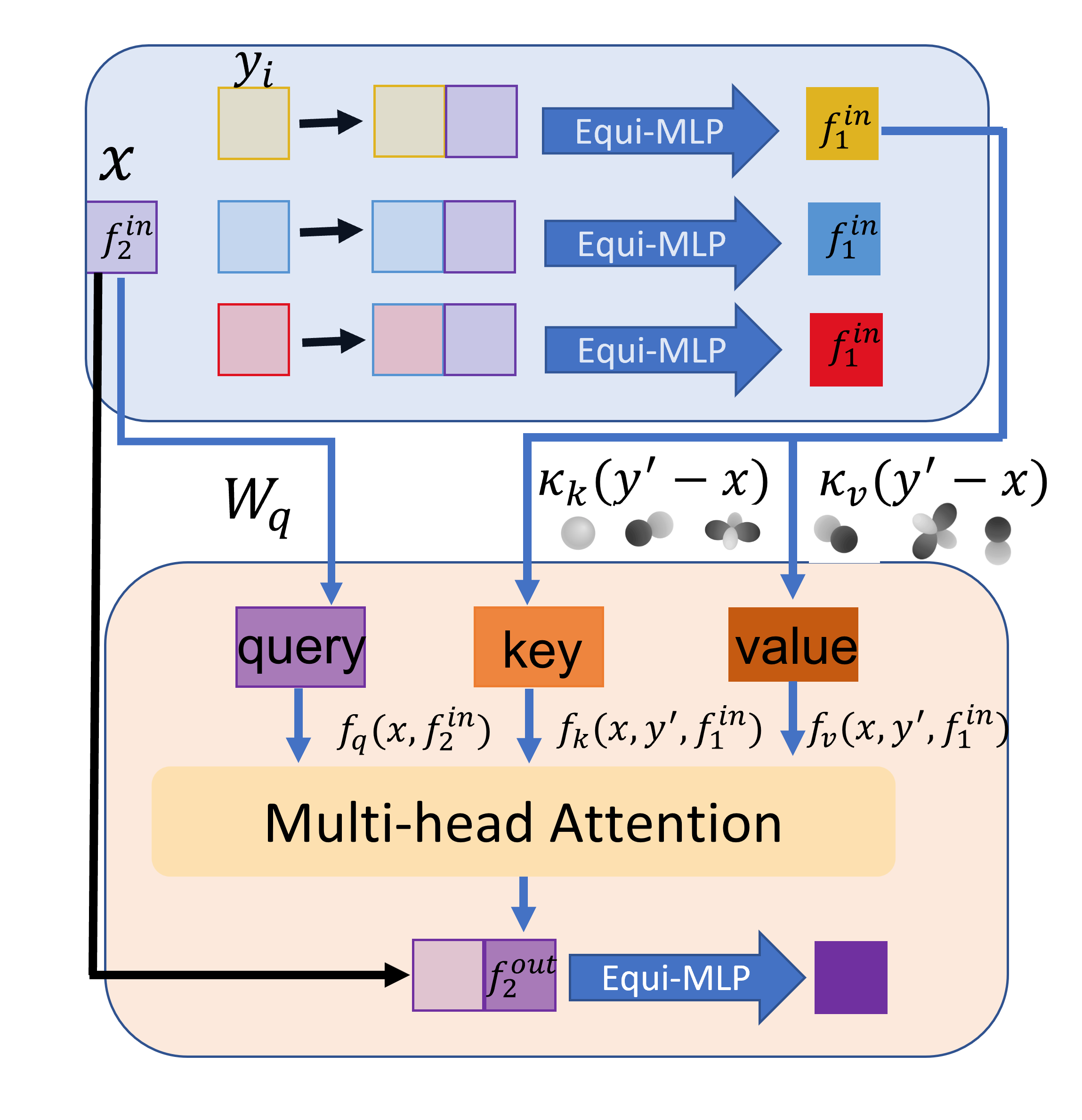}
   \caption{The structure of ray updating and $SE(3)$ transformer. We treat any ray $y$ going through point $x$ as a point $y'  \in \mathbb{R}^3$ such that  $y'-x = \bm{d}_{s_2(x)^{-1}y}$. The blue block indicates the ray feature update, and the pink block is the equivariant attention model. For the ray feature updating, the point feature (lavender) is concatenated to every ray feature (light yellow, light blue, and light red) and goes through an equivariant MLP. For the transformer, we get the equivariant query, key, and value feature through
   the designed linear matrix $W_q$, designed kernels $\kappa_k$ and $\kappa_v$, then apply multi-head attention to obtain the output point feature, which can subsequently be fed into the next ray feature updating and $SE(3)$ transformer block.}
   \label{fig:transformer}
\end{figure}

Since $y$ becomes point $y'$, and $f^{in}_1$ is the feature over $R^3$ attached to ``points" $y'$, it becomes $\oplus_{i}f_1^{l_{in_i}}$\footnote{Since here $f^{in}_1$ is the fields over $\mathbb{R}^3$, we use $l$ instead of $l'$ as the denotation}. Then transformer \ref{transformer1} would be converted to the transformer in \cite{fuchs2020se} over $\mathbb{R}^3$:

\begin{align}
&f^{out}_2(x) \nonumber \\
= &\sum_{y' \in \mathcal{N}(x)} \frac{exp(\langle f_q(x,f^{in}_2), f_k(x,y',f^{in}_1)\rangle)}{\sum_{y' \in \mathcal{N}(x)}exp(\langle f_q(x,f^{in}_2) f_k(x,y',f^{in}_1)\rangle} \nonumber \\
&f_v(x,y',f^{in}_1)), 
\label{transformer}
\end{align}
where the subscript denotes the points to which the feature is attached, i.e., $x$ and $y'$. 

The features $f_k$, $f_v$ are constructed by the equivariant kernels $\kappa_k =\oplus_{j,i}\kappa_k^{l_{k_j},l_{in_i}}$, $\kappa_v=\oplus_{j,i}\kappa_v^{l_{v_j},l_{in_i}}$:  
\begin{align*}
&f^{l_{k_j}}_k(x,y,f_1^{in})=\sum_{i}\kappa_k^{l_{k_j},l_{in_i}}(y'-x)f_1^{l_{in_i}}(y);\\
&f^{l_{k_j}}_v(x,f_2^{in})=\sum_{i}\kappa_v^{l_{v_j},l_{in_i}}(y'-x)f_2^{l_{in_i}}(y),
\end{align*}
where for any $i,j$, any $h_2 \in SO(3)$, and any $x \in \mathbb{R}^3$
$\kappa_k^{l_{k_j},l_{in_i}}$ and 
$\kappa_v^{l_{v_j},l_{in_i}}$ should satisfy that:

$$\kappa_k^{l_{k_j},l_{in_i}}(h_2x)=\rho_2^{l_{k_j}}(h_2)\kappa_k^{l_{k_j},l_{in_i}}(x)\rho_2^{l_{in_i}}(h_2^{-1});$$
$$\kappa_v^{l_{v_j},l_{in_i}}(h_2x)=\rho_2^{l_{v_j}}(h_2)\kappa_v^{l_{v_j},l_{in_i}}(x)\rho_1^{l_{in_i}}(h_2^{-1})$$
as stated in \cite{fuchs2020se}.

The feature $f_q$ is constructed in the same way as  Equation \ref{construction_of_fq}.

Figure \ref{fig:transformer} shows the structures of ray feature update and $SE(3)$ equivariant transformer.

In figure \ref{fig:comparison}, we compare the $SE(3)$ equivariant transformer and the conventional transformer to illustrate how the equivariance is guaranteed in the equivariant transformer. In figure \ref{fig:comparison_multihead}, we present the types of futures in $SE(3)$ equivariant attention head and conventional attention head, respectively. It indicates that geometric information is aggregated equivariantly in multi-head attention in the equivariant transformer.

\begin{figure}[t]
  \centering
   \includegraphics[width=1.0\linewidth]{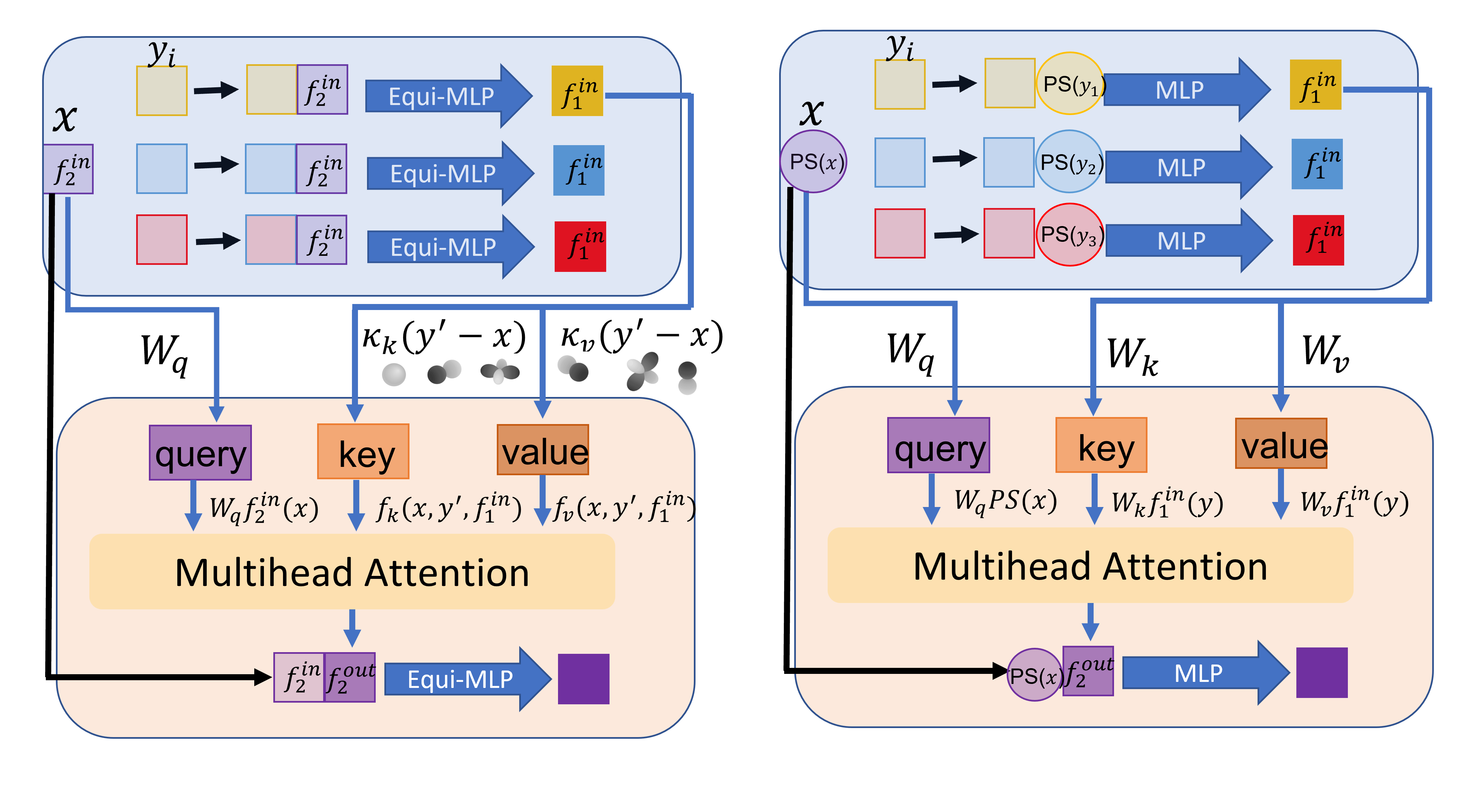}
   \caption{The comparison of the equivariant light field transformer and the conventional transformer. The left is the equivariant light field transformer, and the right is the conventional transformer. In our light field transformer, the position encoding is not directly concatenated to the features because this is not equivariant. We first obtain the equivariant feature attached to the point by equivariant convolution over the rays. We then construct features $f_k$, and $f_v$ with derived designed kernels $\kappa_k$ and $\kappa_v$ to keep them equivariant; we construct $f_q$ by the designed equivariant linear layer $W_q$. Since $f_k$, $f_q$, and $f_v$ are all equivariant, the inner product of $f_k$ and $f_q$ is invariant, which results in invariant attention weight. Therefore, the whole transformer is equivariant. 
   In contrast, the conventional transformer concatenates the ray position encoding with the feature attached to the ray, uses the point position encoding for the query feature for the point, and applies multi-head attention using $f_k$, $f_q$, and $f_v$, which are obtained by the Linear layer. We should note that $W_q$ in the light field transformer is designed to be equivariant, satisfying equation \ref{Q construction},  which differs from the conventional linear map $W_q$ in the conventional transformer. For the attention blocks after the first block, the query features of the point in our model and the conventional model are both the output of the last attention block. The difference is that our query feature keeps equivariant while the feature in the conventional transformer is not.}
   \label{fig:comparison}
\end{figure}

\begin{figure}[t]
  \centering
   \includegraphics[width=0.8\linewidth]{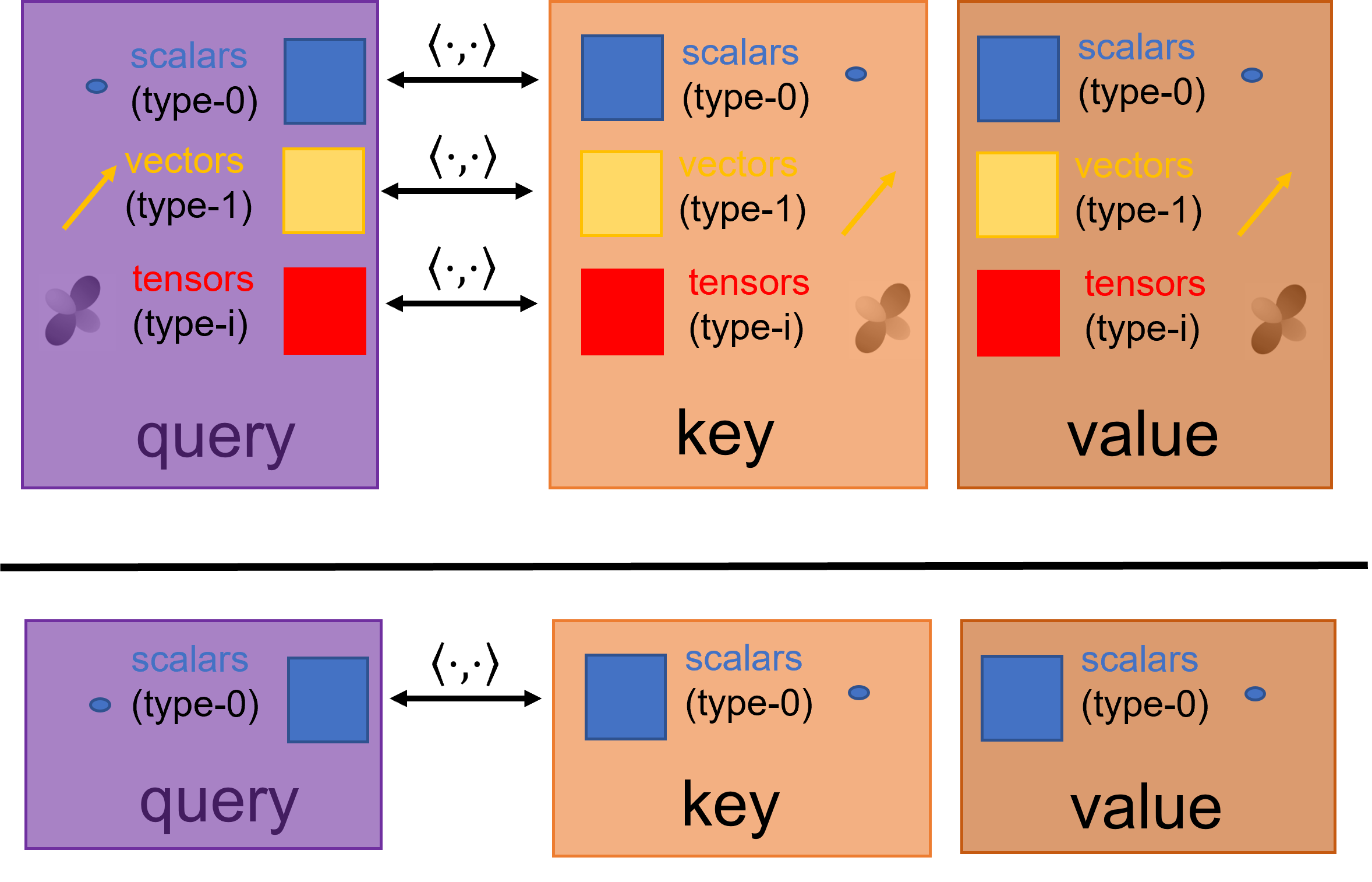}
   \caption{The comparison of multi-head attention modules in the equivariant light field transformer and that in the conventional transformer. The figure above is the multi-head attention module in an equivariant light transformer, and the figure below is the conventional transformer. In the light field transformer, the query, key, and value features are composed of different types of features; they can be scalars, vectors, or higher-order tensors. The inner product should apply to the same type of features, and the type of feature determines the way of applying the inner product. In contrast, the feature in a conventional transformer doesn't contain vectors and tensors, and the inner product is conventional.}
   \label{fig:comparison_multihead}
\end{figure}
\section{Equivariant Neural Rendering}
\label{equi_render}
Equivariant rendering relates to equivariant $3D$ reconstruction, where we focus on multiple views instead of the entire light field. The equivariance property is maintained when the ray sampling is invariant up to a coordinate change.
\subsection{Convolution from Rays to Rays}
\label{rendering_ray2ray}
For neural rendering tasks, we query one ray and apply the convolution over the neighboring rays to obtain the feature attached to the target query ray. Similar to the reconstruction, we utilize a kernel with local support. However, there is a distinction in that for neural rendering, the kernel $\kappa$ is constrained to be nonzero only when $d(x,\eta)=0$, while there are no constraints on $\angle(\bm{d}_x,[0,0,1]^T)$. As a result, the neighboring rays exclusively encompass the rays on the epipolar line for the target ray in each source view, as depicted in Figure \ref{fig:mvconv_rendering}.

\begin{figure}[t]
  \centering
\includegraphics[width=0.9\linewidth]{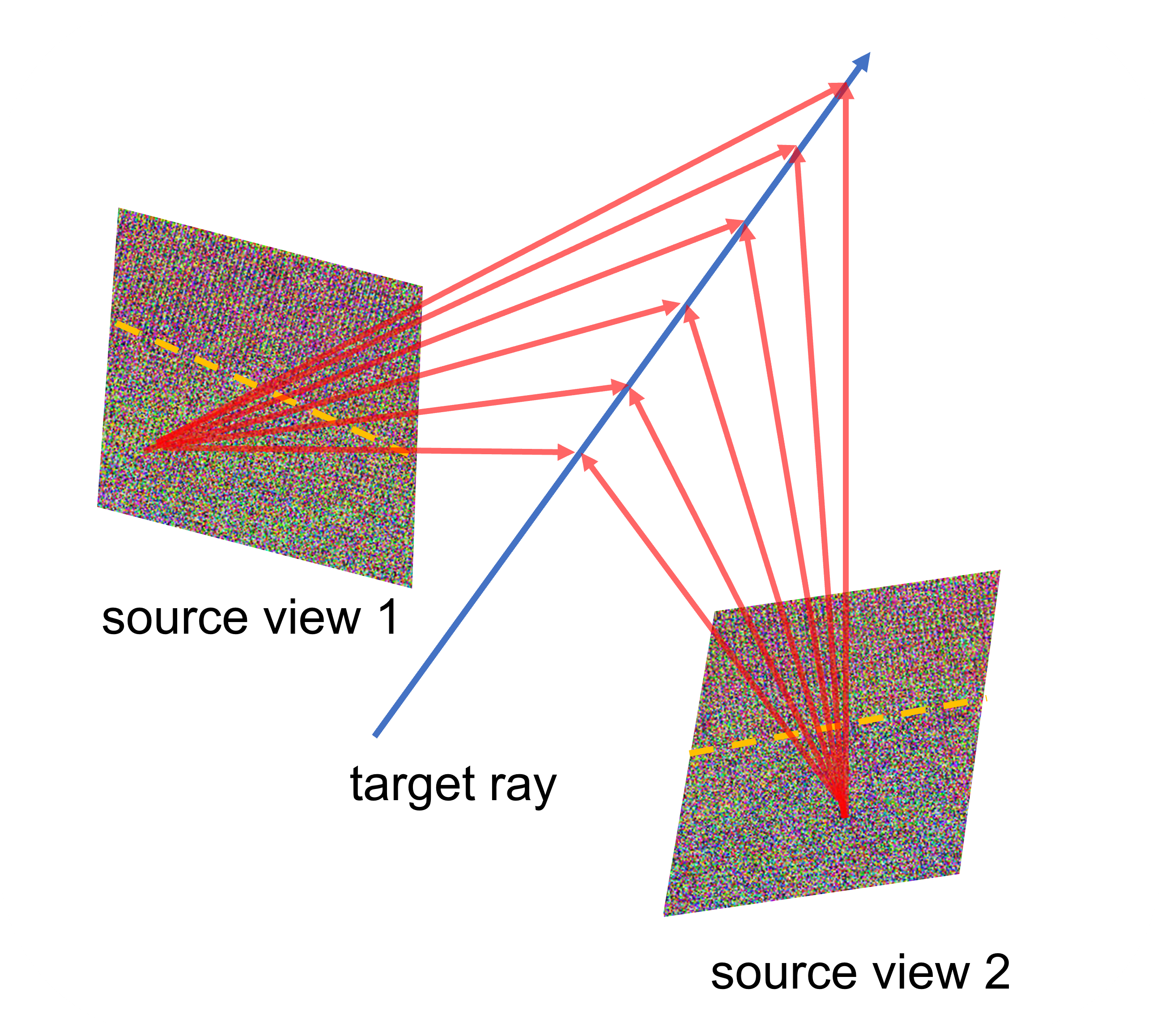}
   \caption{For simplification, we show two source views.
   For a target query ray $x$, the neighboring rays (denoted by red rays) are on the 
   epipolar lines (denoted as yellow dotted dashes) for the target ray in each source view. For any ray $y \in \mathcal{N}(x)$, $d(x,y)=0$.}
   \label{fig:mvconv_rendering}
\end{figure}

The scalar field over rays serves as the input to the convolution.  The output field type corresponds to the regular representation of translation. This is because this field type serves as the input for the cross-attention module later on. If this field type were not utilized, the transformer would reach the entire neighboring set, leading to inferior performance compared to applying the transformer individually for each point and then applying it over the points along the ray. A similar observation is made in \cite{varma2022attention}, which states that the two-stage transformer outperforms the one-stage transformer. Using the field type corresponding to the regular representation of the translation as the input, the transformer from rays to rays is equivalent to performing a transformer for each point, respectively, as explained in the following section.

In Eq. \ref{solution_for_re}, we already provide the solution of the kernel. We give a detailed explanation in this case and show that it is equivalent to performing convolution from rays to rays with  output field types corresponding to irreducible representations, followed by applying Inverse Fourier Transform. Given that the input field is a scalar field, we have $\omega^1_{in}=0$ and $\omega^2_{in}=0$. When considering an output field type of $(\omega^1_{out}, reg)$, where $reg$ represents the regular representation of translation, the convolution can be expressed as follows:
\begin{align}
(f_{out}^{(\omega^1_{out}, reg)})_t&=\int_{y \in \mathcal{N}(x)} \kappa_1(s(x)^{-1}y)(\kappa_2(s(x)^{-1}y))_t f_{in}(y)dy \nonumber \\ 
&=\int_{y \in \mathcal{N}(x)} \kappa_1(s(x)^{-1}y)f(d(\eta, s(x)^{-1}y), \angle([0,0,1]^T,\bm{d}_{s(x)^{-1}y}))\delta(t-g(s(x)^{-1}y))f_{in}(y)dy \nonumber \\ 
&=\int_{g(s(x)^{-1}y)=t}\kappa_1(s(x)^{-1}y)f(d(\eta, s(x)^{-1}y), \angle([0,0,1]^T,\bm{d}_{s(x)^{-1}y}))f_{in}(y)dy. \nonumber
\end{align}

From the above equation, we can intuitively find that when the output field corresponds to the regular representation of the translation, the convolution happens at every point along the ray, respectively.  We can treat $f_{out}^{(\omega^1_{out}, reg)}$ as a function over $\mathbb{R}$, and for any $\omega \in \mathbb{R}$ we apply the Fourier Transform to $f_{out}^{(\omega^1_{out}, reg)}$:

\begin{align}
\mathcal{F}(\omega)&= \int_t f_{out}^{(\omega^1_{out}, reg)}(t) e^{-i\omega t} dt \nonumber \\ 
&=\int_t \int_{g(s(x)^{-1}y)=t}\kappa_1(s(x)^{-1}y)f(d(\eta, s(x)^{-1}y), \angle([0,0,1]^T,\bm{d}_{s(x)^{-1}y}))f_{in}(y)dy e^{-i\omega t} dt \nonumber \\ 
& =\int_y\kappa_1(s(x)^{-1}y)f(d(\eta, s(x)^{-1}y), \angle([0,0,1]^T,\bm{d}_{s(x)^{-1}y}))e^{-i\omega g(s(x)^{-1}y)}f_{in}(y)dy \nonumber \\ 
& = \int_y \kappa_1(s(x)^{-1}y) \kappa'_2(s(x)^{-1}y)f_{in}(y)dy,\nonumber
\end{align}

where $\kappa'_2 =f(d(\eta, s(x)^{-1}y), \angle([0,0,1]^T,\bm{d}_{s(x)^{-1}y}))e^{-i\omega g(s(x)^{-1}y)}$, which is exactly the kernel corresponding to $\omega_{out}^2 =\omega$ and $\omega_{in}^2 =0$ as stated in Eq. \ref{solution_for_nonre}. Therefore, we know that the field corresponding to the irreducible representation of the translation can be treated as the Fourier coefficients of the field corresponding to the regular representation. We can first obtain the features of different irreducible representations attached the ray and subsequently apply the Inverse Fourier Transform to get the features for points along the ray,%
as shown in figure \ref{fig:inverse_fourier}. 
\begin{figure}[t]
  \centering
\includegraphics[width=0.9\linewidth]{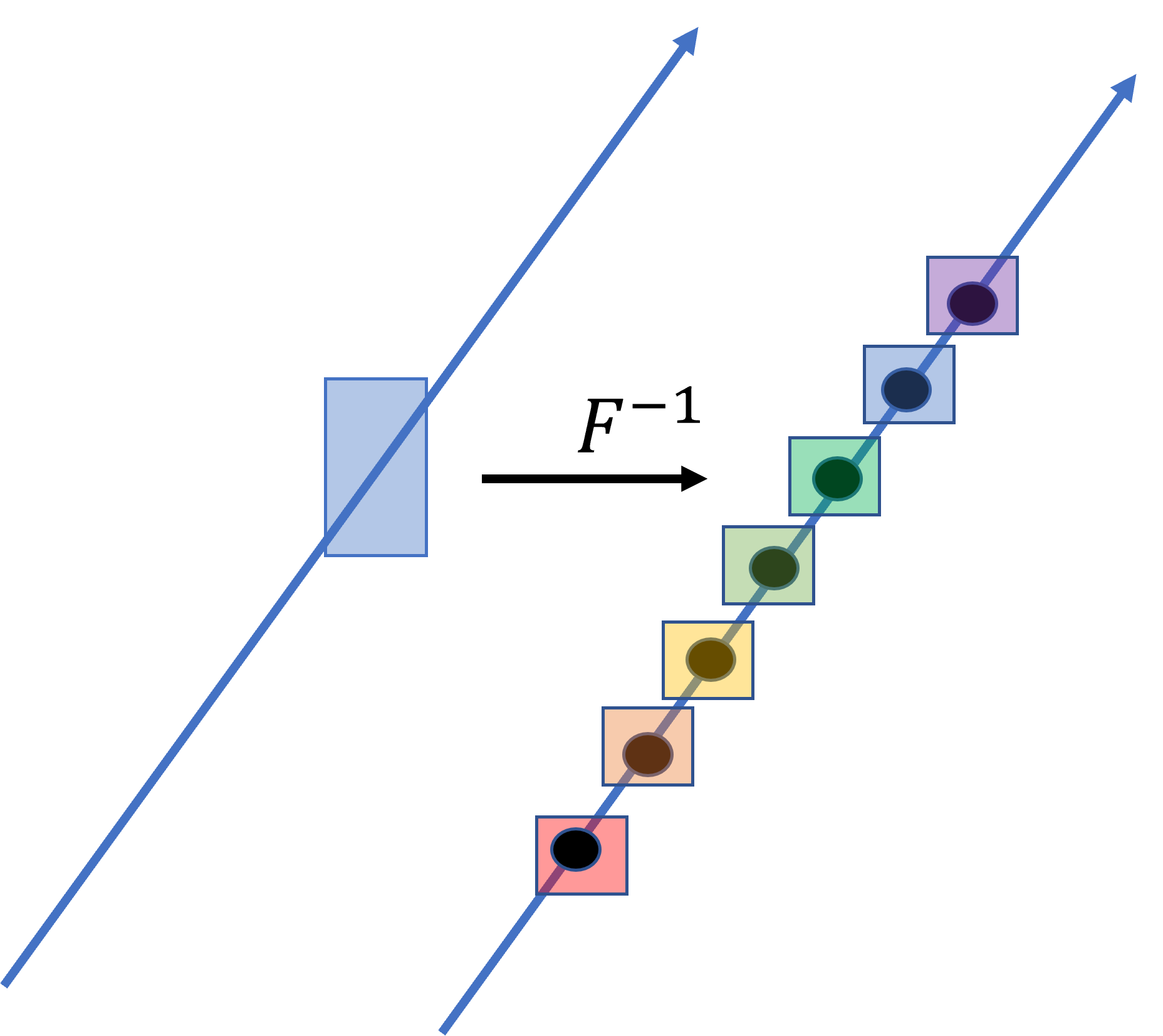}
   \caption{The features for points along the ray (the field type corresponds to the regular representation) can be obtained by the Inverse Fourier Transform of features attached to the ray, where the types of feature fields correspond to the irreducible representation of the translation.}
   \label{fig:inverse_fourier}
\end{figure}

\subsection{Cross-attention over Rays}
\label{attention_regular}
The feature that generates the query in the transformer is the feature attached to the target ray, whose feature type corresponds to the regular representation of the translation. The feature that generates the key and value in the transformer is attached to the neighboring rays in the source view, whose feature type corresponds to the scalar field. The output is the feature attached to the target ray, whose feature type corresponds to the regular representation. Therefore, the transformer becomes: 
\begin{align}
(f^{out}_2(x))_t = &\sum_{y \in \mathcal{N}(x)} \frac{exp(\langle (f_q(x,f^{in}_2))_t, (f_k(x,y,f^{in}_1))_t\rangle)}{\sum_{y \in \mathcal{N}(x)}exp(\langle (f_q(x,f^{in}_2))_t (f_k(x,y,f^{in}_1))_t\rangle}) (f_v(x,y,f^{in}_1))_t, 
\end{align}
where 
\begin{align*}
    &(f_k(x,y,f_1^{in}))_t  = (\kappa_k(s_2(x)^{-1}y))_tf_1^{in}(y)\\
    &(f_v(x,y, f_1^{in}))_t = (\kappa_v(s_2(x)^{-1}y))_tf_1^{in}(y)\\
    &(f_q (x,f_2^{in}))_t= C (f_2^{in}(x))_t.
\end{align*}
In the aforementioned equations,  $\kappa_k$ and $\kappa_v$ are the kernels derived in  Ex. \ref{filter_solution} Eq. \ref{whole_so_re}, $C$ is the equivariant weight matrix satisfying Eq. \ref{Q construction}. 

The expression above indicates that the feature types of both key and value correspond to the regular representation of translation, as well as the feature type of the query. Moreover, the transformer operates on each point along the ray independently. It should be noted that the features $(f_k)_t$, $(f_q)_t$, $(f_v)_t$ and $(f^{in}_2)_t$ may have multiple channels and may consist of different types of features corresponding to various representations of $SO(2)$. The inner product $\langle \cdot,\cdot \rangle$ can only happen in the field type of the same representation of $SO(2)$. This allows for the implementation of a multi-head attention module, where each head can attend to a specific type of feature and multiple channels.

\subsection{Self-attention over Points Along the Ray}
After the cross-attention over rays, we get the features of the points along the ray, i.e., the feature attached to the ray corresponding to the regular representation of translation. $SE(3)$ acts on the feature $f'$ attached to the point along the ray as mentioned in Eq.\ref{regular_action} :
\begin{align*}
(\mathcal{L}_gf')(\bm{x},\bm{d}) = \rho_1(R_Z(R_{g^{-1}},\bm{d}))^{-1}
f'(g^{-1}\bm{x},R_{g^{-1}}\bm{d}),
\end{align*}
where $\rho_1$ is the group representation of $SO(2)$.

We will apply the self-attention model to these points along the same ray. For two points $\bm{x}_1$ and $\bm{x}_2$ on the same ray $(\bm{d},\bm{x}_1\times \bm{d})$, one can observe that for the same type of feature, $\langle (\mathcal{L}_gf')(\bm{x_1},\bm{d}) , (\mathcal{L}_gf')(\bm{x_1},\bm{d}) \rangle = \langle f'(g^{-1}\bm{x_1},R_{g^{-1}}\bm{d}) , f'(g^{-1}\bm{x_1},R_{g^{-1}}\bm{d}) \rangle $, which makes attention weight invariant, the transformer could be formulated as: 

\begin{align}
f^{out}(x) = &\sum_{\text{y on the same ray as x}} \frac{exp(\langle f_q(f^{in},x), f_k(f^{in},x,y)\rangle)}{\sum_{\text{y on the same ray as x}} exp(\langle f_q(f^{in},x), f_k(f^{in},x,y)\rangle)} f_v(x,y,f^{in}), 
\end{align}
where 
\begin{align*}
    &f_k^{l}(x,y,f^{in})  = c_k(d(x,y))I (f^{in})^l(y)\\
    &f_v^{l}(x,y,f^{in}) = c_v(d(x,y))I(f^{in})^l(y)\\
    &f_q^{l} (x,f^{in})= c_qI(f^{in})^l(x),
\end{align*}

where $x$ and $y$ are the points along the same ray with direction $\bm{d}$, we can denote $x$ as $(\bm{x},\bm{d})$ and $y$ as $(\bm{y},\bm{d})$, $d(x,y)$ is the signed distance $\langle \bm{d}, \bm{y}-\bm{x}\rangle$, $c_k$,$c_v$ are arbitrary functions that take signed distance as the input and output complex values and $c_q$ is an arbitrary constant complex. It should be noted that the features $f_k$, $f_q$, $f_v$, and $f^{in}$ may have multiple channels and consist of different types of features corresponding to various representations of $SO(2)$, the inner product $\langle \cdot,\cdot \rangle$ can only happen in the same type of field. This allows for implementing a multi-head attention module, where each head can attend to a specific type of feature and multiple channels. Here, $f_k^l$ denotes the type$-l$ feature in feature $f_k$, $f_v^l$ represents the type$-l$ feature in feature $f_v$,
$f_q^l$ denotes the type$-l$ feature in feature $f_l$, and
$(f^{in})^l$ represents the type$-l$ feature in feature $f^{in}$.

 Note that this transformer architecture also follows the general format of the transformer in Eq. \ref{paper:transformer} 
 . We only simplify the kernel $\kappa_k$, $\kappa_v$ to be trivial equivariant kernels.

 To obtain a scalar feature density for each point, the feature output of each point can be fed through an equivariant MLP, which includes equivariant linear layers and gated/norm nonlinear layers. These layers are similar to the ones used in \cite{weiler2019general} and \cite{weiler20183d}.
\section{$3D$ Reconstruction Experiment}
\label{reconstruction_experiment}
\subsection{Generation of the Dataset}
\label{data_generation}
The I dataset is obtained by fixing the orientation of the object as well as the eight camera orientations. With the object orientation fixed, we can independently rotate each camera around its optical axis by a random angle in a uniform distribution of $(-\pi,\pi]$ to obtain the Z dataset. For the R dataset, we rotate every camera randomly by any rotation in $SO(3)$ while fixing the object. %
The equivariance stands with the content unchanged. Therefore in practice, we require that the object projection  after the rotation does not have new parts of the object. We satisfy this assumption by forcing the camera to fixate on a new random point inside a small neighborhood and subsequently rotate each  camera around its optical axis with the uniformly random angle in $(-\pi,\pi]$.
We generate the $Y$ dataset by rotating the object only with azimuthal rotations while keeping  the camera orientations the same. 
The $SO(3)$ dataset is generated by rotating the object with random rotation in $SO(3)$ with the orientations of cameras unchanged, which will potentially result in 
new image content. Equivariance is not theoretically guaranteed in this setup, but we still want to test the performance of our method.
\subsection{Implementation Details}
\label{implementation_reconstruction}
We use $SE(2)$ equivariant CNNs to approximate the equivariant convolution over the rays. We use the same ResNet backbone as implemented in \cite{han2021redet} that is equivariant to the finite group $C_8$, which we find achieves the best result compared with other $SE(2)$ equivariant CNNs. We use a similar pyramid structure as \cite{xu2019disn} that concatenates the output feature of every block. Since every hidden feature is the regular representation, in the final layer we use $1\times1$ $SE(2)$-equivariant convolutional layers to transfer the hidden representation to scalar type.

For the fusion from the ray space to the point space model, we use one layer of convolution and three combined blocks of updating ray features and $SE(3)$ transformers. For the  equivariant  $SE(3)$ multi-head-attention, we only use the scalar feature and the vector (type-1) feature in the hidden layer.  The kernel matrix includes the spherical harmonics of degrees 0 and 1. We also concatenate every output point feature of every block as in the $2D$ backbone. Since the output feature of every block includes the vector feature, we transfer it to the scalar feature through one vector neuron layer and the inner vector product.  We use the same weighted SDF loss as in \cite{xu2019disn} during training, which applies both uniform and near-surface sampling.  We report the number of parameters and floating-point operations (FLOPs)  of our $2D$ backbone and light fusion networks in Fig. \ref{fig:se2_flops} and Fig. \ref{fig:fusion_flops} respectively.

\begin{figure}[t]
  \centering
   \includegraphics[width=0.9\linewidth]{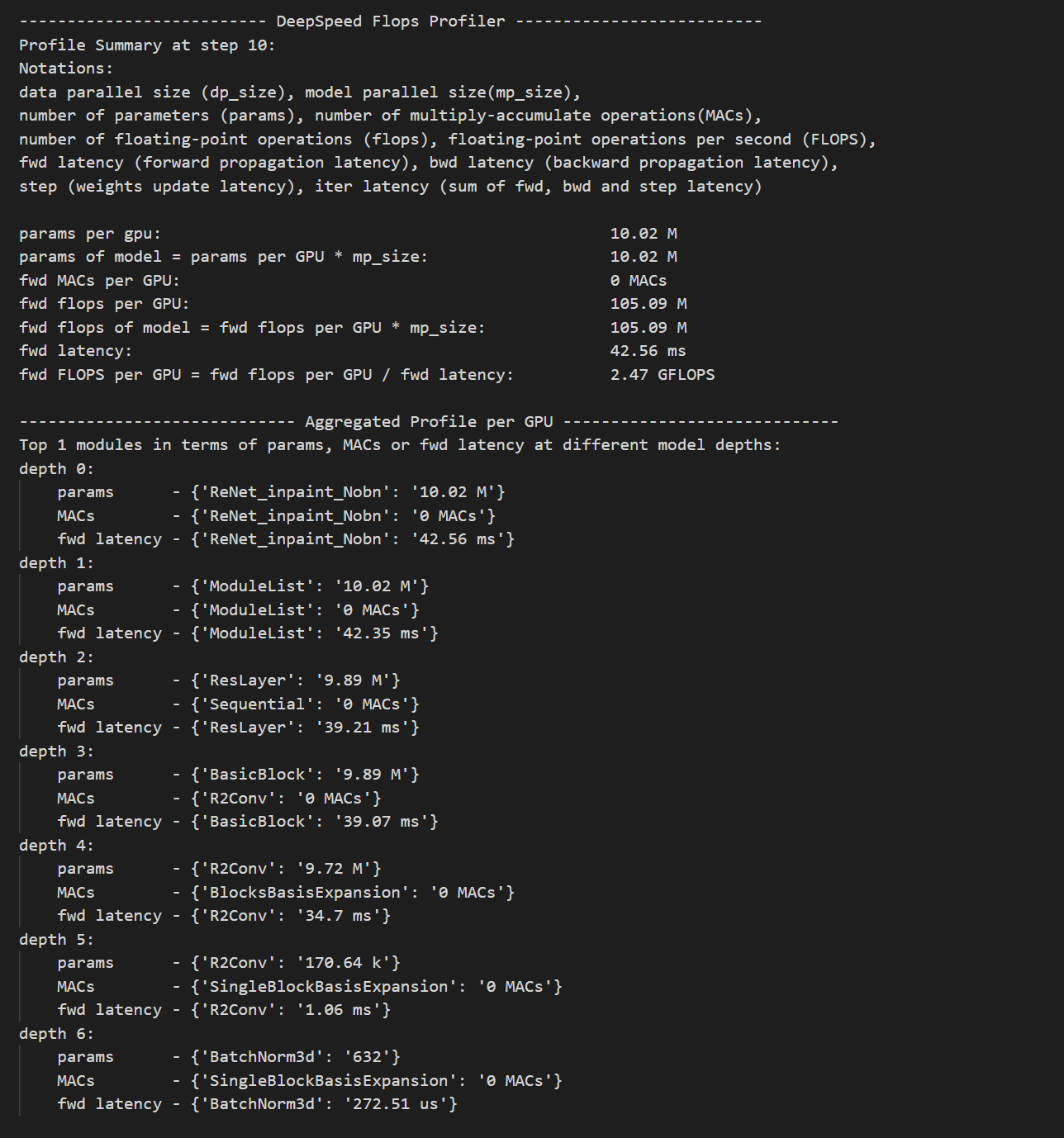}
   \caption{The number of parameters and FLOPs of $SE(2)$ equivariant CNNs. We set batch size as one to calculate number of FLOPs.}
   \label{fig:se2_flops}
\end{figure}

\begin{figure}[t]
  \centering
   \includegraphics[width=0.9\linewidth]{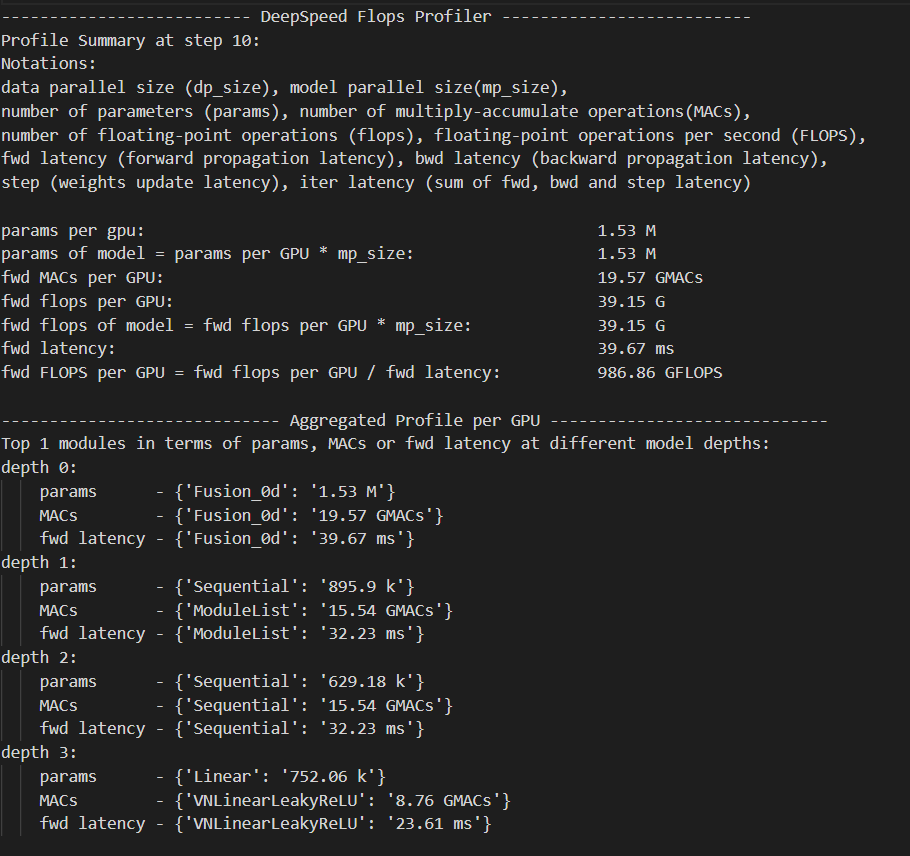}
   \caption{The number of parameters and FLOPs of the ray fusion model, which is composed of convolution from rays to points and transformer from rays to points.We set batch size as one to calculate the number of FLOPs.}
   \label{fig:fusion_flops}
\end{figure}

\subsection{Discussion of Results}
\label{discussion_result}
There is still a performance gap between $I/I$ and $I/Z$. This is because although $SE(2)$ equivariant networks are theoretically strictly equivariant, the error in practice is introduced by the finite sampling of the image and the pooling layers. Additionally, 
we use the ResNet that is equivariant to $C_8$ approximation of $SO(2)$, which  causes this gap but increases the whole pipeline performance in the other tasks. There is not a significant difference between $I/Z$ and $I/R$, which shows that approximating the spherical field convolution by $SE(2)$ equivariant convolution is reasonable in practice. 
\subsection{Qualitative Results}
\label{qual_results}
Figure \ref{fig:qualitive_result} shows a qualitative result for the chair category. There are more qualitative results shown in Fig. \ref{fig:qual_2}, Fig. \ref{fig:qual_3}, and Fig. \ref{fig:qual_4}.
\begin{figure}[t]
  \centering
   \includegraphics[width=0.86\linewidth]{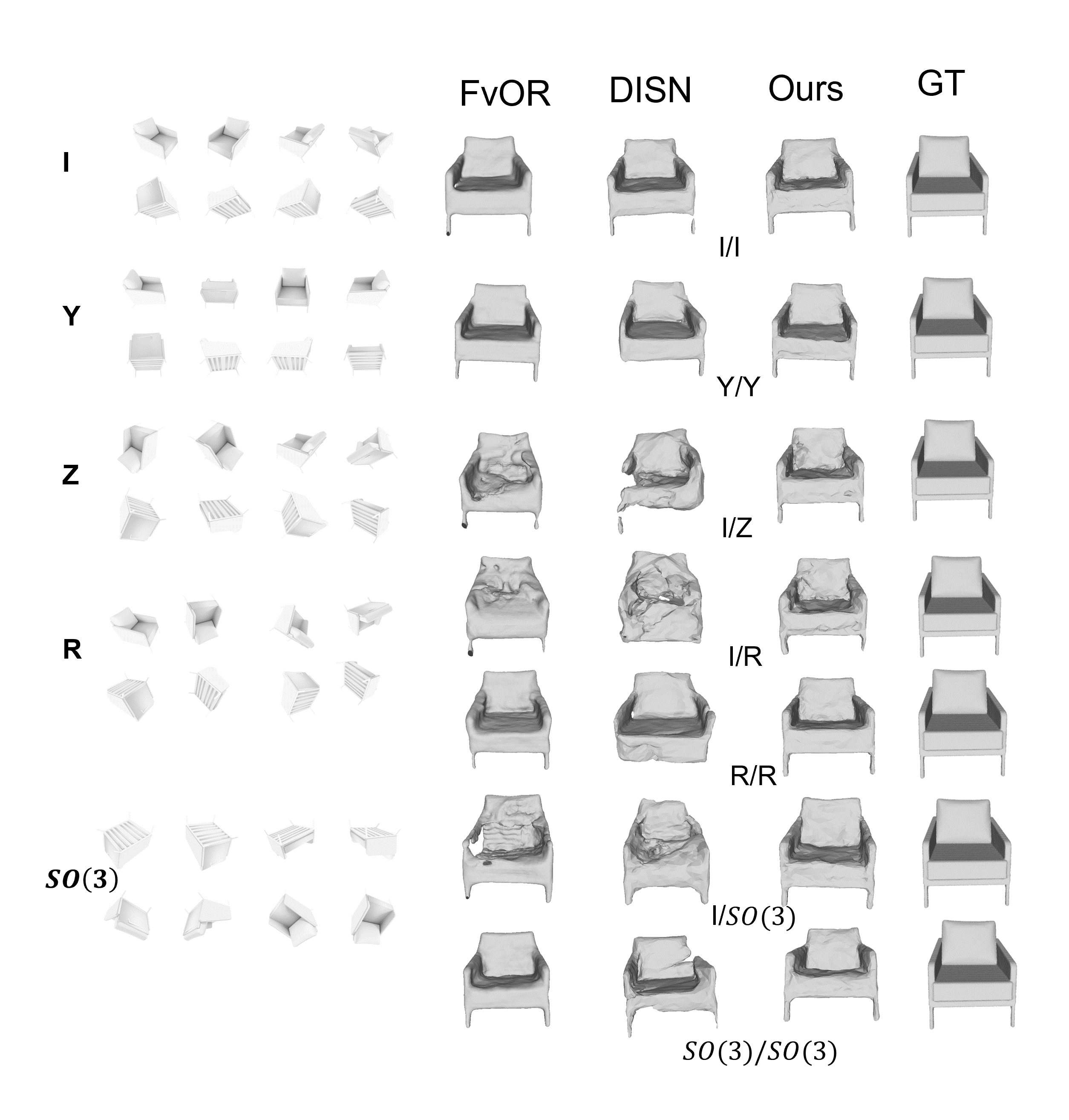}
   \caption{Qualitative results for equivariant reconstruction. Left: input views; Right: reconstruction meshes of different models and ground truth meshes. The captions below the meshes show how the model is trained and tested, explained in the text.}
   \label{fig:qualitive_result}
\end{figure}

\subsection{Ablation Study}
\label{ablation_study}
First, we replace the $SE(2)$ CNNs backbone with the conventional CNNs to test the effectiveness of $SE(2)$ CNNs. 
Secondly, we remove the equivariant convolution/transformer part and use trivial aggregation (max-pooling) combined with MLP. 
Finally, we run an equivariant convolution and transformer without using the type-1 (vector) feature while keeping the number of parameters similar to our model. 
\begin{table}
\scalebox{1.0}{
\begin{tabular}{|c|c|c|c|c|}
\hline
    Method & w/o SE(2) & w/o conv\& trans & w/o type-1&
    Full model\\
    \hline
    I/I&\textbf{0.767}/\textbf{0.079} &0.695/0.105 & 0.722/0.093& 0.731/0.090\\
    I/Z&0.430/0.234 &0.533/0.175 &  0.553/0.158&\textbf{0.631}/\textbf{0.130}\\
    I/R& 0.417/0.249&0.442/0.241&0.466/0.203 &\textbf{0.592}/\textbf{0.137}\\
    R/R& 0.672/0.112&0.658/0.122 &0.682/0.109 &\textbf{0.689}/\textbf{0.105}\\
    Y/Y&\textbf{0.731}/\textbf{0.090} &0.644/0.124 & 0.677/0.111&0.698/0.102\\
    Y/SO(3)&0.467/0.0.217 &0.534/0.170 & 0.569/0.163&\textbf{0.589}/\textbf{0.142}\\
    SO(3)/SO(3)&0.655/0.120 &0.616/0.142 & 0.636/0.130&\textbf{0.674}/\textbf{0.113}\\    
\bottomrule
\end{tabular}
}
\caption{Ablation: w/o $SE(2)$ means replacing $SE(2)$ equivariant network with conventional; w/o ray conv\& trans denotes the model where we replace the light field convolution and the light field equivariant transformer with  max-pooling; w/o type-1 means using only scalar features in convolution and transformers. }
  \label{tab:ablation_study}
\end{table}

Table \ref{tab:ablation_study} summarizes the result on the chair category, which illustrates that in the $I/I$ and $Y/Y$ trials, $SE(2)$ CNN is less expressive than traditional CNN, but it contributes to the equivariance of our model looking at the results of $I/Z$, $I/R$, and $Y/SO(3)$.
Equivariant ray convolution and transformer improve both the reconstruction performance and the equivariance outcome. We also compare the ray convolution and transformer with the models operating only on scalar features without vector features, and again we see a drop in performance in every setting, proving the value of taking ray directions into account.

We also compare to a baseline where the ray difference information is encoded in the feature explicitly. Most models that encode ray directions aim at rendering, like IBRnet. Here we modified IBRnet (Fig.2 of IBRnet paper) to query 3D points only for their SDF value instead of querying all densities along the ray that would be necessary for rendering. We replaced the ray direction differences with the ray directions themselves because we use a query point and not a query ray. We report in table  \ref{tab:ablation_study_2} IoU result for Y/Y and Y/SO(3) (where Y is augmentation only along the vertical axis) for two models -- IBRNet with conventional CNNs as 2D backbone and IBRNet with SE(2)-equivariant CNNs as 2D backbone. For the $SO(3)$ setting, we rotate the whole $8$ cameras with the same rotation, which is equivalent to rotating the object with the inverse rotation, and we use the object canonical frame to encode the ray information. 

\begin{table}
\scalebox{1.0}{
\begin{tabular}{|c|c|c|c|}
\hline
    Method & Y/Y & Y/SO(3)& SO(3)/SO(3) \\
    \hline
    IBRNet \cite{wang2021ibrnet} w/o SE(2)& 0.689& 0.432&0.611\\
    IBRNet \cite{wang2021ibrnet} w/SE(2)&0.652&0.501&0.619\\
    Ours &\textbf{0.698}&\textbf{0.598}&\textbf{0.674}\\    
\bottomrule
\end{tabular}
}
\caption{Comparison of our model and a baseline which encodes the ray information explicitly. IBRNet w/o SE(2) is the modified IBRNet with conventional CNN backbone, IBRNet w/SE(2) is the model where we replace the conventional CNN backbone with the SE(2) equivariant CNN.}
  \label{tab:ablation_study_2}
\end{table}

The baseline is not equivariant: It explicitly uses the ray directions as inputs to MLPs. Ray directions or their differences change when the coordinate system is transformed, breaking, thus, equivariance. Table \ref{tab:ablation_study_2} demonstrates that our model is more resilient to object rotations. We can enhance equivariance by using SE(2) equivariant modeling, and our model outperforms the baseline in the Y/Y setting. We believe that the transformer in our model is responsible for the performance improvement.

\begin{figure}[t]
  \centering
   \includegraphics[width=0.9\linewidth]{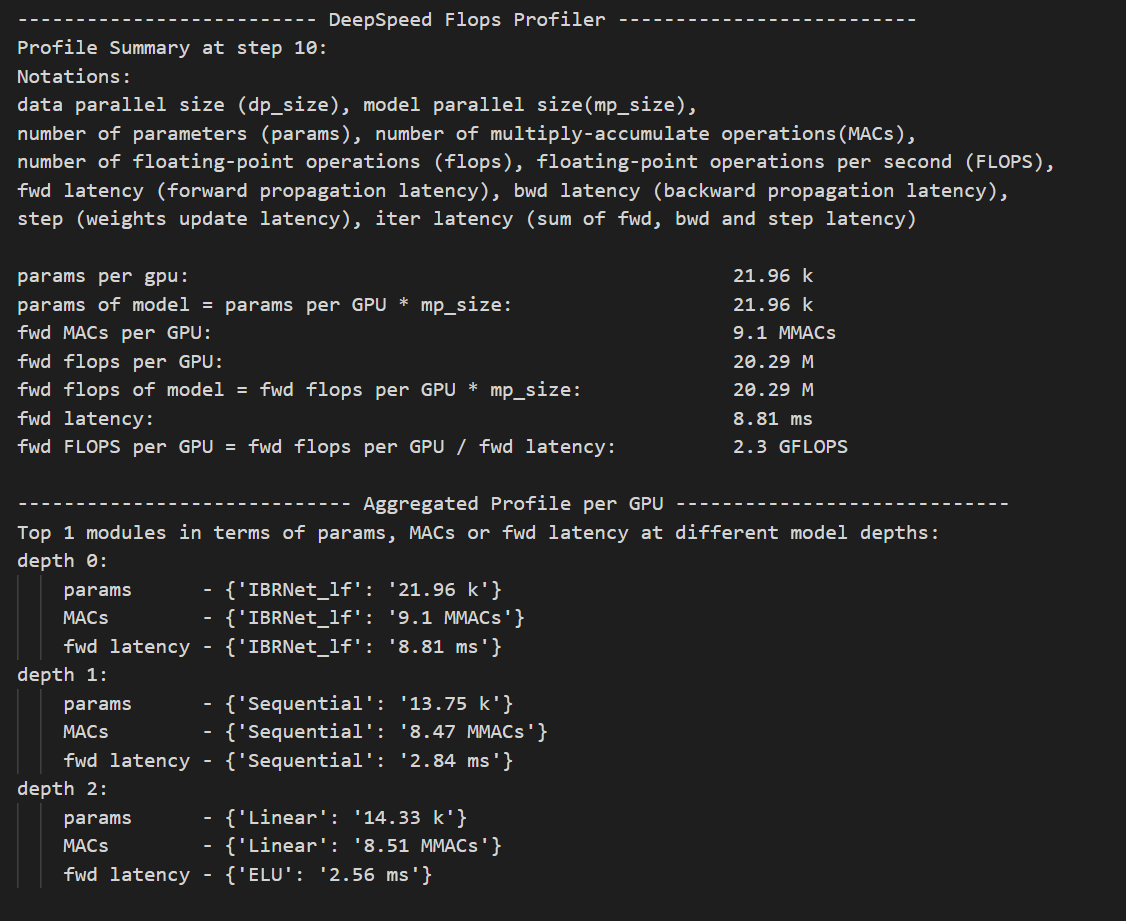}
   \caption{The number of parameters and FLOPs of the model, which takes the scalar feature attached to rays as input and predicts the color and density for points along the target ray. The calculation of FLOPs is performed for single-pixel rendering with $10$ source views.}
   \label{fig:ibr_lf_flops}
\end{figure}

\begin{figure}[t]
  \centering
   \includegraphics[width=0.88\linewidth]{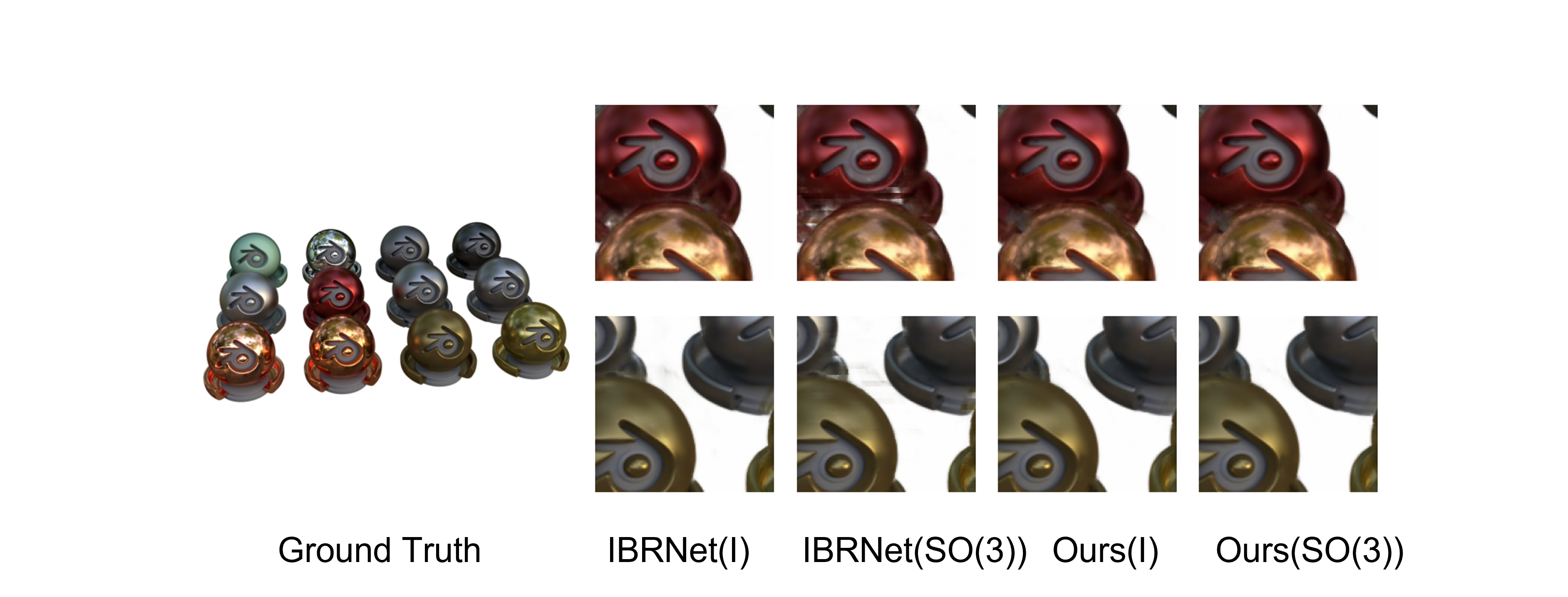}
   \caption{In terms of qualitative results for rendering, we compare the performance of IBRNet and our model in both the given canonical frame (denoted as "IBRNet(I)" and "Ours(I)" respectively) and a rotated frame (denoted as "IBRNet(SO(3))" and "Ours(SO(3))" respectively). Our model performs comparably to IBRNet in the canonical setting. However, IBRNet experiences a performance drop in the rotated frame, while our model remains robust to the rotation.}
   \label{fig:qual_rendering_1}
\end{figure}

\begin{figure}[t]
  \centering
   \includegraphics[width=0.88\linewidth]{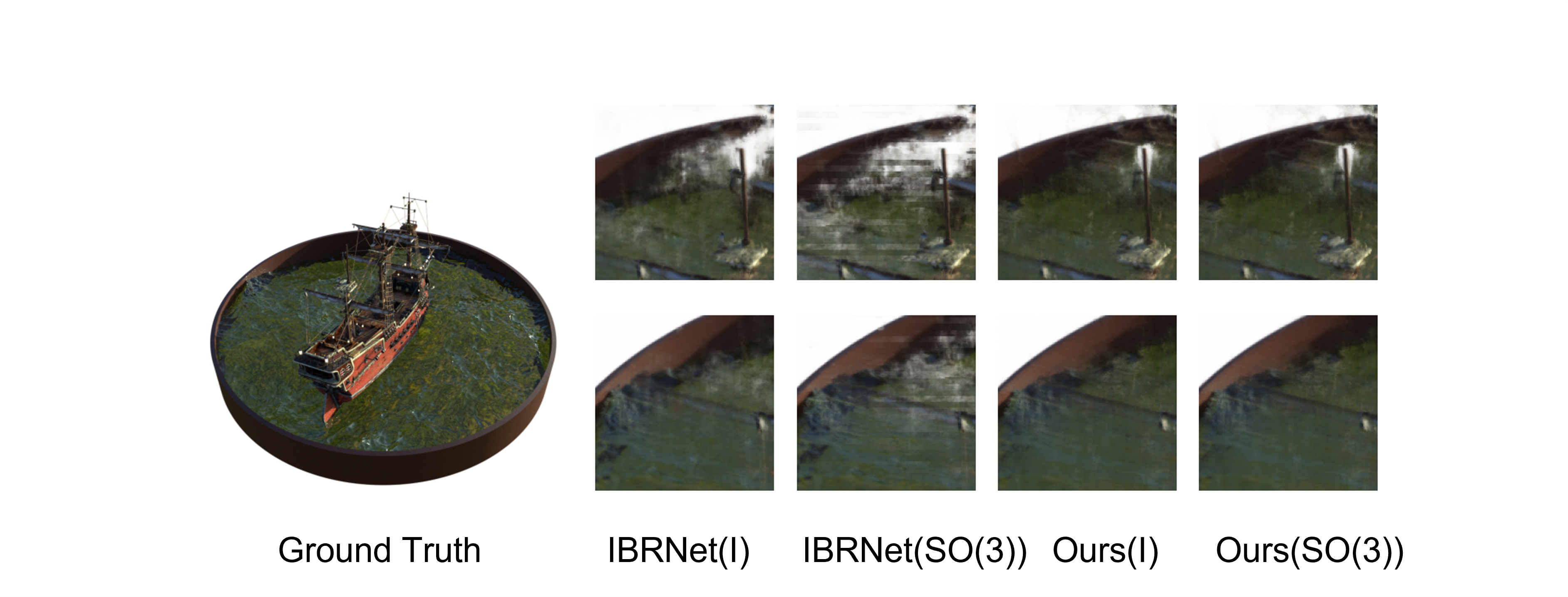}
   \caption{In terms of qualitative results for rendering, we compare the performance of IBRNet and our model in both the given canonical frame (denoted as "IBRNet(I)" and "Ours(I)" respectively) and a rotated frame (denoted as "IBRNet(SO(3))" and "Ours(SO(3))" respectively). Our model performs comparably to IBRNet in the canonical setting. However, IBRNet experiences a performance drop in the rotated frame, while our model remains robust to the rotation.}
   \label{fig:qual_rendering_2}
\end{figure}

\begin{figure}[t]
  \centering
   \includegraphics[width=0.88\linewidth]{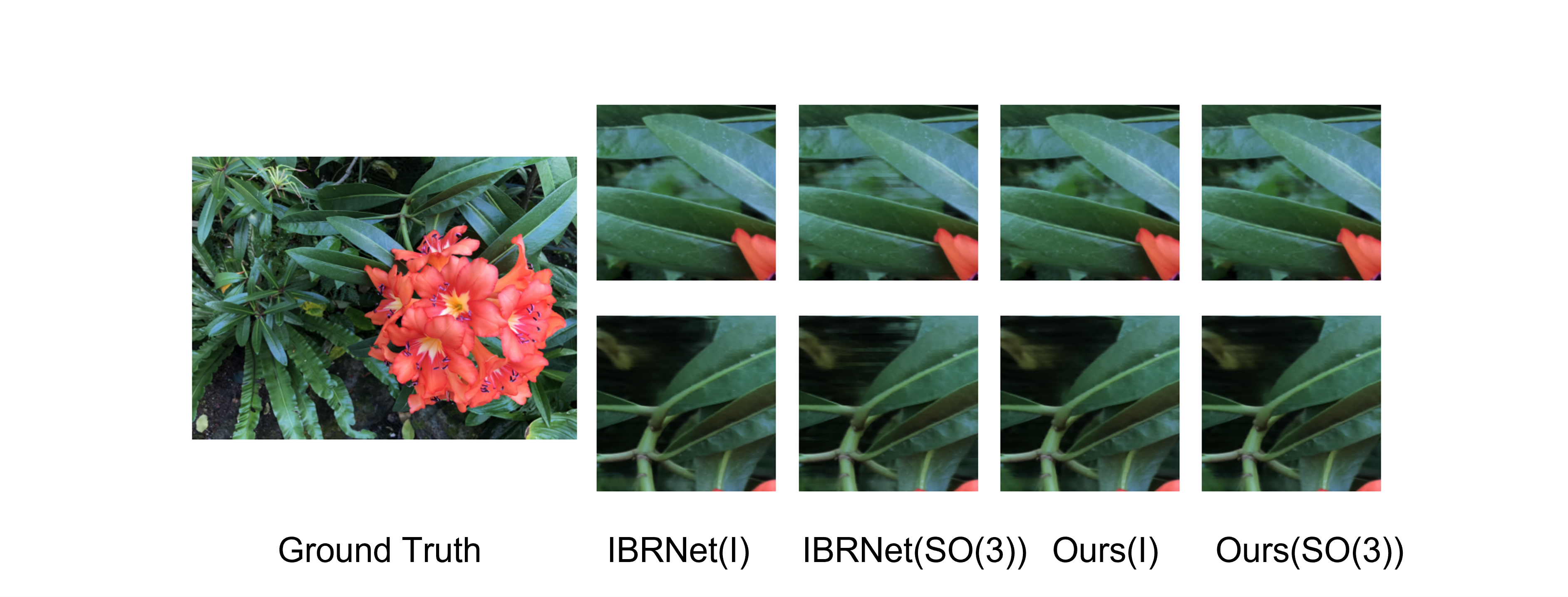}
   \caption{In terms of qualitative results for rendering, we compare the performance of IBRNet and our model in both the given canonical frame (denoted as "IBRNet(I)" and "Ours(I)" respectively) and a rotated frame (denoted as "IBRNet(SO(3))" and "Ours(SO(3))" respectively). Our model performs comparably to IBRNet in the canonical setting. However, IBRNet experiences a performance drop in the rotated frame, while our model remains robust to the rotation.}
   \label{fig:qual_rendering_3}
\end{figure}

\section{Neural Rendering Experiment}
\label{rendering_experiment}
\subsection{Experiment Settings Discussion}
\label{rendering_setting_discussion}
Two experiment settings illustrate our model's equivariance: $I/I$ and $I/SO(3)$. $I/I$ is the canonical setting, where we train and test the model in the same canonical frame defined in the dataset. $I/SO(3)$ is that we test the model trained in the canonical frame under arbitrary rotated coordinate frames, which means that all the camera poses in one scene are transformed by the same rotation without changing their relative camera poses and relative poses between the camera and the scene, which doesn't change the content of the multiple views.  The reason we don't apply translation to the cameras is that there exists a depth range for points sampling in the model and the comparing baseline \cite{wang2021ibrnet}, which effectively mitigates the impact of translation.

We should note that the $SO(3)$ setting in this experiment setting differs from $R$ and $SO(3)$ settings in reconstruction. $R$ changes the relative pose of the cameras, and each image is transformed due to the rotation of each camera without altering the content, i.e., the sampling of the light field is nearly unchanged. The $R$ setting aims to demonstrate that replacing the conventional method with ray-based convolution can get rid of the canonical frame for each view. 

$SO(3)$ in reconstruction is to rotate the object pose randomly without changing the pose of the camera, which is equivalent to transforming the cameras by the inverse rotation but fixing the object, resulting in changes in the relative poses between the camera and the object, the content of the image and, therefore, the sampling of the light field. This setting shows that even for non-theoretically equivariant cases, our model in reconstruction still demonstrates robustness.

In the rendering experiment using the $SO(3)$ setting, each image itself is not transformed, unlike the $R$ setting in the reconstruction. The content of the images remains unchanged, including the light field sampling, unlike the $SO(3)$ setting in the reconstruction.  Since each image is not transformed, even if the conventional $2D$ convolution is applied to the image, the scalar feature attached to the ray is not altered, and the light feature field sampling remains the same up to the transform of the coordinate frame. This setting was used to demonstrate that our model is $SE(3)$-equivariant when the input is the scalar light feature field.

\subsection{Implementation Details}
\label{implementation_rendering}
As described in the paper, we use a similar architecture as \cite{wang2021ibrnet}, where we replace the aggregation of view features by equivariant convolution and equivariant transformer over rays. In equivariant convolution, the input is scalar feature field over rays, which means that $\omega^1_{in}=0$ and $\omega^2_{in}=0$; for the output field, we use  regular representation %
of translation as described in Sec. \ref{neural_rendering} 
, and we use $\omega^1_{out}= 0,2^1,\cdots, 2^7$ for group representation of $SO(2)$, each field type has $4$ channels. In equivariant transformer over rays, we update the key and value before going to the attention module in the experiment; the specific operation is that we concatenate key $f_k$ and query $f_q$, we concatenate $f_v$ and query $f_q$,  and then we feed the concatenated key and value into two equivariant MLPs (equivariant linear layers and gated/norm nonlinear layers, similar to the ones used in \cite{weiler2019general}) to get the newly updated key and updated value, which will be fed into attention module. In line with \cite{wang2021ibrnet}, our approach does not involve generating features for the color of every point. In our implementation, we directly multiply the attention weights obtained from the softmax operator in the transformer with the corresponding colors in each view to perform color regression.

We replace the ray transformer with the equivariant transformer over the points along the ray; the input features comprise the feature types corresponding to the group representations $\omega_{in}= 0,2^1,\cdots, 2^7$ for $SO(2)$. Each feature type has $4$ channels; the output comprises the same  feature type, and each type has $2$ channels. We will first convert the feature into a scalar feature by an equivariant MLP (equivariant linear layers and gated/norm nonlinear layers, similar to the ones used in \cite{weiler2019general}.) and then feed it into a conventional MLP to get the density. We report in Fig. \ref{fig:ibr_lf_flops} the number of parameters and  floating-point operations (FLOPs)  of the model composed of the convolution and transformers.

\subsection{Qualitative Results}
\label{rendering_qualitive_result}
Fig. \ref{fig:qual_rendering_1}, Fig. \ref{fig:qual_rendering_2} and  Fig. \ref{fig:qual_rendering_3} show the qualitative results on Real-Forward-Facing \cite{mildenhall2019local} and Realistic Synthetic $360^{\circ}$ \cite{sitzmann2019deepvoxels} data.  Our model performs comparably to IBRNet in the canonical setting. However, IBRNet experiences a performance drop in the rotated frame, while our model remains robust to the rotation.

\begin{figure}[t]
  \centering
   \includegraphics[width=0.9\linewidth]{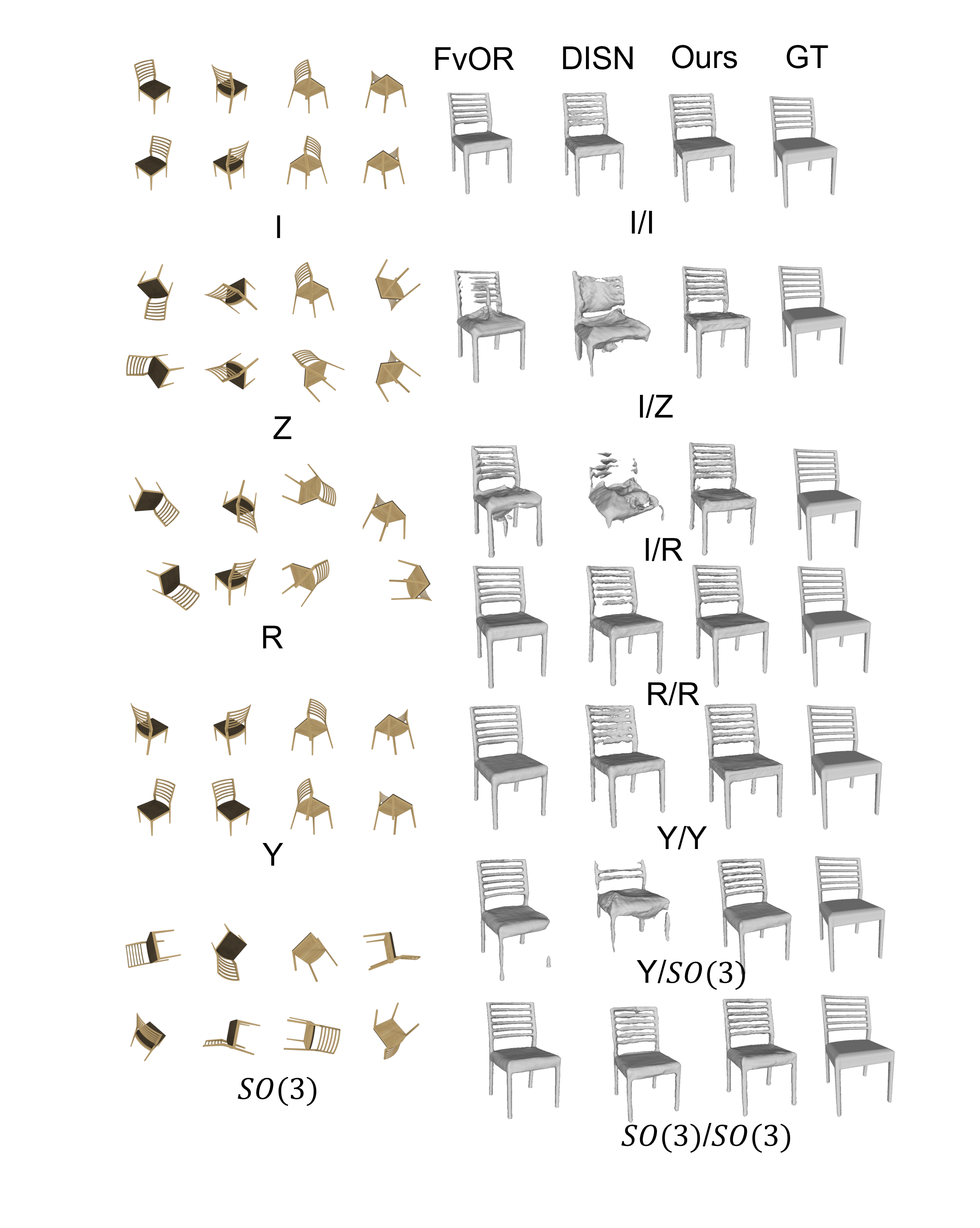}
   \caption{Qualitative Result for the chair. Left: input views; Right: reconstruction meshes of different models. The captions below the meshes show how the model is trained and tested.}
   \label{fig:qual_2}
\end{figure}

\begin{figure}[t]
  \centering
   \includegraphics[width=0.9\linewidth]{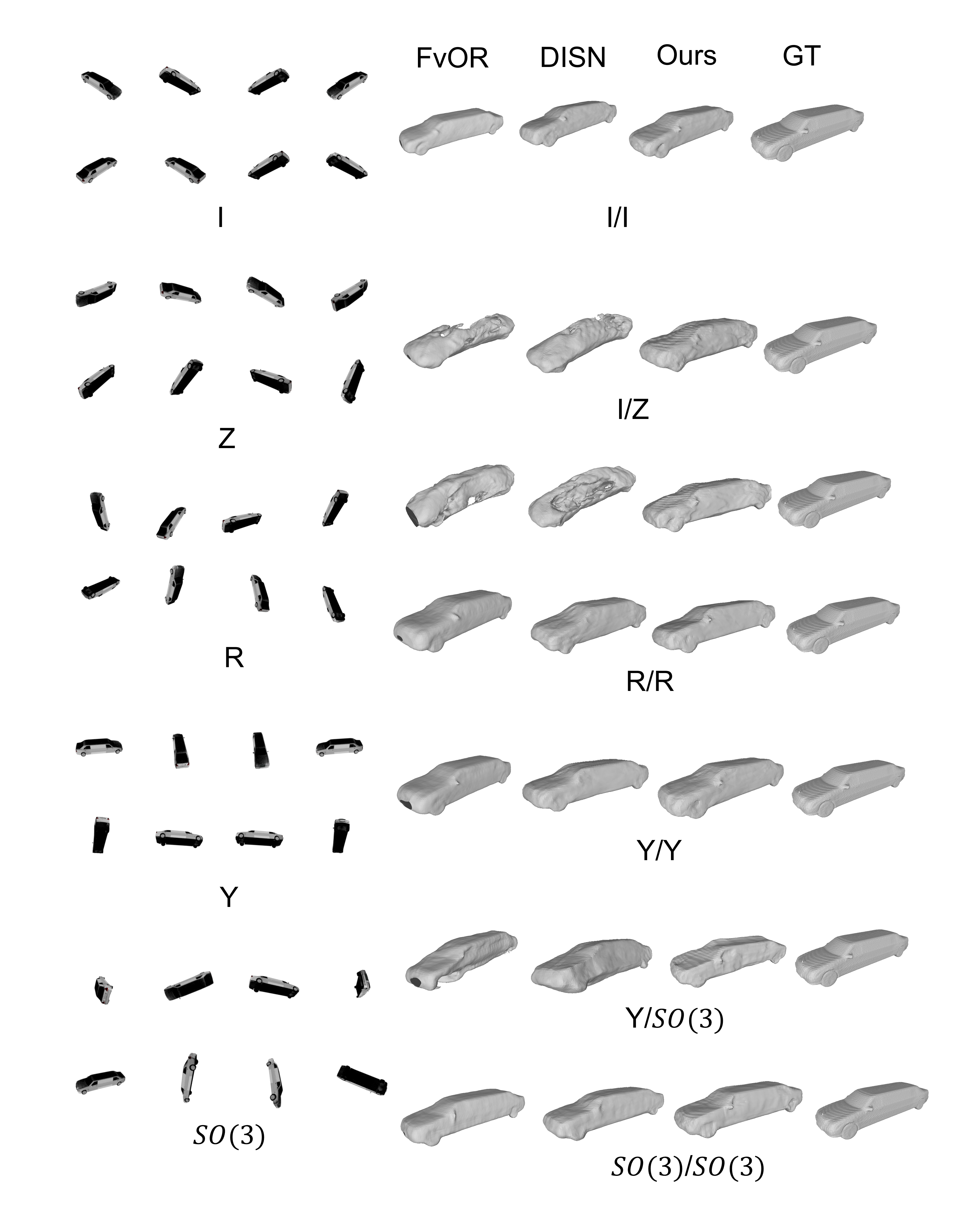}
   \caption{Qualitative Result for the car. Left: input views; Right: reconstruction meshes of different models. The captions below the meshes show how the model is trained and tested.}
   \label{fig:qual_3}
\end{figure}

\begin{figure}[t]
  \centering
   \includegraphics[width=1.0\linewidth]{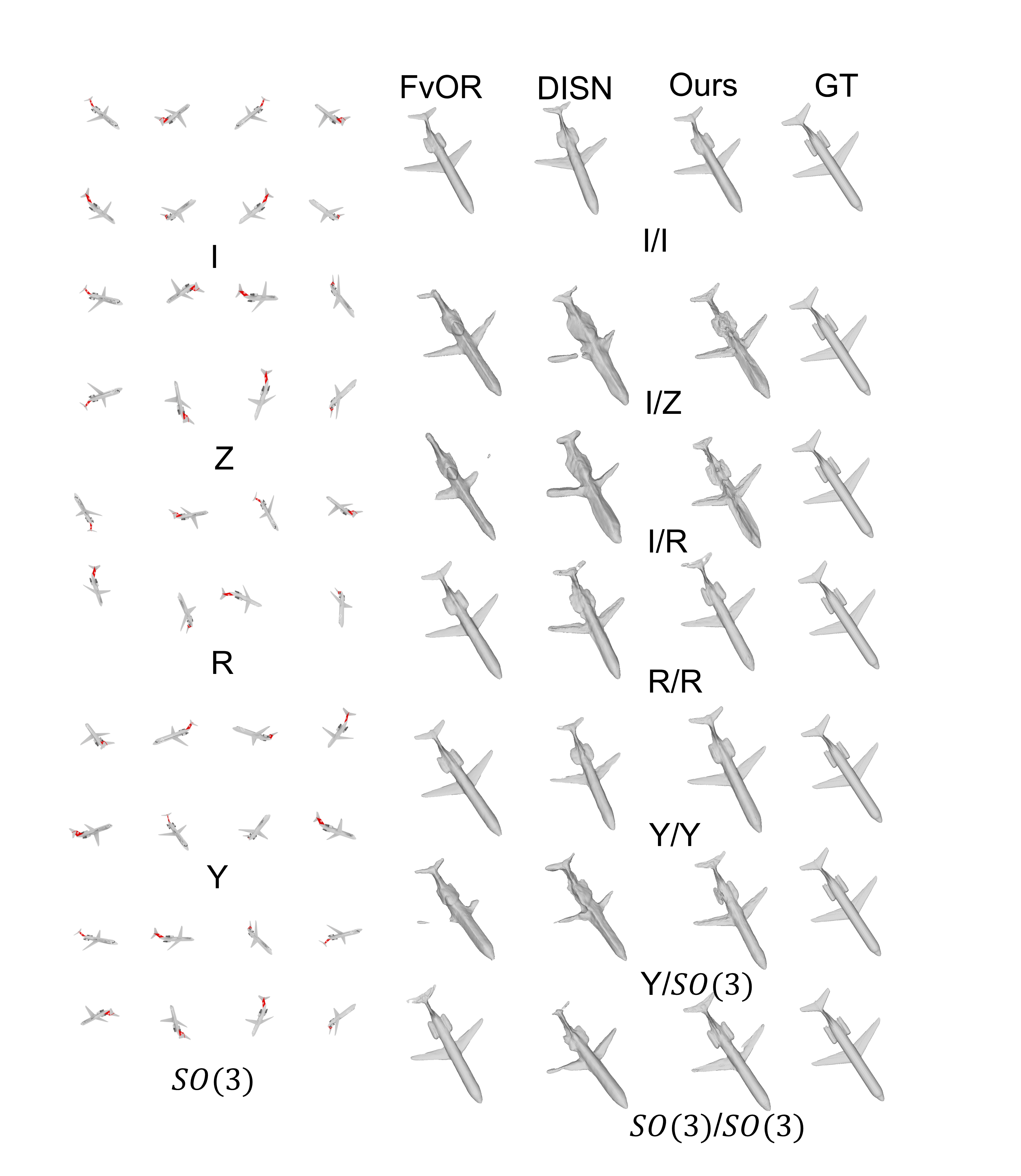}
   \caption{Qualitative Result for the car. Left: input views; Right: reconstruction meshes of different models. The captions below the meshes show how the model is trained and tested.}
   \label{fig:qual_4}
\end{figure}

\end{document}